%% file: neurips_2026.tex
\documentclass{article}

\usepackage[preprint]{neurips_2026}

\usepackage[utf8]{inputenc} 
\usepackage[T1]{fontenc}    
\usepackage{hyperref}       
\usepackage{url}            
\usepackage{booktabs}       
\usepackage{amsfonts}       
\usepackage{nicefrac}       
\usepackage{microtype}      
\usepackage[table]{xcolor}  
\usepackage{amsmath}
\usepackage{amsthm}
\usepackage{wrapfig,lipsum,booktabs}
\usepackage{multirow}
\usepackage{subcaption}
\usepackage{algorithm}
\usepackage{algpseudocode}

\usepackage{graphicx}
\usepackage{booktabs} 

\usepackage{star}
\input{math_commands.tex}

\newcommand{\m}[1]{\mathbf{#1}}

\newtheorem{proper}{Criterion}
\usepackage{hyperref}
\usepackage{amsmath}
\usepackage{amssymb}
\usepackage{mathtools}
\usepackage{amsthm}
\usepackage[normalem]{ulem}

\usepackage{url}
\usepackage{wrapfig,lipsum,booktabs}
\usepackage{multirow}
\usepackage{subcaption}
\usepackage{fontawesome5}

\newcommand{\unif}{{\textnormal{\small Unif}}}

\renewcommand{\hat}{\widehat}

\newcommand{\tr}{\mathop{\mathrm{tr}}}

\renewcommand{\triangleq}{:=}

\theoremstyle{plain}
\theoremstyle{definition}

\usepackage{tcolorbox}
\definecolor{linen}{RGB}{250,240,230}

\definecolor{mydarkblue}{rgb}{0,0.08,0.45}

\hypersetup{ %
    pdftitle={},
    pdfauthor={},
    pdfsubject={},
    pdfkeywords={},
    pdfborder=0 0 0,
    pdfpagemode=UseNone,
    colorlinks=true,
    linkcolor=mydarkblue,
    citecolor=mydarkblue,
    filecolor=mydarkblue,
    urlcolor=mydarkblue,
    pdfview=FitH
}

\usepackage{tcolorbox}
\definecolor{linen}{RGB}{250,240,230}
\usepackage{enumitem}

\title{Distributional Matching for Vector Quantization: A Unified Theoretical and Empirical Framework}

\vspace{-5pt}
\author{
  \vspace{-30pt}\\
  \textbf{Xianghong Fang$^{1}$$^{*}$  \quad Litao Guo$^{2}$ \quad Hengchao Chen$^{1}$ \quad Yuxuan Zhang$^1$ \quad XiaofanXia$^{1}$} \\[2pt]
  \textbf{\quad Dingjie Song$^4$  \quad Yexin Liu$^{2}$  \quad Hao Wang$^5$  \quad Harry Yang$^{2}$ \quad Qiang Sun$^{1}$$^{*}$ \quad Yuan Yuan$^{3}$$^{*}$} \vspace{5pt} \\
  $^1$University of Toronto ~\quad $^2$The Hong Kong University of Science and Technology
  ~\quad $^3$Boston College \\[2pt]
  $^4$Lehigh University 
  $^5$Southern University of Science and Technology \\[4pt]
  {\upshape
  \vspace*{15pt}
    \href{https://vq-research.github.io/Wasserstein-VQ/}{\faGlobe\enspace Website}
    \quad
    \href{https://github.com/VQ-Research/Wasserstein-VQ}{\faGithub\enspace Code \& Models}
  }}
\vspace{-4pt}

\usepackage{array}

\begin{document}

\newcolumntype{x}[1]{>{\centering\arraybackslash}p{#1pt}}
\newcolumntype{y}[1]{>{\raggedright\arraybackslash}p{#1pt}}
\newcolumntype{z}[1]{>{\raggedleft\arraybackslash}p{#1pt}}

\maketitle

\begingroup
\renewcommand\thefootnote{}
\footnotetext{
$^{*}$ Corresponding authors. Email: xianghong.fang@mail.utoronto.ca, yuanyua@bc.edu.
}
\endgroup

\begin{abstract}
The effectiveness of modern visual representation learning and autoregressive models critically depends on vector quantization (VQ), which discretizes continuous feature representations using a learnable codebook. Despite its widespread use, existing VQ methods often suffer from training instability and codebook collapse, arising from gradient mismatch induced by the straight-through estimator and the under-utilization of code vectors. In this work, we show that both issues can be traced to a fundamental mismatch between the distributions of feature vectors and code vectors, leading to inefficient representation and information loss. Building on this observation, we propose a distributional matching framework for vector quantization. We introduce principled criteria for desirable VQ behavior and demonstrate through theoretical analysis and empirical evaluation that aligning feature and code vector distributions provides a unifying mechanism for mitigating training instability and codebook collapse. We instantiate this framework using a Wasserstein-based objective with an efficient closed-form under a mild Gaussian approximation, and further show that a nonparametric alternative based on maximum mean discrepancy yields comparable performance. Extensive experiments on visual tokenization benchmarks support the effectiveness and robustness of the proposed approach. 
\end{abstract}

\input{sections/introduction.tex}

\input{sections/vector_quantization.tex}

\input{sections/method.tex}
\input{sections/atomic_setting.tex}
\input{sections/experiment.tex}
\input{sections/conclusion.tex}

\newpage
\bibliographystyle{plain}
\bibliography{neurips_2026}

\newpage
\appendix
\input{sections/appendix}

\end{document}

%% file: math_commands.tex
\usepackage{amsmath,amsfonts,bm}

\def\eqref#1{equation~\ref{#1}}

\def\1{\bm{1}}

\DeclareMathAlphabet{\mathsfit}{\encodingdefault}{\sfdefault}{m}{sl}
\SetMathAlphabet{\mathsfit}{bold}{\encodingdefault}{\sfdefault}{bx}{n}

\newcommand{\KL}{D_{\mathrm{KL}}}

\DeclareMathOperator*{\argmin}{arg\,min}

%% file: sections/introduction.tex
\section{Introduction}\label{sec:introduction}

Vector quantization (VQ)~\cite{Oord2017NeuralDR} is a fundamental building block in a wide range of modern representation learning and visual tokenization frameworks, including autoregressive and hybrid generative models~\cite{Razavi2019GeneratingDH,Esser2020TamingTF,Chang2022MaskGITMG,Lee2022AutoregressiveIG,Gu2021VectorQD,Tian2024VisualAM,Li2024XQGANAO}. By discretizing continuous feature representations using a learnable codebook, VQ enables compact and structured latent representations that are well suited for downstream modeling. Despite its broad adoption, VQ remains notoriously difficult to optimize in practice, often exhibiting unstable training dynamics and severe codebook collapse.

The first challenge stems from the non-differentiability of the quantization operation, which prevents gradients from being directly propagated from quantized features to their continuous counterparts. To address this issue, prior work introduces the straight-through estimator (STE)~\cite{bengio2013estimating,Oord2017NeuralDR}, which approximates gradients by copying them from the quantized features to encoder outputs. However, the effectiveness of this approximation critically depends on the magnitude of the quantization error. When the discrepancy between continuous features and their assigned code vectors becomes large, the resulting gradient mismatch can lead to unstable optimization and degraded training behavior~\cite{Lee2022AutoregressiveIG}.

A second, closely related challenge is codebook collapse, where only a small subset of code vectors receives assignments while the majority remain unused. From a geometric perspective, this corresponds to a degenerate Voronoi partition\footnote{Appendix.~\ref{appendix:understanding codebook collapse by Voronoi Partition} discusses codebook collapse via Voronoi partitioning.} in which most cells are never activated~\cite{Zheng2023OnlineCC}. Although extensive research has sought to alleviate this problem, low code vector utilization often remains in practice, particularly with large codebooks~\cite{Dhariwal2020JukeboxAG,Takida2022SQVAEVB,Yu2021VectorquantizedIM,Lee2022AutoregressiveIG,Zheng2023OnlineCC}. This limitation is further exacerbated as increasing the codebook size expands the number of Voronoi cells, making it substantially harder to ensure that all cells are sufficiently populated.

\begin{wrapfigure}[12]{r}{0.54\textwidth}
\small
\vspace{-6mm}
\begin{center}
\includegraphics[width=0.52\textwidth]{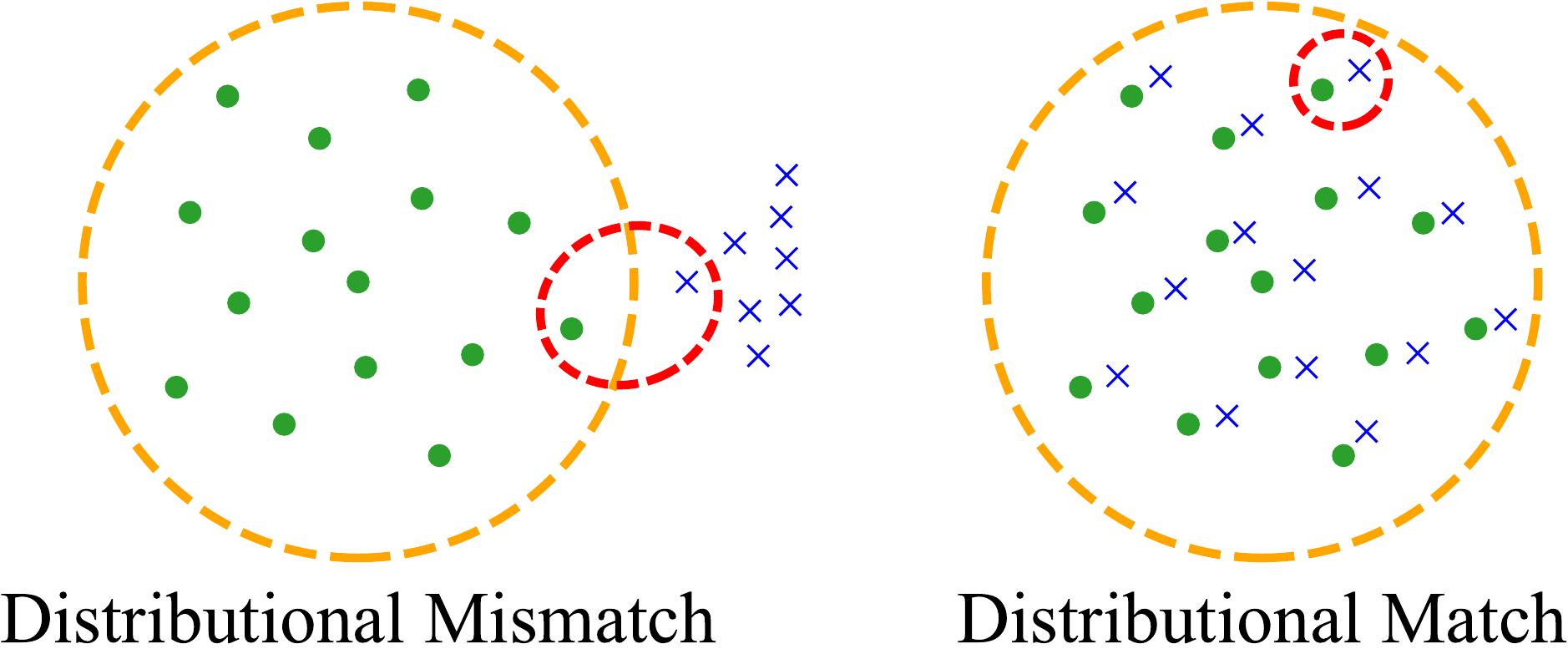}
\vspace{-2mm}
\caption{\small{The symbols $\cdot$ and $\times$ represent the feature and code vectors, respectively. The left figure illustrates the distributional mismatch between the feature and code vectors, while the right figure visualizes their distributional match.}}
\label{fig:distributional match}
\end{center}
\end{wrapfigure} 
In this work, we examine these challenges from a distributional perspective. Rather than treating training instability and codebook collapse as separate phenomena, we observe that both issues are closely tied to a fundamental mismatch between the distributions of feature vectors produced by the encoder and the distributions of the learnable code vectors. Figure~\ref{fig:distributional match} illustrates this intuition using two representative scenarios. When the two distributions are poorly aligned, feature vectors concentrate around a small subset of code vectors, resulting in large quantization errors and low codebook utilization. In contrast, when the distributions are well matched, quantization errors are reduced and codebook utilization approaches its maximum.

Building on this observation, we introduce three principled criteria that characterize the desirable behavior in vector quantization. Guided by this criterion triple, we show through both theoretical analysis and empirical evaluation that distributional alignment between feature vectors and code vectors provides a unifying mechanism for mitigating training instability and codebook collapse. To operationalize this idea, we adopt a distribution matching objective based on the quadratic Wasserstein distance. Under a mild Gaussian approximation, this objective admits a closed-form expression that can be efficiently optimized during training. Importantly, we show that even when this approximation is relaxed, a nonparametric alternative based on maximum mean discrepancy (MMD)~\cite{Gretton2012AKT, Sriperumbudur2009HilbertSE,Fang2025VQTransplant} yields comparable performance, suggesting that distributional matching effectively captures the essential structure required for effective vector quantization. These insights translate into consistent improvements in reconstruction fidelity across visual tokenization benchmarks.

%% file: sections/vector_quantization.tex
\section{A Distribution Matching Perspective on Vector Quantization}\label{sec:understanding}

This section introduces a novel distributional perspective for understanding vector quantization. By defining three principled criteria for VQ evaluation, we provide empirical and theoretical evidence that aligning feature and code vector distributions  yields a near-optimal VQ solution.

\subsection{An Overview of Vector Quantization}\label{sec:overview vq}
As a core component of visual tokenizers~\cite{Oord2017NeuralDR,Lee2022AutoregressiveIG,Tian2024VisualAM}, VQ acts as a compressor that discretizes continuous latent features into discrete visual tokens by mapping them to the nearest code vectors within a learnable codebook. 

\begin{wrapfigure}[9]{r}{0.50\textwidth}
\small
\vspace{-8mm}
\begin{center}
\includegraphics[width=0.48\textwidth]{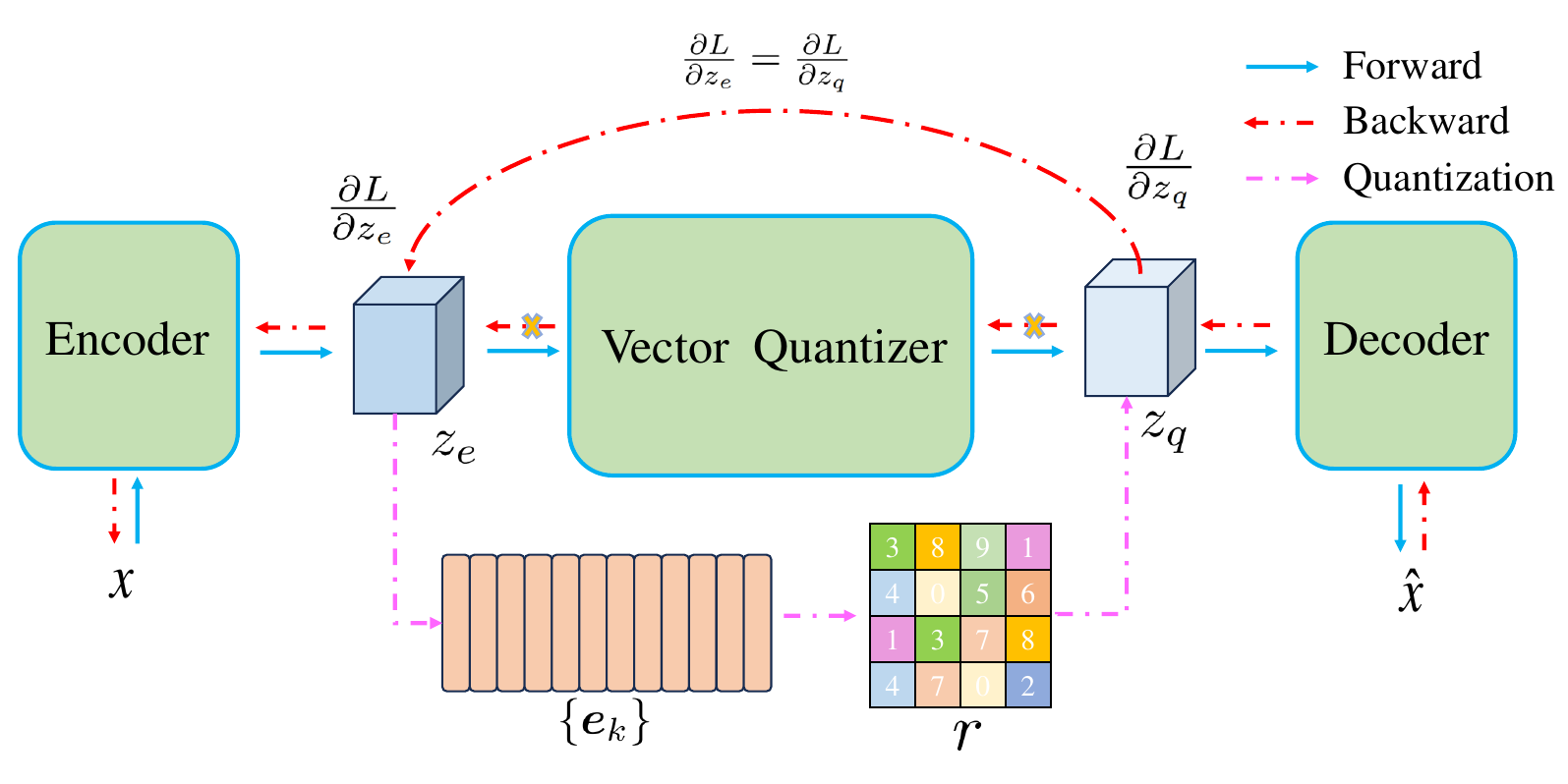}
\vspace{-4mm}
\caption{\small{The illustration of VQ.}}
\label{fig:vq model}
\end{center}
\end{wrapfigure} 
Figure~\ref{fig:vq model} illustrates the classic VQ process~\cite{Oord2017NeuralDR}, which consists of an encoder $E(\cdot)$, a decoder $D(\cdot)$, and an updatable codebook  $\{\m e_k\}_{k=1}^{K} \in \mathbb{R}^d$ containing a finite set of
code vectors. Here, $K$ denotes the codebook size and $d$ the code vector dimension. Given an image $\bm{x}\in \mathbb{R}^{H\times W \times 3}$, the goal is to derive a spatial collection of codeword IDs $r \in \mathbb{N}^{h\times w}$ as image tokens. This is achieved by encoding the image to obtain $\bm{z}_e = E(\bm{x}) \in \mathbb{R}^{h\times w \times d}$, followed by a spatial quantizer $\mathcal{Q}(\cdot)$ that maps each spatial feature $\bm{z}_{e}^{ij}$ to its nearest code vector $\bm{e}_k$:
\begin{equation}
r^{ij} = \argmin_{k} \|\bm{z}_{e}^{ij} - \bm{e}_k\|_2^2.
\end{equation}
The resulting tokens retrieve the corresponding codebook entries $ \bm{z}_q^{ij} = \mathcal{Q}(\bm{z}_{e}^{ij}) = \bm{e}_{r^{ij}}$, which are then passed through the decoder to reconstruct the image as $\bm{\hat{x}} = D(\bm{z}_q)$. Despite its widespread adoption in visual tokenization, representation learning, and high-fidelity image synthesis~\cite{Esser2020TamingTF}, VQ still faces two key challenges: training instability and codebook collapse.

\paragraph{Training Instability}
This issue arises because during backpropagation, the gradient of $\bm{z}_q$ cannot flow directly to $\bm{z}_e$ due to the non-differentiable function $\mathcal{Q}$. To optimize the encoder's parameters, VQ-VAE~\cite{Oord2017NeuralDR} employs the straight-through estimator~(STE)~\cite{Bengio2013EstimatingOP}, which copies gradients from $\bm{z}_q$ to $\bm{z}_e$. However, this approach carries significant risks, especially when $\bm{z}_q$ and $\bm{z}_e$ are far apart. In such cases, the gradient gap can grow substantially, destabilizing training\footnote{Appendix~\ref{appendix:gradient gap quantization error} details the gradient-quantization relationship.}. In this work, we tackle the training instability challenge from a distributional viewpoint. This perspective highlights that training instability is not merely an implementation artifact, but a consequence of systematic mismatch between feature and code distributions.

\paragraph{Codebook Collapse} 
Codebook collapse occurs when only a small subset of code vectors receives gradients, while most remain unrepresentative and unupdated~\cite{Dhariwal2020JukeboxAG,Yu2021VectorquantizedIM,Lee2022AutoregressiveIG,Zheng2023OnlineCC}. Researchers have proposed solutions such as improved codebook initialization~\cite{Zhu2024ScalingTC}, reinitialization strategies~\cite{Dhariwal2020JukeboxAG,Williams2020HierarchicalQA}, and classical clustering algorithms like $k$-means~\cite{Bradley1998RefiningIP} and $k$-means++\cite{Arthur2007kmeansTA} for codebook optimization~\cite{Razavi2019GeneratingDH,Zheng2023OnlineCC}. Beyond deterministic approaches that select the best-matching token, stochastic quantization strategies have also been explored~\cite{Zhang2023RegularizedVQ,Ramesh2021ZeroShotTG,Takida2022SQVAEVB}. However, these methods still exhibit low code vector utilization, particularly with large codebook sizes $K$~\cite{Zheng2023OnlineCC,Mentzer2023FiniteSQ}. In this paper, we address this issue via distributional matching between feature and code vectors.

\subsection{Evaluation Criteria}\label{sec:distributional formulation}
While these quantities have appeared individually in prior work, we emphasize that considering them jointly provides a unified lens for analyzing both optimization stability and codebook collapse.

Given feature vectors $\{\bz_i\}_{i=1}^N$ from feature distribution $\mathcal{P}_A$ and code vectors $\{\be_k\}_{k=1}^K$ sampled from codebook distribution $\mathcal{P}_B$, vector quantization involves finding the nearest, and thus most representative, code vector for each feature: 
\begin{align}
\bz_i'=\argmin_{\be\in\{\be_k\}}\norm{\bz_i-\be}.
\end{align}
Each original feature vector $\bz_i$ is then quantized to $\bz_i'$. We denote the index of the selected code vector for the $i$-th feature as $r_i$, such that $\bm{z}'_i = \bm{e}_{r_i}$.
We next introduce three key criteria to evaluate this process.

\begin{proper}[Quantization Error]\label{criteria:qe} 
The quantization error measures the average distortion introduced and is defined as:
\vspace{-1.0ex}
\begin{align}
\cE(\{\be_k\};\{\bz_i\})=\frac{1}{N}\sum_{i}\norm{\bz_i-\bz_i'}^2.
\end{align}
\end{proper}
A smaller $\cE$ signifies a more accurate quantization of the original feature vectors. 
Beyond this geometric interpretation, $\cE$ is also directly linked to training stability. 
As formally derived in Appendix~\ref{appendix:gradient gap quantization error}, the gradient discrepancy $\mathcal{G}_i$ for a feature $\bz_i$ introduced by the straight-through estimator (STE) is theoretically bounded by its distance to the code vector:
\begin{equation}
    \mathcal{G}_i \le \| \mathbf{H} \|_2 \cdot \| \bz_i-\bz_i' \|,
\end{equation}
where $\mathbf{H} = \left.\frac{\partial^2 \mathcal{L}}{\partial x^2}\right|_{x=\bz_i'}$ 
denotes the Hessian of the overall training loss $\mathcal{L}$ with respect to the latent input,
evaluated at the quantized code $\bz_i'$, characterizing the local curvature of the loss landscape.
Since $\cE$ is the mean squared magnitude of the term $\| \bz_i-\bz_i' \|$, minimizing $\cE$ explicitly tightens this bound. This ensures that the approximated gradients remain faithful to the true gradients, effectively mitigating the gradient mismatch that causes training instability.

\begin{proper}[Codebook Utilization Rate]\label{criteria:cu} The codebook utilization rate is the fraction of code vectors used in VQ
\begin{align}
\mathcal{U}(\{\bm{e}_k\};\{\bm{z}_i\}) 
= \frac{1}{K} \sum_{k=1}^K \mathbf{1}\left(k \in \{ r_i \}_{i=1}^N \right),
\end{align}
\end{proper}
where the indicator evaluates to 1 if code vector $\be_k$ is selected by at least one feature. 

A higher $\cU$ reduces the risk of codebook collapse. Ideally, $\cU$ should reach 100\%, meaning all code vectors are utilized. As discussed in Appendix~\ref{appendix:explanation on criterions}, $\cU$ measures only the \emph{completeness} of codebook utilization and cannot evaluate the degree of collapse. This motivates introducing the codebook perplexity criterion.

\begin{proper}[Codebook Perplexity]\label{criteria:cp} The codebook perplexity measures the uniformity of codebook utilization in VQ
\begin{align}
\mathcal{C}(\{\be_k\};\{\bz_i\}) = \exp(-\sum_{k=1}^{K} p_k \log p_k),
\end{align} 
\end{proper}
where $p_k =\frac{1}{N}\sum_{i=1}^N\one(\bz_i'=\be_k)$. A higher $\mathcal{C}$ indicates more uniform selection of code vectors in VQ. Ideally, $\mathcal{C}$ reaches its maximum at $\mathcal{C}_{0} = \exp(-\sum_{k=1}^{K} \frac{1}{K} \log \frac{1}{K})=K$ when code vectors are completely uniformly utilized. Thus, as a complement to Criterion~\ref{criteria:cu}, $\cU$ combined with $\mathcal{C}$ effectively evaluates codebook collapse.

We refer to $(\cE, \cU, \mathcal{C})$ as the criterion triple. When comparing extreme cases of distributional match and mismatch shown in Figure~\ref{fig:distributional match}, we find that distributional matching significantly outperforms mismatching across all three criteria. Using this triple, we can analyze the benefits of distribution matching more systematically.

\paragraph{Quantization Error vs. Codebook Utilization Rate}
Quantization error ($\cE$) is a more fundamental objective than codebook utilization ($\cU$) in VQ. We theoretically show in \cite{Fang2026UnifyingView} that minimizing $\cE$ necessarily induces full codebook utilization, whereas the converse does not hold. This hierarchical relationship suggests treating $\cE$ as the primary optimization target, with $\cU$ serving as a complementary metric to monitor coverage.

\paragraph{Remark} Notably, $\cE$ is sensitive to the variance of the latent feature distribution (Appendix~\ref{appendix:Impact of Distribution Variance}). Consequently, direct comparisons of raw $\cE$ values across latent spaces with different variances can be misleading. To fairly evaluate quantization methods via criterion triple, all comparisons should be made under identical latent distributions.



\subsection{The Effects of Distribution Matching} \label{sec:effects of distribution matching}
We conduct a deliberately simple synthetic experiment to isolate the geometric effects of distribution mismatch (see details in Appendix~\ref{appendix:details part1}).
Specifically, we assume that $\mathcal{P}_A$ and $\mathcal{P}_B$ are uniform distributions within two distinct disks, as shown in Figure~\ref{fig:criterion triple analysis}. We sample feature vectors $\{\bz_i\}_{i=1}^N$ from the red disk and code vectors $\{\be_k\}_{k=1}^K$ from the green circle. The criterion triple $(\cE, \cU, \mathcal{C})$ is calculated by Criteria~\ref{criteria:qe} to~\ref{criteria:cp}.

We examine two cases. The \underline{first} involves disks with identical radii but different centers. As shown in Figures~\ref{fig:diagram 1 group 1} to~\ref{fig:diagram 4 group 1}, when the centers move closer, the criterion triple improves toward optimal values. Specifically, $\cE$ decreases from 1.19 to 0.05, $\cU$ rises from 2\% to 100\%, and $\mathcal{C}$ increases from 3.8 to 344.9. The \underline{second} case shows distributions with identical centers but different radii. When the codebook support lies within the feature support (Figures~\ref{fig:diagram 1 group 2} and~\ref{fig:diagram 2 group 2}), $\cE$ is larger, $\cU$ slightly lower, and $\mathcal{C}$ smaller compared to the aligned distributions in Figure~\ref{fig:diagram 4 group 1}. When the codebook support extends beyond the feature support, $\cE$ increases modestly while $\cU$ and $\mathcal{C}$ drop significantly (Figures~\ref{fig:diagram 3 group 2} and~\ref{fig:diagram 4 group 2}). Detailed explanations are provided in Appendix~\ref{appendix:prototypical study}.

Overall, VQ achieves the optimal criterion triple when feature and codebook distributions are identical. This is further supported by quantitative analyses in Appendix~\ref{appendix:supplementary comprehensive quantitative analyses}.

\begin{figure*}[!t]
    \vspace{-2mm}
	\centering
	\subfloat[$(1.19, 2\%, 3.8)$]{ 
		\label{fig:diagram 1 group 1}  
		\includegraphics[width=0.22\textwidth]{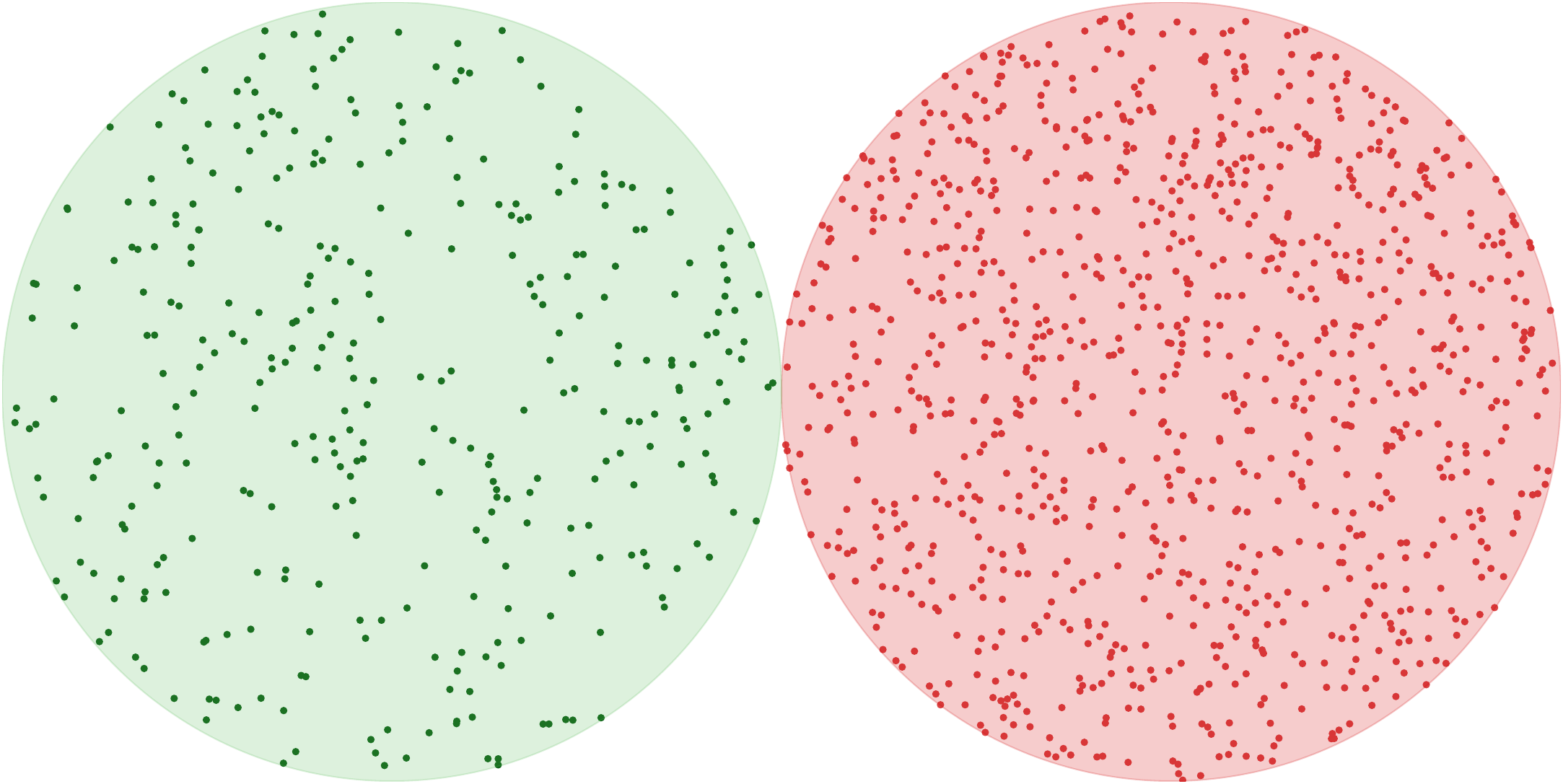}}
	\subfloat[$(0.70, 20.8\%, 16.5)$]{
		\label{fig:diagram 2 group 1}
		\includegraphics[width=0.22\textwidth]{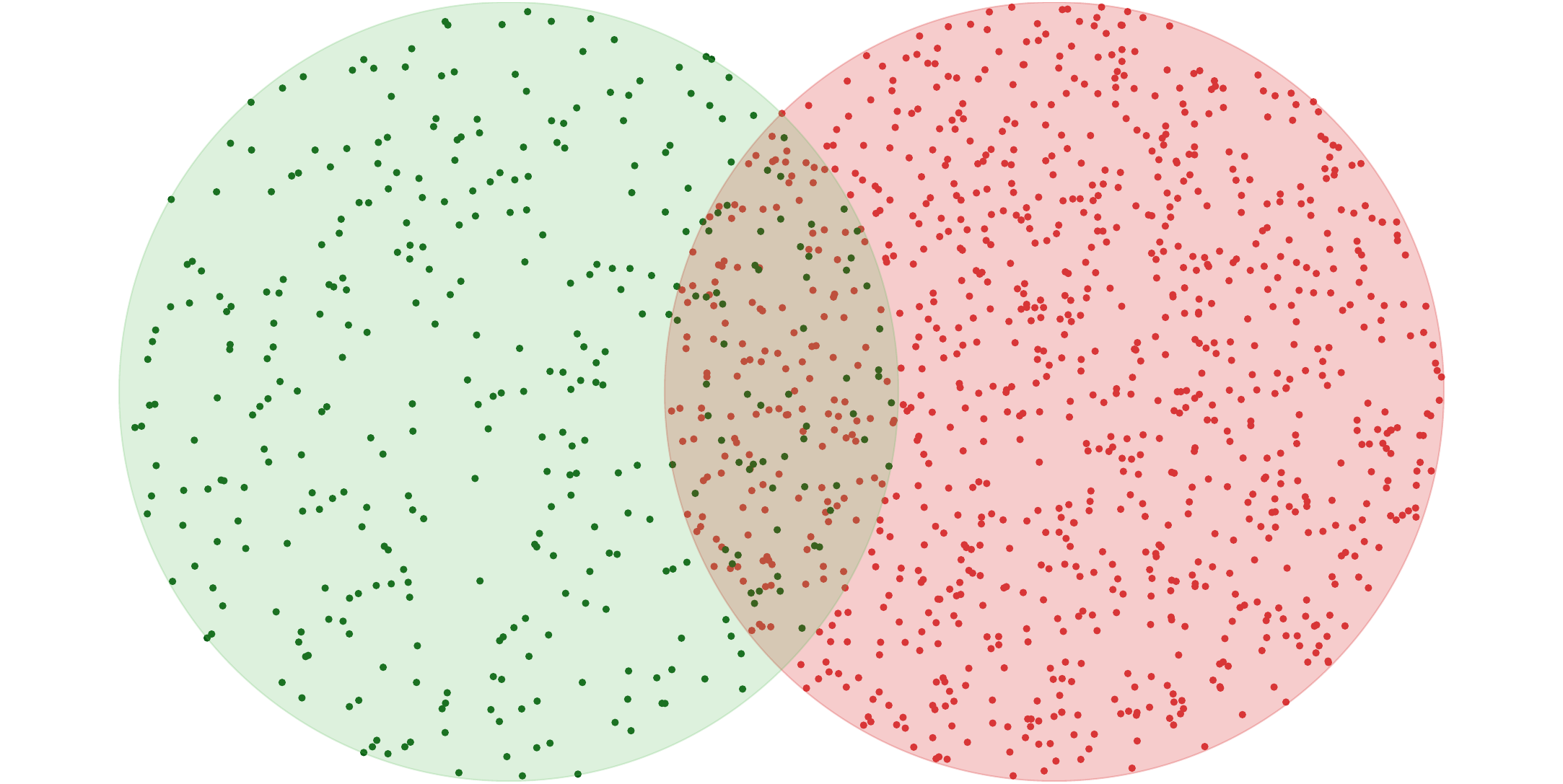}
	}
	\subfloat[$(0.26, 57.8\%, 96.9)$]{ 
		\label{fig:diagram 3 group 1}  
		\includegraphics[width=0.22\textwidth]{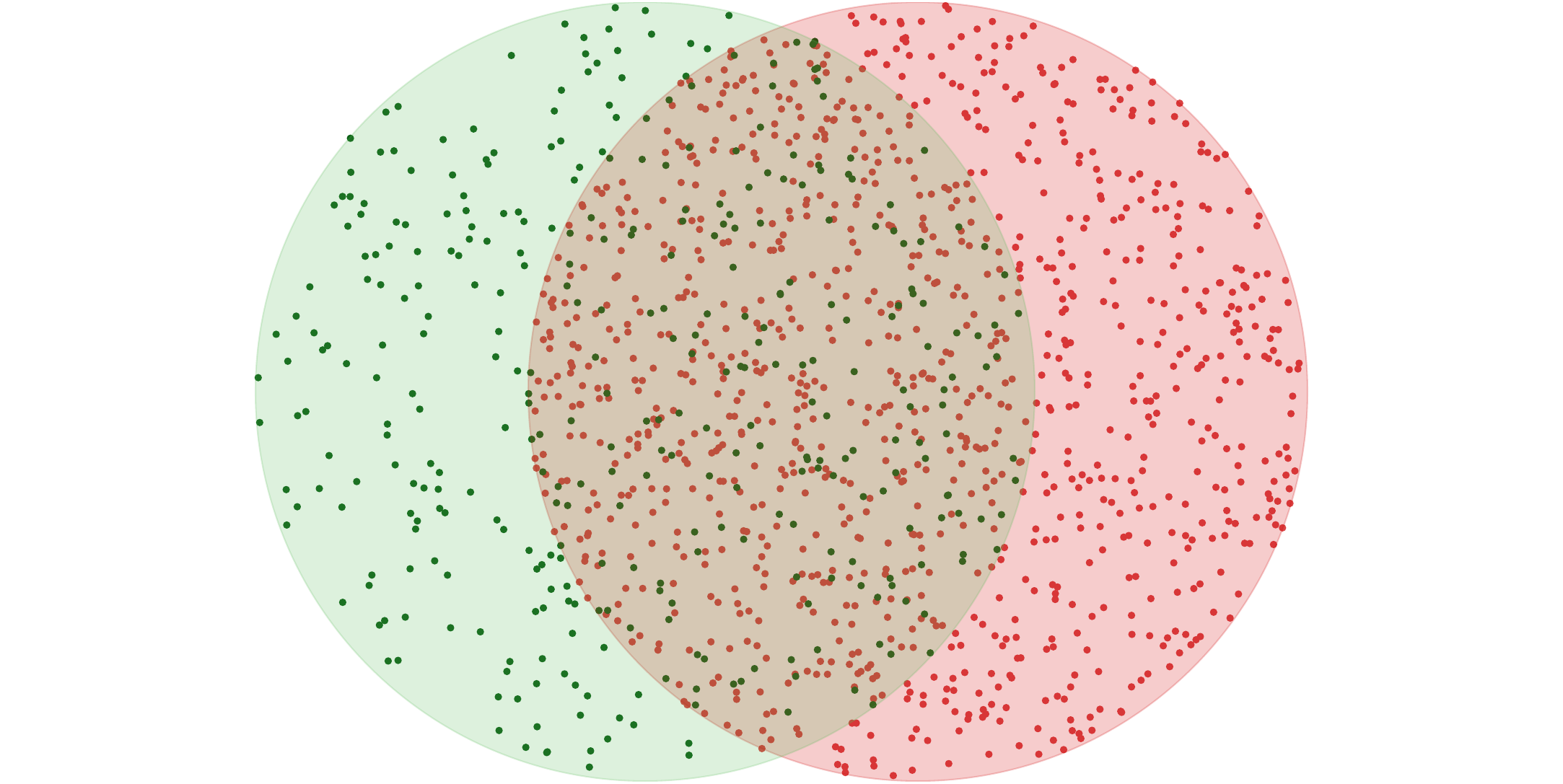}
	}
	\subfloat[$(0.05, 100\%, 344.9)$]{
		\label{fig:diagram 4 group 1}
		\includegraphics[width=0.22\textwidth]{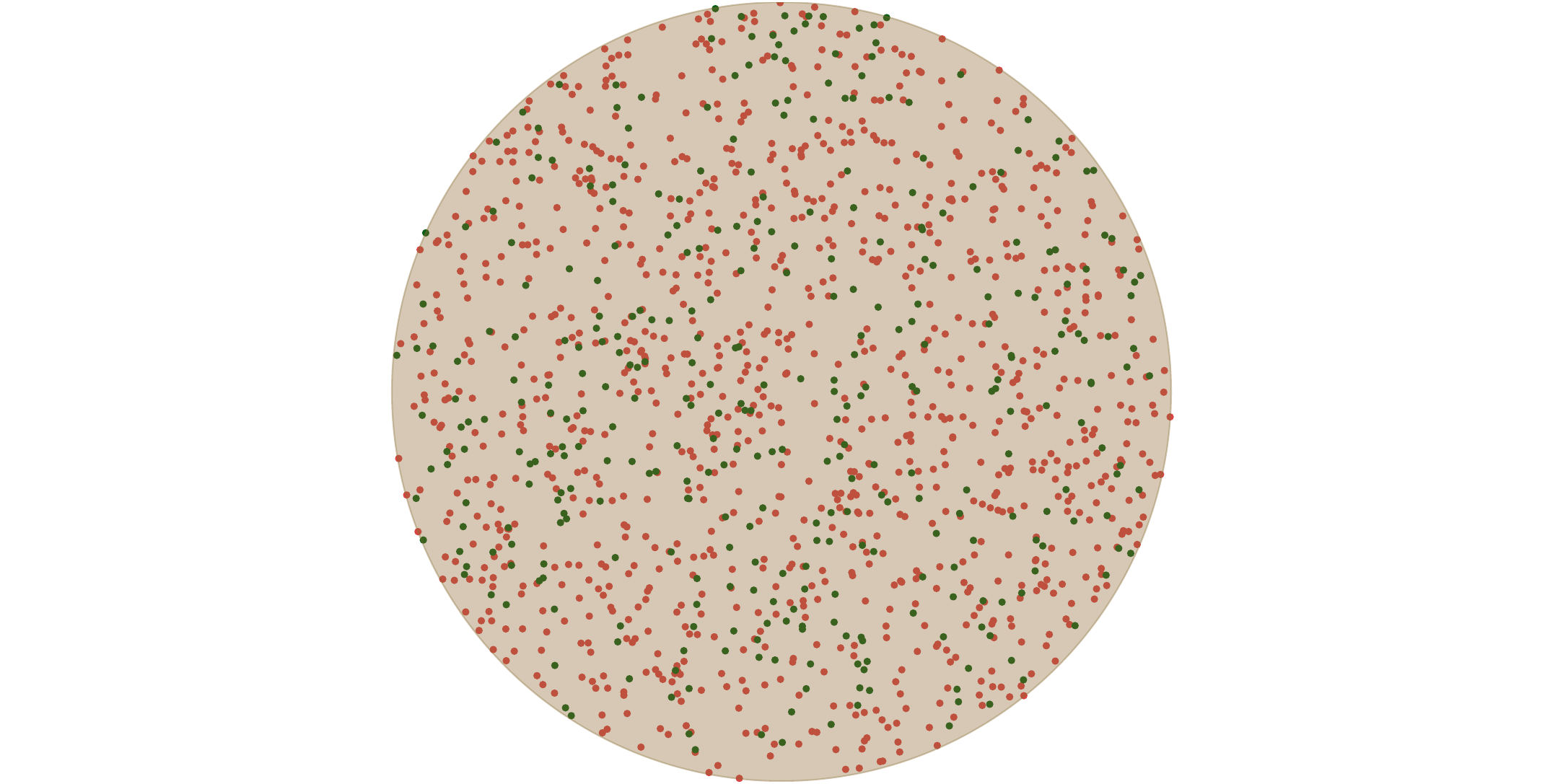}
	}
    \vspace{-0.1mm}
	\subfloat[$(0.36, 93.3\%, 63.2)$]{
		\label{fig:diagram 1 group 2}
		\includegraphics[width=0.22\textwidth]{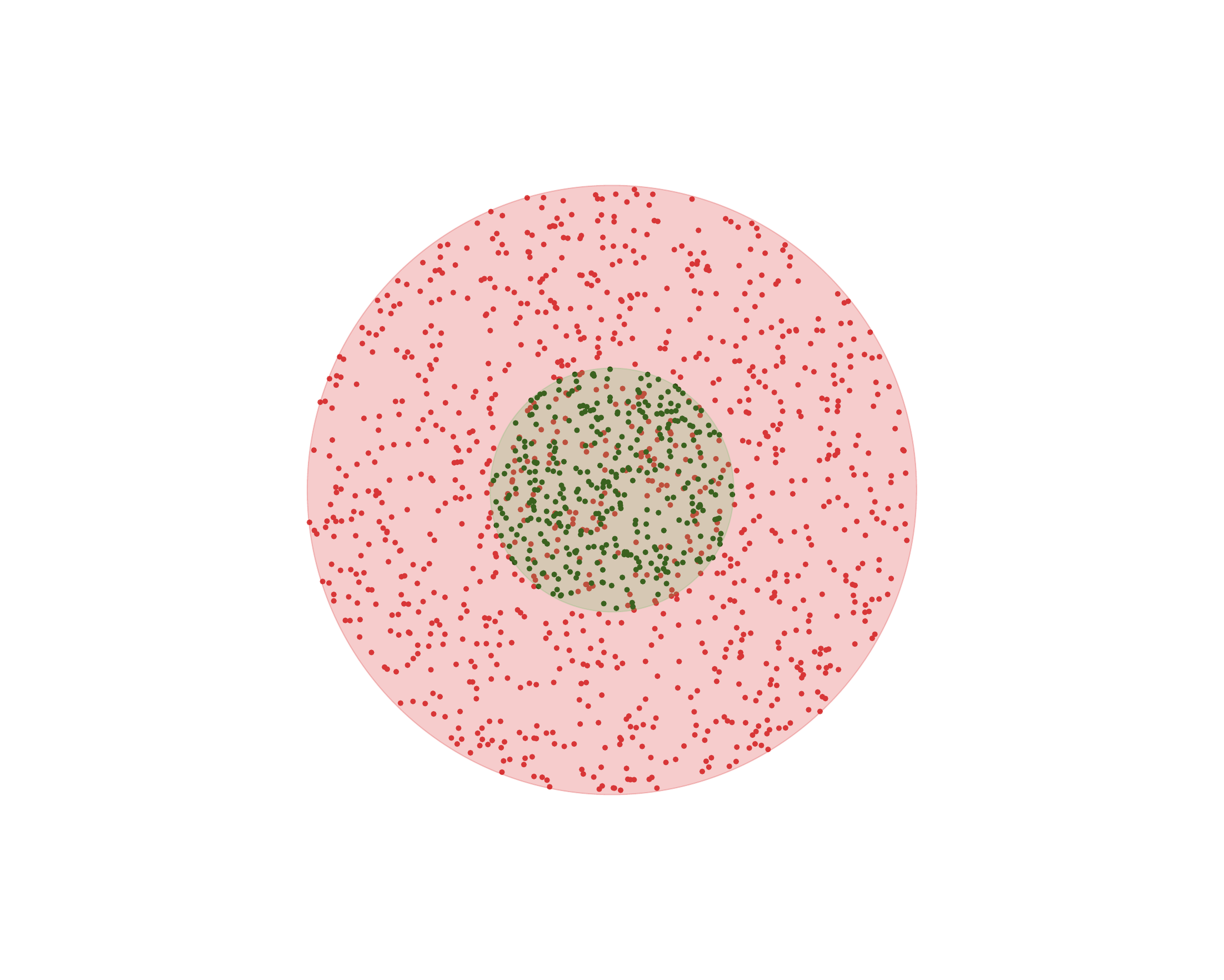}
	}
	\subfloat[$(0.10, 99.8\%, 250.5)$]{
		\label{fig:diagram 2 group 2}
		\includegraphics[width=0.22\textwidth]{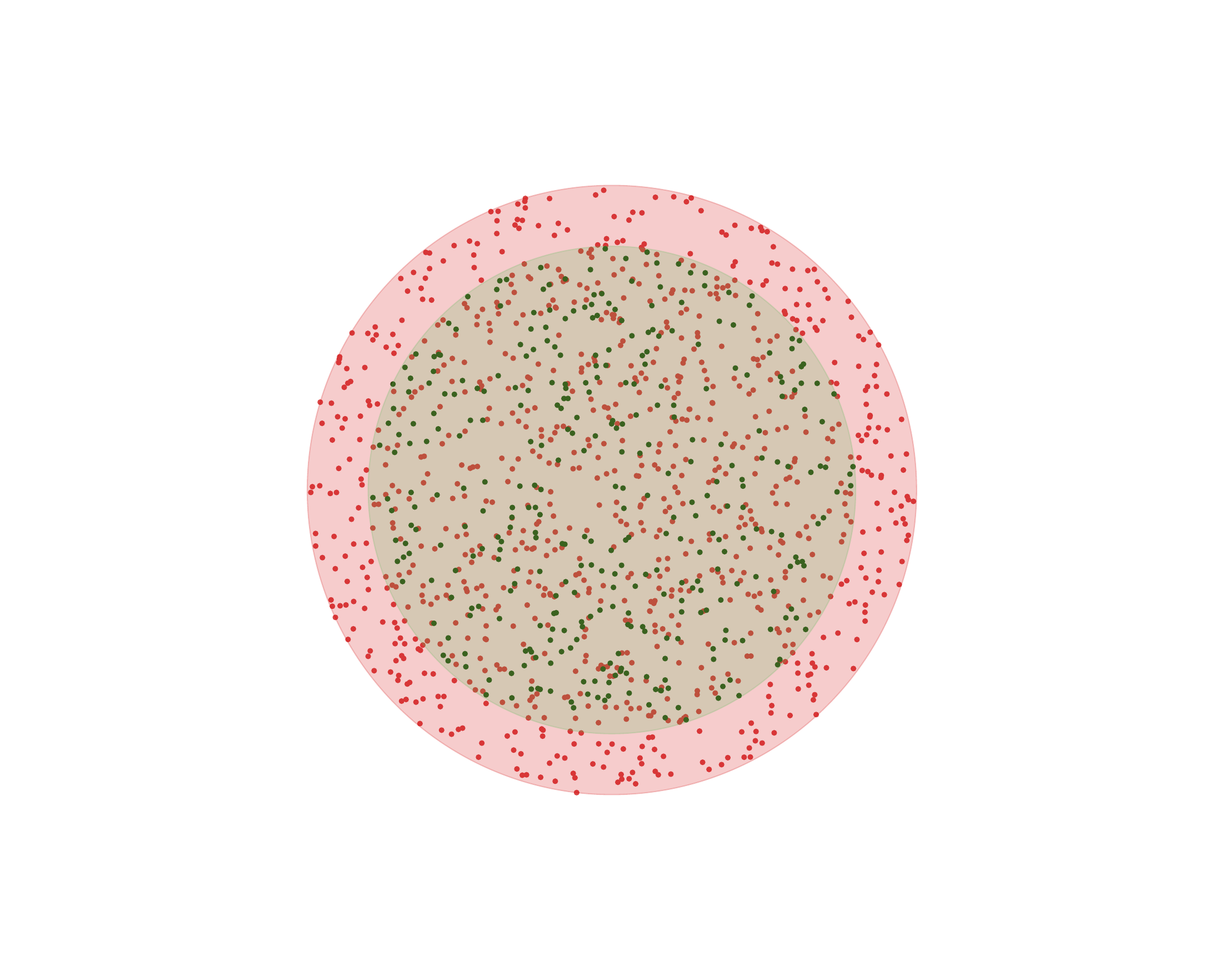}
	}
	\subfloat[$(0.07, 61.3\%, 199.7)$]{
		\label{fig:diagram 3 group 2}
		\includegraphics[width=0.22\textwidth]{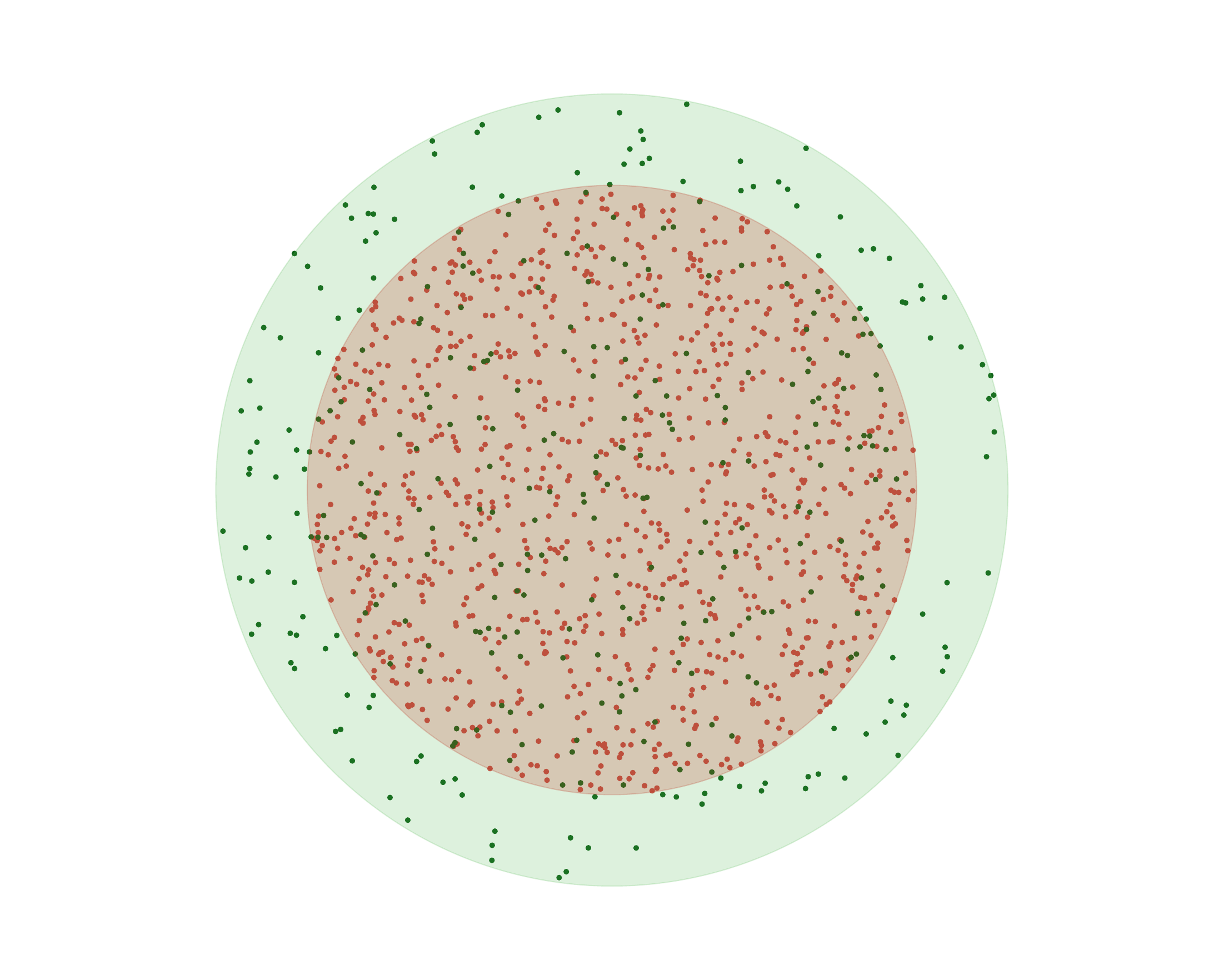}
	}
	\subfloat[$(0.08, 45.3\%, 151.5)$]{
		\label{fig:diagram 4 group 2}
		\includegraphics[width=0.22\textwidth]{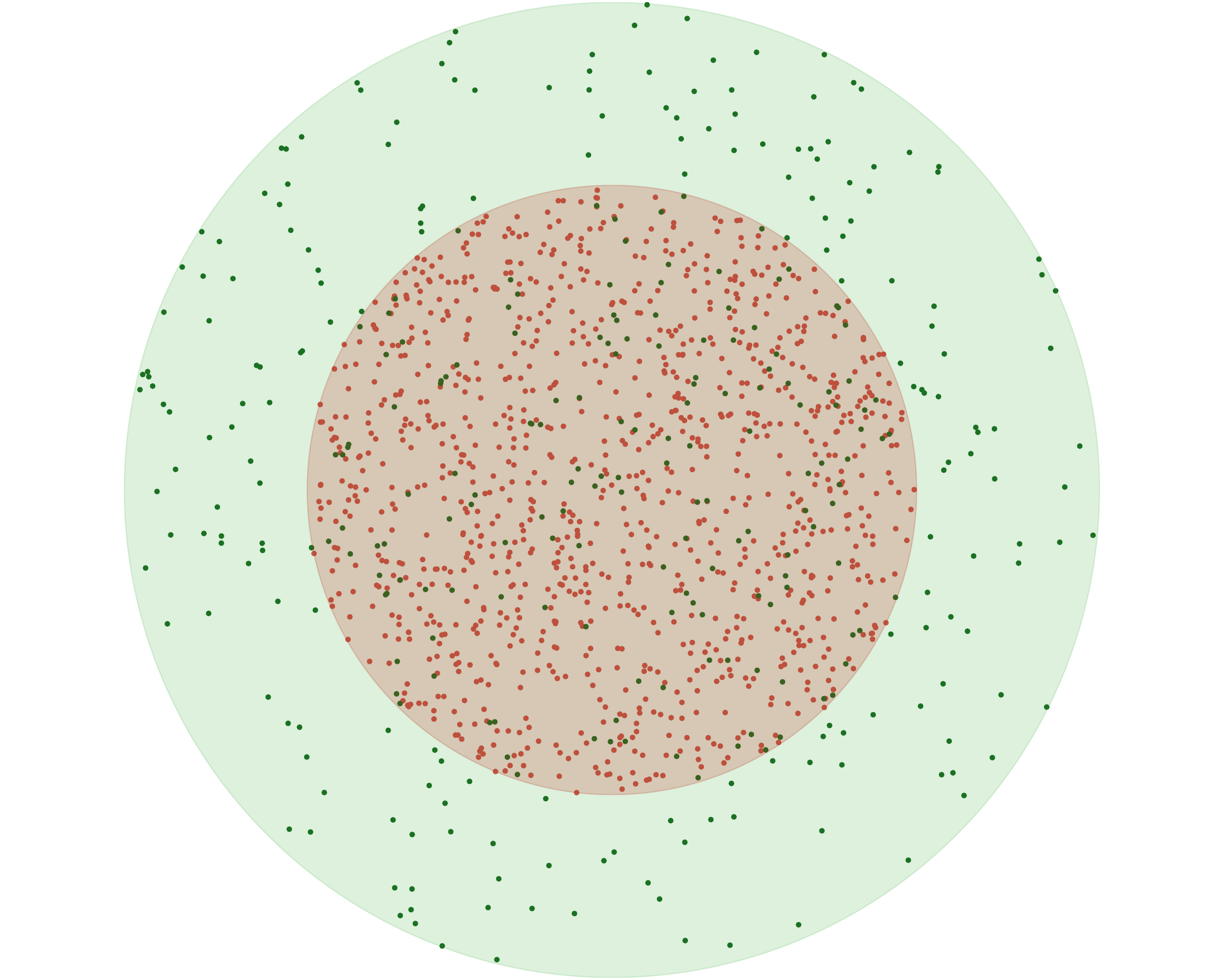}
	}
    \vspace{-1mm}
	\caption{Qualitative analyses of the criterion triple $(\cE, \cU, \cC)$: The red and green disks represent the uniform distributions of feature and code vectors, respectively.}
	\label{fig:criterion triple analysis}
    \vspace{-5mm}
\end{figure*}

\subsection{Theoretical Analyses} \label{sec:theoretical analysis}

In this section, we provide theoretical evidence to support our empirical observations. 
Let the code vectors $\{\be_k\}_{k=1}^K$ and  feature vectors $\{\bz_i\}_{i=1}^N$ be independently and identically drawn from $\cP_B$ and $\cP_A$, respectively.
We say a codebook $\{\be_k\}_{k=1}^K$ attains full utilization asymptotically with respect to  $\{\bz_i\}_{i=1}^N$ if the codebook utilization rate $\cU(\{\be_k\}_{k=1}^K;\{\bz_i\}_{i=1}^N)$ tends to 1 in probability as $N$ approaches infinity:
\begin{align}
\cU(\{\be_k\}_{k=1}^K;\{\bz_k\}_{i=1}^N)\overset{p}{\to}1,\quad \textnormal{as }N\to\infty.
\end{align}
For the codebook distribution $\cP_B$, we say it attains full utilization asymptotically with respect to $\cP_A$ if, with probability 1, the randomly generated codebook $\{\be_k\}_{k=1}^K$ achieves full utilization asymptotically.

Additionally, a codebook distribution $\cP_B$ is said to have vanishing quantization error asymptotically with respect to a domain $\Omega\subseteq\RR^d$ if the quantization error over all data of size $N$  tends to zero in probability as $K$ approaches infinity:
\begin{align}\label{equ:W25}
\sup_{\{\bz_i\}\subseteq\Omega}\cE(\{\be_k\}_{k=1}^K;\{\bz_i\}_{i=1}^N)\overset{p}{\to}0,\quad \textnormal{as }K\to\infty.
\end{align}
Our first theorem shows that $\overline{\supp(\cP_A)}=\overline{\supp(\cP_B)}$ is sufficient and necessary  for the codebook distribution $\cP_B$ to attain both full utilization and vanishing quantization error asymptotically. For simplicity, $\cP_A$ is assumed to have a density function $f_A$ with bounded support $\Omega\subseteq\RR^d$. 

\begin{theorem}\label{thm:1}
    Assume $\Omega=\supp(\cP_A)$ is a bounded open area. The codebook distribution $\cP_B$ attains full utilization and vanishing quantization error asymptotically if and only if $\overline{\supp(\cP_B)}=\overline{\supp(\cP_A)}$, where $\overline{\cS}$ denotes the closure of the set $\cS$.
\vspace{-1ex}
\end{theorem}
Theorem \ref{thm:1} establishes the optimal support of the codebook distribution. The boundedness of $\Omega$ is required as we consider the worst case quantization error in \eqref{equ:W25}. In real applications, when $\cP_A$ follows an absolutely continuous distribution over an unbounded domain, then $\{z_i\}_{i=1}^N$ generated from $\cP_A$ will be bounded with high probability. Thus, Theorem \ref{thm:1} also provides theoretical insights for a target distribution $\cP_A$ with an unbounded domain.

Besides the optimal support, we also determine the optimal density of the codebook distribution by invoking existing results characterizing asymptotic optimal quantizers \cite{graf2000foundations}. Specifically, we consider the case where $N$ approaches to infinity and define the expected quantization error of a codebook $\{\be_k\}$ with respect to $\cP_\cA$ as 
\begin{align}
\cE(\{\be_k\}_{k=1}^K;\cP_A)=\EE_{\bz\sim\cP_A}\min_{\be\in\{\be_k\}}\norm{\bz-\be}^2.
\end{align}
A codebook $\{\be^*_k\}_{k=1}^K$ is called the set of optimal centers for $\cP_A$ if it achieves the minimal quantization error:
\begin{align}
\cE(\{\be^*_k\}_{k=1}^K;\cP_A)=\min_{\{\be_k\}_{k=1}^K}\cE(\{\be_k\}_{k=1}^K;\cP_A).
\end{align}
Intuitively, Theorem~\ref{thm:1} shows that mismatched supports inevitably lead to either unused code vectors or large quantization error, whereas matching supports is both necessary and sufficient to avoid these failure modes.

Theorem \ref{thm:2} demonstrates that,  under weak regularity conditions,  the empirical measure of the optimal centers for $\cP_A$ converges in distribution to a fixed distribution determined by $\cP_A$. Notably, we do not assume a bounded domain in the following theorem.

\begin{theorem}[Theorem 7.5, \cite{graf2000foundations}]\label{thm:2} 
Suppose $Z\sim\cP_A$ is absolutely continuous w.r.t. the Lesbegue measure in $\RR^d$ and $\EE\norm{Z}^{2+\delta}<\infty$ for some $\delta>0$.
Then the empirical measure of the optimal centers for $\cP_A$, 
\begin{align}
\frac{1}{K}\sum_{k=1}^K\delta_{\be^*_k},
\end{align}
converges weakly to a fixed distribution $\cP_A^*$, whose density function $f_A^*$ is proportional to $f_A^{(d+2)/d}$. 
\end{theorem}

\paragraph{High-dimensional Implication.}
Theorem~\ref{thm:2} characterizes the asymptotically optimal codebook distribution $\mathcal{P}^*_A$, whose density satisfies
$f^*_A(z) \propto f_A(z)^{(d+2)/d}$.
A key implication arises in high-dimensional settings typical of visual tokenization: as the feature dimension $d$ increases, the exponent converges to unity,
\begin{align}
\lim_{d\to\infty} \frac{d+2}{d} = \lim_{d\to\infty} \left(1+\frac{2}{d}\right) = 1.
\end{align}
As a result, the optimal codebook density $f^*_A$ becomes increasingly close to the feature density $f_A$ in the high-dimensional regime.
This observation provides a rigorous theoretical justification for our approach: explicitly aligning the codebook distribution with the feature distribution (i.e., $\mathcal{P}_B \approx \mathcal{P}_A$) serves as a principled and effective surrogate for the asymptotically optimal design in high-dimensional latent spaces.



%% file: sections/method.tex
\section{Methodology}
\label{sec:method}



Based on the theoretical insights in Section~\ref{sec:understanding}, which show that matching the feature and codebook distributions is both necessary and sufficient for asymptotically optimal quantization, we propose a general framework for enhancing vector quantization via distributional matching. 
This framework can be instantiated using either parametric objectives (e.g., the Wasserstein distance for efficiency) or non-parametric alternatives (e.g., Maximum Mean Discrepancy, which offers robustness and adaptability to general distributions). In the main text, we focus primarily on the Wasserstein instantiation due to its computational efficiency arising from a closed-form solution. Finally, we integrate this objective into two representative VQ frameworks, namely VQ-VAE~\cite{Oord2017NeuralDR} and VQGAN~\cite{Esser2020TamingTF}.

\subsection{A General Distributional Matching Framework}
\label{sec:distributional matching framework}
To resolve the fundamental mismatch between the feature distribution $\mathcal{P}_A$ and the codebook distribution $\mathcal{P}_B$, we introduce an auxiliary alignment loss $\mathcal{L}_{match}$ to the standard VQ objective. The goal is to minimize the divergence $\mathcal{D}$ between the two distributions:
\begin{equation}
    \mathcal{L}_{match} = \mathcal{D}(\mathcal{P}_A, \mathcal{P}_B),
\end{equation}
where $\mathcal{P}_A$ is empirically defined by the encoder outputs $\{\bm{z}_e\}$ and $\mathcal{P}_B$ by code vectors $\{\bm{e}_k\}$. Crucially, gradients from $\mathcal{L}_{match}$ are only back-propagated to the codebook, preserving encoder feature expressiveness. 
\vspace{-2ex}

\begin{figure*}[t]
    \centering
    \includegraphics[width=\linewidth]{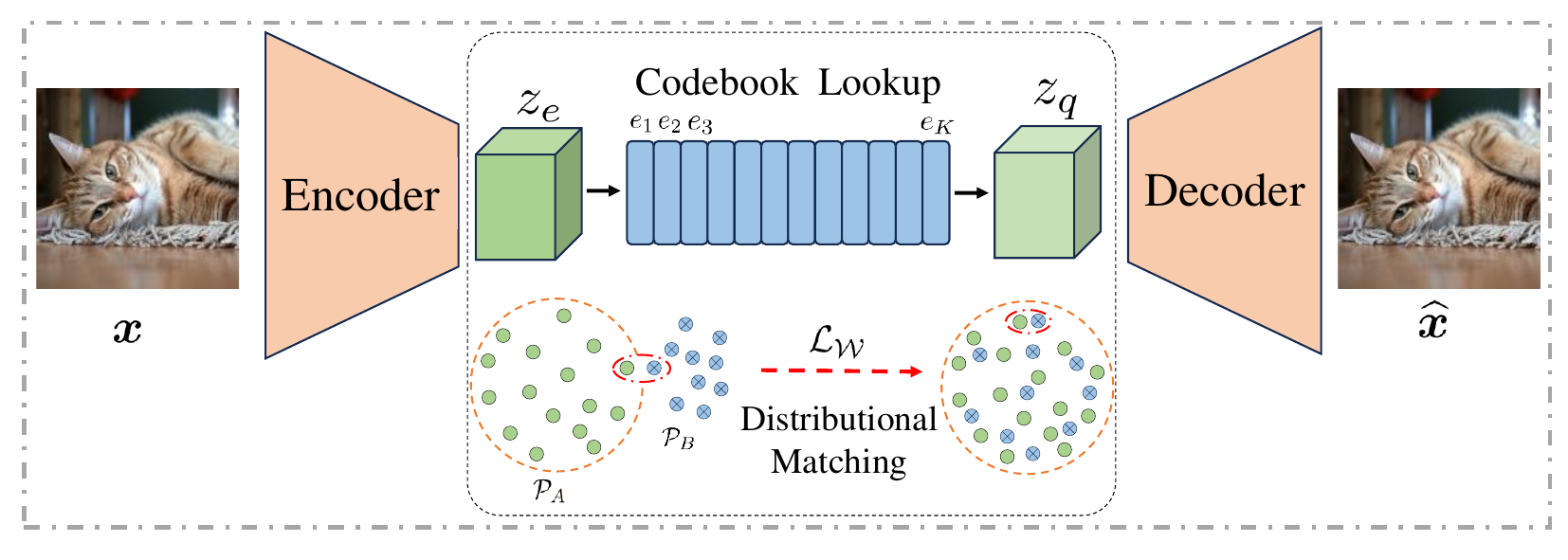}
    \vspace{-4ex}
    \caption{Illustration of the \emph{Wasserstein VQ}. The architecture integrates an encoder-decoder network with a  VQ module. In the VQ module, we augment the vanilla VQ framework~\cite{Oord2017NeuralDR} by incorporating our proposed Wasserstein loss $\mathcal{L}_{\mathcal{W}}$ to achieve distributional matching between features $\bz_e$ ($\bz_e^{ij}\sim \mathcal{P}_{A}$) and the codebook $\be_k$ ($\be_k\sim \mathcal{P}_{B}$).}
    \label{fig:wasserstein vq}
    \vspace{-3ex}
\end{figure*}
\subsection{Wasserstein-Based Distribution Matching}
\label{sec:wasserstein distance}
We consider a parametric instantiation of the proposed framework using a mild Gaussian approximation to enable computationally efficient distributional matching.
Specifically, we approximate the feature and codebook distributions as $\mathcal{P}_A \approx \mathcal{N}(\bm{\mu}_1, \bm{\Sigma}_1)$ and $\mathcal{P}_B \approx \mathcal{N}(\bm{\mu}_2, \bm{\Sigma}_2)$.
Under this approximation, we employ the quadratic Wasserstein distance, as defined in Appendix~\ref{appendix: distribution distance}, which admits a computationally efficient closed-form solution.
Although other statistical distances, such as the Kullback-Leibler divergence \cite{Kingma2013AutoEncodingVB, Ho2020DenoisingDP}, are viable alternatives, they typically lack simple closed-form representations, rendering them computationally prohibitive in high-dimensional settings. 
Consequently, we adopt the quadratic Wasserstein distance, whose closed-form representation for Gaussian distributions is provided by the following lemma:

\begin{lemma}[\cite{Olkin1982TheDB}]\label{theorem:Wasserstein Distance}
The quadratic Wasserstein distance between $\mathcal{N}(\bm \mu_1, \m \Sigma_1)$ and $\mathcal{N}(\bm \mu_2, \m \Sigma_2)$ 
\begin{equation}\label{eq:wasserstein distance definition}
\sqrt{\Vert \bm \mu_1 - \bm \mu_2 \Vert^2_2 + \tr( {\m \Sigma_1}+ {\m \Sigma_2} - 2 (\Sigma_1^{\frac{1}{2}} {\m \Sigma_2} \m \Sigma_1^{\frac{1}{2}})^{\frac{1}{2}})}.
\end{equation}
\end{lemma}
The lemma shows that the quadratic Wasserstein distance admits a closed-form expression in terms of the means and covariance matrices of the two distributions.
In practice, we estimate these population quantities, ${\bm \mu}_{1}$, ${\bm \mu}_{2}$, ${\m \Sigma}_{1}$, and ${\m \Sigma}_{2}$, with their sample counterparts: $\hat {\bm \mu}_{1}$, $\hat {\bm \mu}_{2}$, $\hat {\m \Sigma}_{1}$, and $\hat {\m \Sigma}_{2}$. 
The empirical quadratic Wasserstein distance is then used as the optimization objective to align  the feature and codebook distributions:
\begin{equation}\label{eq:wasserstein distance}
\mathcal{L}_{\mathcal{W}}\! = \!\sqrt{\!\Vert \hat {\bm \mu}_{1} \!- \!\hat {\bm \mu}_{2} \Vert^2_2 \!+\! \tr( \hat {\m \Sigma}_{1}\!+\!\hat {\m \Sigma}_{2}\! -\!2(\hat {\m \Sigma}_{1}^{\frac{1}{2}} {\hat {\m \Sigma}_{2}} \hat {\m \Sigma}_{1}^{\frac{1}{2}})^{\frac{1}{2}})}. \nonumber
\end{equation}  
A smaller value of $\mathcal{L}_{\mathcal{W}}$ indicates stronger alignment between the feature distribution $\mathcal{P}_A$ and the codebook distribution $\mathcal{P}_B$. We refer to the VQ algorithm that employs $\mathcal{L}_{\mathcal{W}}$ as \emph{Wasserstein VQ}.

\paragraph{Remark on the Gaussian Approximation}
The Gaussian approximation is introduced solely to obtain a closed-form and computationally efficient instantiation of the proposed distributional matching objective, simplifying the complex form in Definition~\ref{def:wasserstien distance} in Appendix~\ref{appendix: distribution distance} into Lemma~\ref{theorem:Wasserstein Distance}.
Importantly, this approximation does not constrain the learned feature representations:
the Wasserstein loss $\mathcal{L}_{\mathcal{W}}$ is applied only to update the codebook parameters, while gradients to the encoder are detached.
As a result, the encoder remains free to learn arbitrarily complex feature distributions.

In practice, latent representations in visual tokenization often exhibit approximately Gaussian statistics, a behavior commonly observed in deep representations (see Section~\ref{sec:experiments} and Appendix~\ref{sec:normality_tests} for empirical validation). One possible explanation is that multivariate normal distributions are information-theoretically advantageous, as they maximize entropy for a given covariance, thereby enabling efficient information transmission and robust signal reconstruction. As a result, multivariate normal distributions naturally emerge as both an ideal and empirically observed form for latent representations in visual tokenization and related signal recovery tasks.

\subsection{MMD-Based Distribution Matching}
\label{sec:mmd}

To address scenarios where the Gaussian assumption may not hold (e.g., highly multi-modal distributions), we provide a non-parametric instantiation using Maximum Mean Discrepancy (MMD)~\cite{Gretton2012AKT, Sriperumbudur2009HilbertSE}. MMD measures the distance between distributions in a Reproducing Kernel Hilbert Space (RKHS) without assuming any specific parametric form:
\begin{align}
\mathcal{L}_{\text{MMD}} = \frac{1}{N^2}\sum_{i,j} k(\bm{z}_i, \bm{z}_j) - \frac{2}{NK}\sum_{i,k} k(\bm{z}_i, \bm{e}_k) + \frac{1}{K^2}\sum_{k,l} k(\bm{e}_k, \bm{e}_l),
\end{align}
where $\bm{e}_k$ and $\bm{z}_i$ denote code vectors and encoder output spatial feature vectors, respectively, and $k(\cdot, \cdot)$ is a kernel function (e.g., a Gaussian RBF kernel). We refer to this approach as MMD VQ~\cite{Fang2025VQTransplant}. Similar to Wasserstein VQ, the MMD loss $\mathcal{L}_{\text{MMD}}$ is applied exclusively to update the codebook parameters, with gradients to the encoder detached.

\paragraph{Remark} While $\mathcal{L}_{\text{MMD}}$ offers greater theoretical robustness by relaxing the Gaussian hypothesis, it typically incurs higher computational cost ($\mathcal{O}((N+K)^2)$) compared to the closed-form Wasserstein distance. In our experiments, we observe that Wasserstein VQ achieves performance comparable to MMD VQ (see Section~\ref{sec:mmd-vs-wasserstein}), indicating that the Gaussian approximation sufficiently captures the essential structure for effective tokenization while remaining computationally efficient. Therefore, In the following experiments, we focus on Wasserstein instantiation for exploring the distributional matching framework.


\subsection{Integration into VQ Architectures}
Our distributional matching framework is agnostic to the underlying VQ architecture. We integrate it into two representative frameworks: VQ-VAE and VQGAN.

\subsubsection{VQ-VAE Integration} 
\label{sec:VQVAE}
We first examine \emph{Wasserstein VQ} within the VQ-VAE framework~\cite{Oord2017NeuralDR}. As shown in Figure~\ref{fig:wasserstein vq}, the VQ-VAE model has three key components: an encoder $E(\cdot)$, a decoder $D(\cdot)$, and a quantizer $\mathcal{Q}(\cdot)$ with a learnable codebook $\{\m e_k\}_{k=1}^{K}$. As described earlier in Section~\ref{sec:overview vq}, for an input image $\bm{x}$, the encoder produces a spatial feature $\bm{z}_e = E(\bm{x}) \in \mathbb{R}^{h\times w \times d}$. The quantizer maps $\bm{z}_e$ to a quantized feature $\bm{z}_q$, which the decoder uses to reconstruct the image as $\bm{\hat{x}} = D(\bm{z}_q)$. Incorporating our Wasserstein loss $\mathcal{L}_{\mathcal{W}}$ into the VQ-VAE framework, the overall loss is formulated as follows:
\begin{align}
\label{eq:vqvae}
\mathcal{L}_{\text{VQ-VAE}} =\| \bm{\hat{x}} - \bm{x} \|^{2}_2 + \beta \| \textrm{sg}(\bm{z}_q)-\bm{z}_e \|^{2}_2 + \| \textrm{sg}(\bm{z}_e) -\bm{z}_q \|^{2}_2 + \gamma \mathcal{L}_{\mathcal{W}}.
\end{align} 
where $\textrm{sg}$ denotes the stop-gradient. $\beta$ and $\gamma$ are hyper-parameters. We set $\gamma=0.5$ for all VQ-VAE experiments.

By incorporating $\mathcal{L}_{\mathcal{W}}$, the codebook is encouraged to globally match the first- and second-order statistics of the feature distribution, complementing the local assignment enforced by the standard VQ loss.

\paragraph{Remark} 
Sections~\ref{sec:effects of distribution matching} and~\ref{sec:theoretical analysis} provide empirical and theoretical evidence that minimizing quantization error is closely linked to aligning the codebook with the feature distribution. By encouraging the codebook to reflect the feature space structure and density, the alignment loss $\mathcal{L}_{\mathcal{W}}$ positions code vectors as cluster centers, minimizing the average squared distance between feature vectors and their assigned codes, $\| \bm{z}_e - \bm{z}_q \|^{2}_2$. In essence, distribution alignment arranges prototypes to cover the feature manifold and fill gaps, reducing overall quantization error.

\vspace{2ex}
\subsubsection{VQGAN Integration}
\label{sec:VQGAN}
To ensure high perceptual quality in the reconstructed images, we further investigate \emph{Wasserstein VQ} within the VQGAN framework~\cite{Esser2020TamingTF}. VQGAN extends the VQ-VAE framework by integrating a VGG network~\cite{Simonyan2014VeryDC}  and a patch-based discriminator \cite{Esser2020TamingTF, Johnson2016PerceptualLF}. The overall training objective of VQGAN can be written as follows:
\begin{align}
\label{eq:vqgan}
\mathcal{L}_{\text{VQGAN}} = \mathcal{L}_{\text{VQ-VAE}} + \mathcal{L}_{\text{Per}} + \lambda \mathcal{L}_{\text{GAN}}.
\end{align}
Where $\mathcal{L}_{\text{Per}}$ and $\mathcal{L}_{\text{GAN}}$ denote the VGG-based perceptual loss~\cite{Zhang2018TheUE} and the GAN loss~\cite{Isola2016ImagetoImageTW, Lim2017GeometricG}, respectively. Achieving state-of-the-art reconstruction fidelity via adversarial training typically incurs substantial computational cost. 
To reduce training overhead, we adopt the VQ-Transplant framework, which attains competitive reconstruction fidelity at a fraction of the cost~\cite{Fang2025VQTransplant}. 
We adopt the VQ-Transplant framework solely as a training-efficient backbone; the proposed Wasserstein VQ is orthogonal to this choice and can be integrated into standard VQGAN training as well.
Specifically, it initializes the encoder and decoder from a state-of-the-art pretrained tokenizer (e.g., VAR~\cite{Tian2024VisualAM}) rather than training from scratch, significantly lowering training cost while preserving reconstruction quality.



%% file: sections/atomic_setting.tex
\section{An Atomic Setting for Evaluating the Criterion Triple}
\label{sec:atomic setting}
\begin{figure*}[!t]
    \vspace{-1ex}
    \centering
    \subfloat[$\cE$ w.r.t. $\mu$]{         
        \label{fig:vq quantization error mean}  \includegraphics[width=0.225\textwidth]{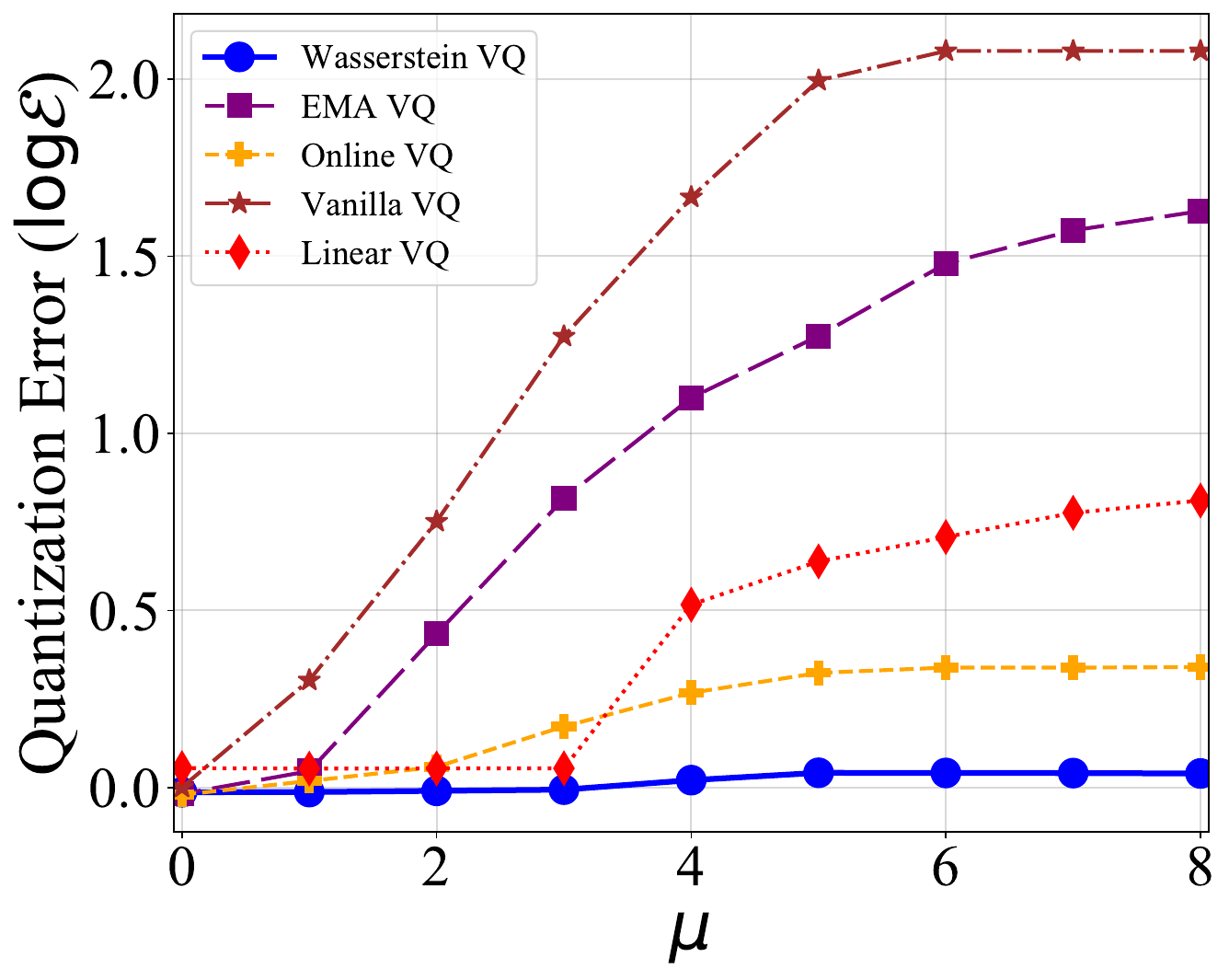}
        }
	\subfloat[$\cU$ w.r.t. $\mu$]{         
        \label{fig:vq codebook utilization mean}  \includegraphics[width=0.225\textwidth]{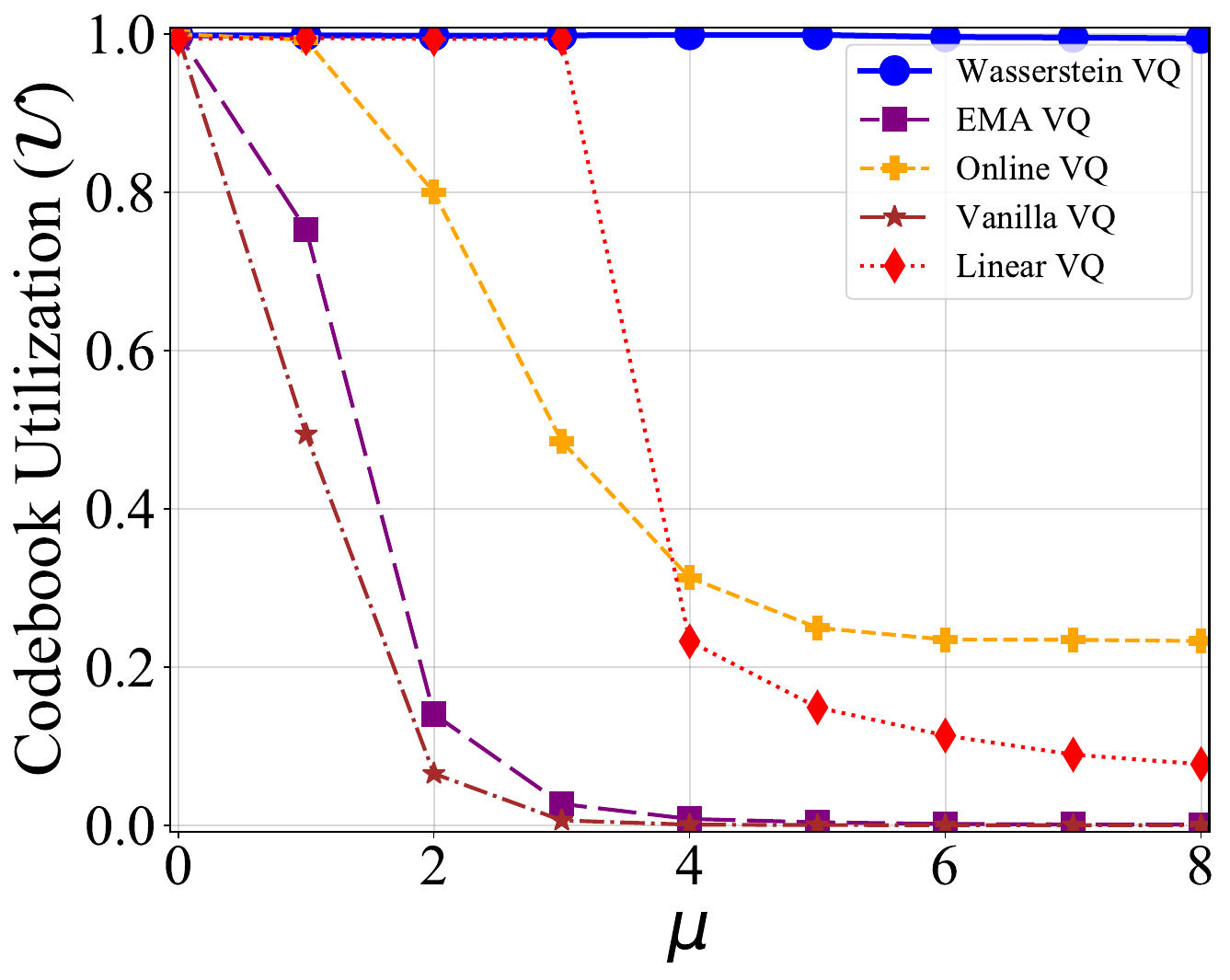}
        }
	\subfloat[$\mathcal{C}$ w.r.t. $\mu$]{         
        \label{fig:vq codebook perplexity mean}  \includegraphics[width=0.225\textwidth]{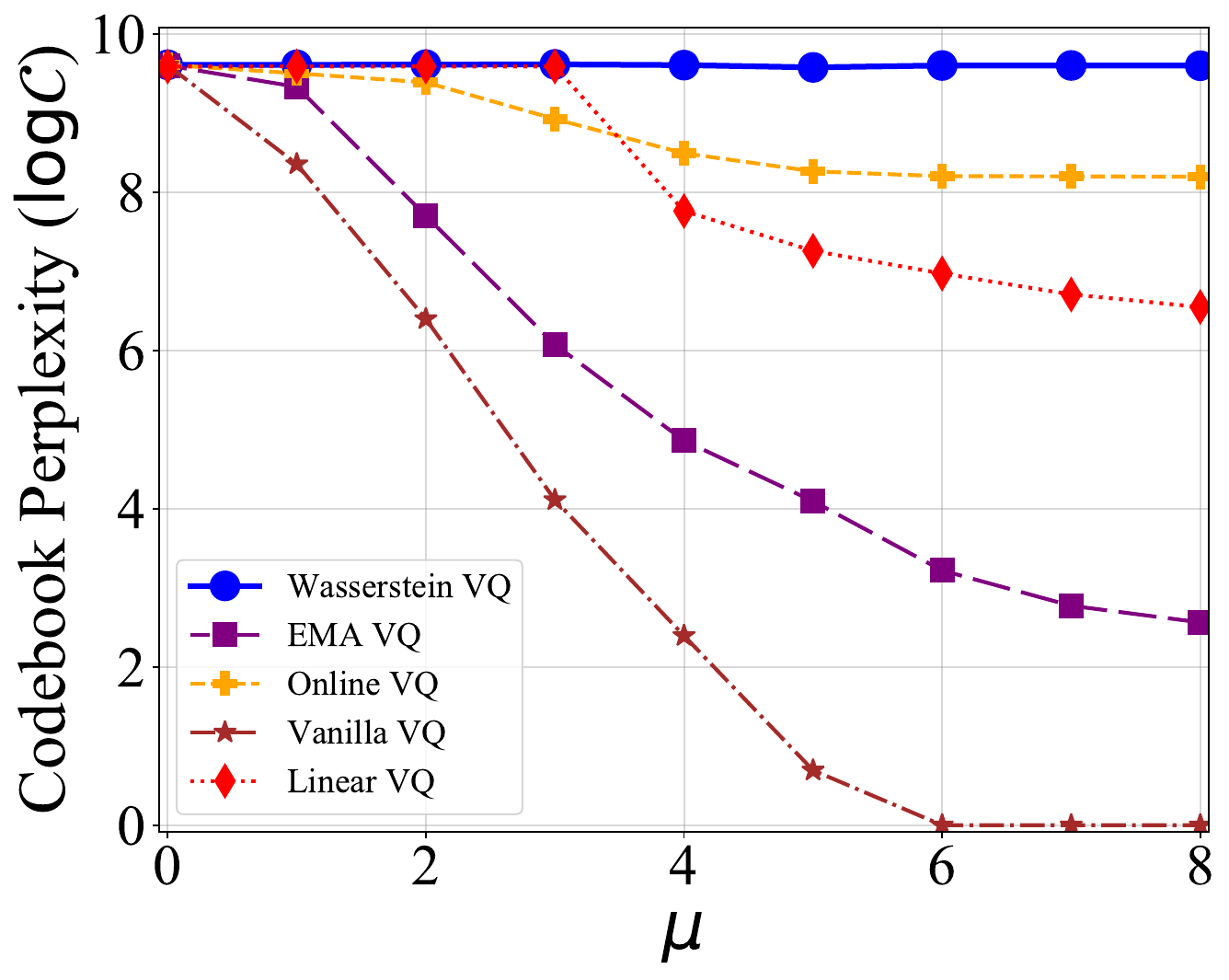}
        }
    \subfloat[$\mathcal{L}_{W}$ w.r.t. $\mu$]{         
        \label{fig:wasserstein distance mean}  \includegraphics[width=0.225\textwidth]{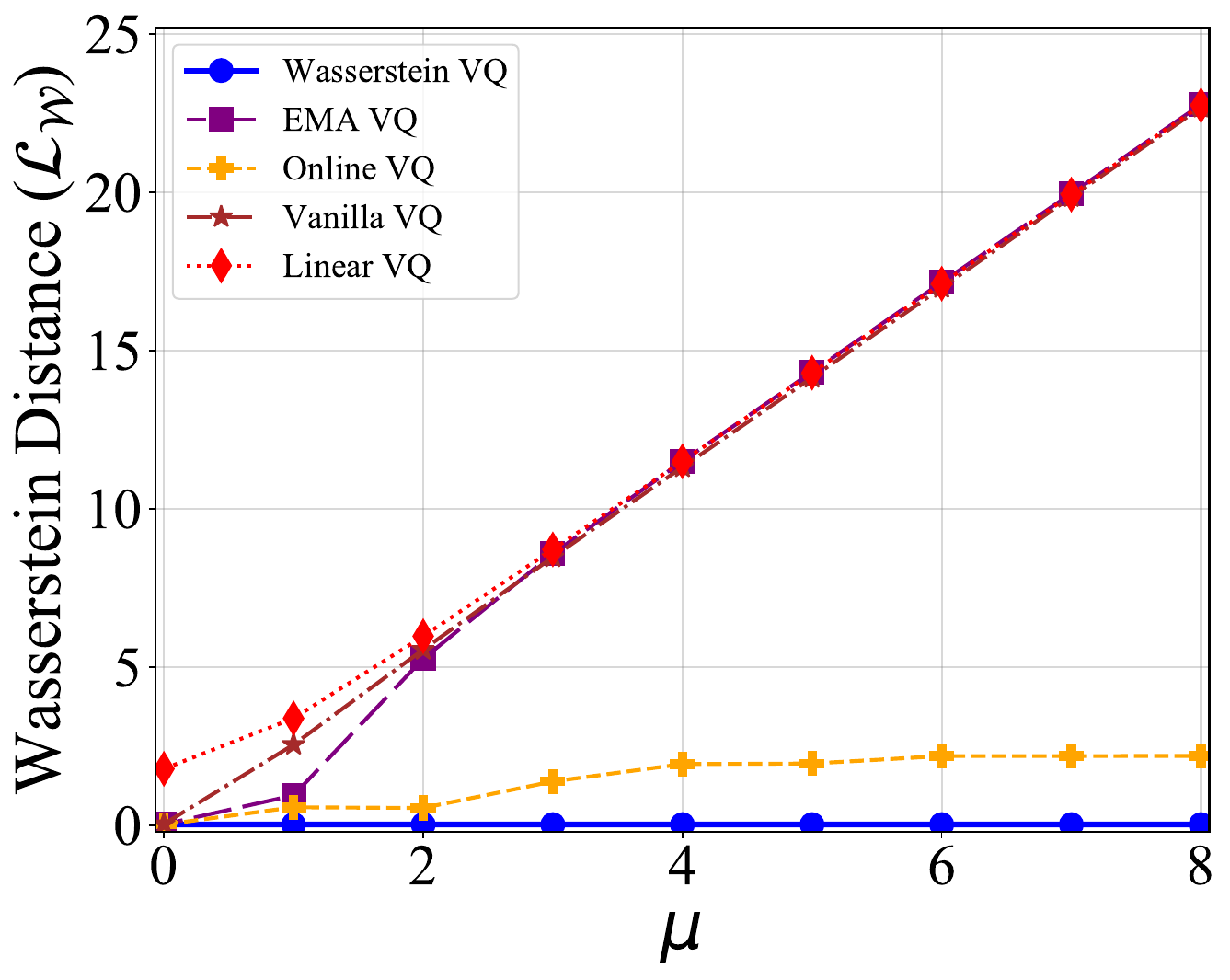}
        }
    \vspace{-0.1mm}
    \subfloat[$\cE$ w.r.t. $\nu$]{         
        \label{fig:vq quantization error mean(uniform)}  \includegraphics[width=0.225\textwidth]{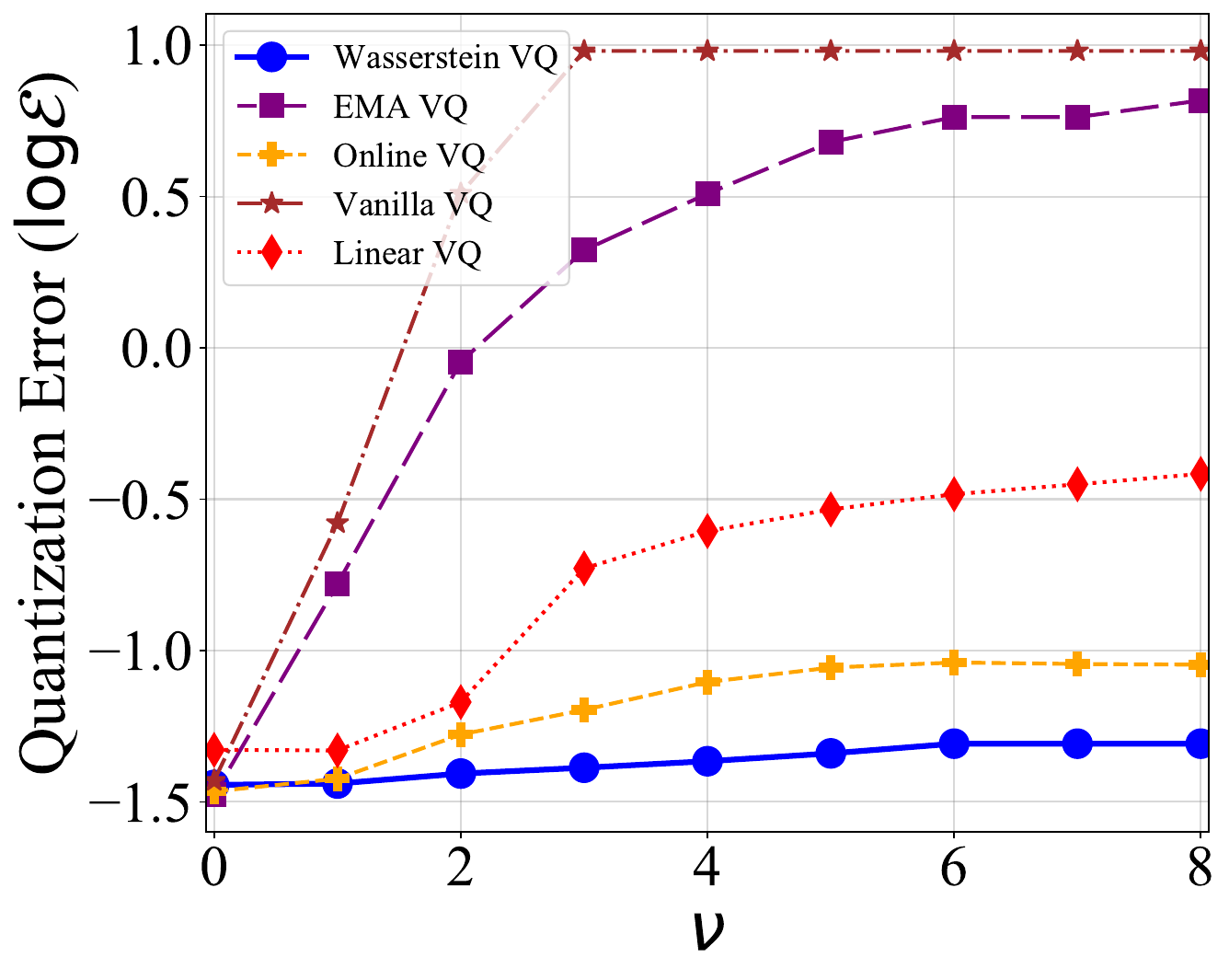}
        }
	\subfloat[$\cU$ w.r.t. $\nu$]{         
        \label{fig:vq codebook utilization mean(uniform)}  \includegraphics[width=0.225\textwidth]{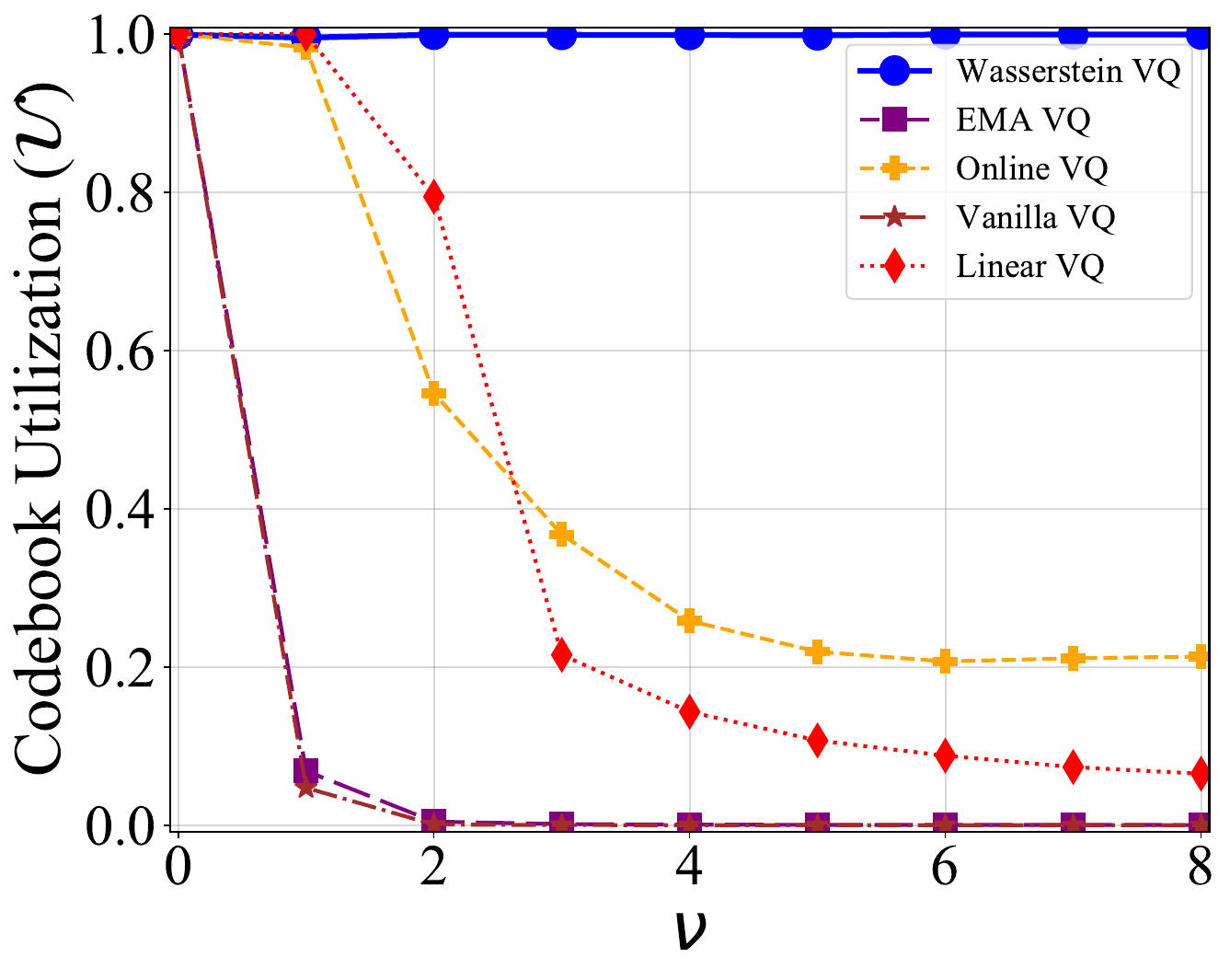}
        }
	\subfloat[$\mathcal{C}$ w.r.t. $\nu$]{         
        \label{fig:vq codebook perplexity mean(uniform)}  \includegraphics[width=0.225\textwidth]{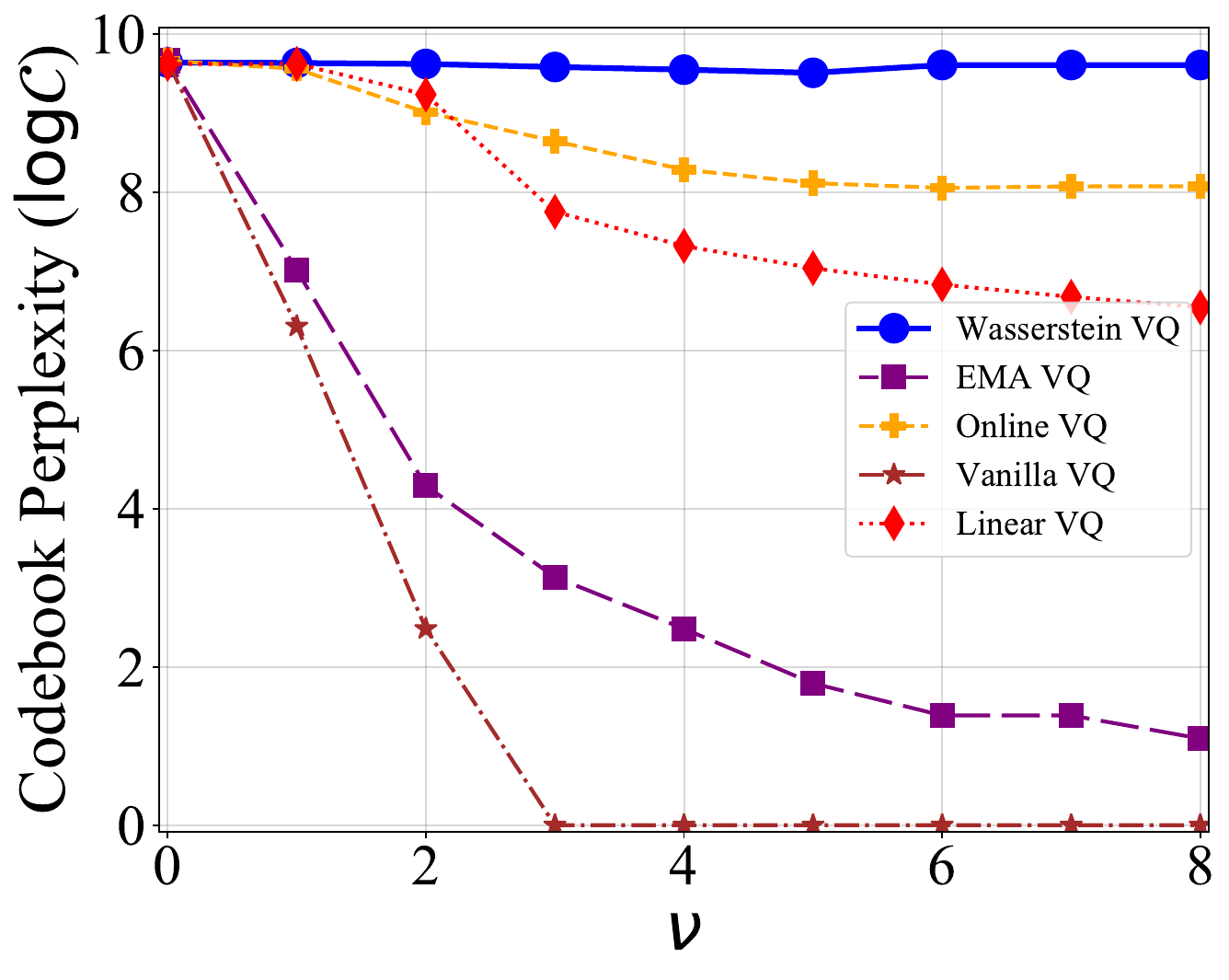}
        }
    \subfloat[$\mathcal{L}_{W}$ w.r.t. $\nu$]{     
        \label{fig:wasserstein distance mean(uniform)}  \includegraphics[width=0.225\textwidth]{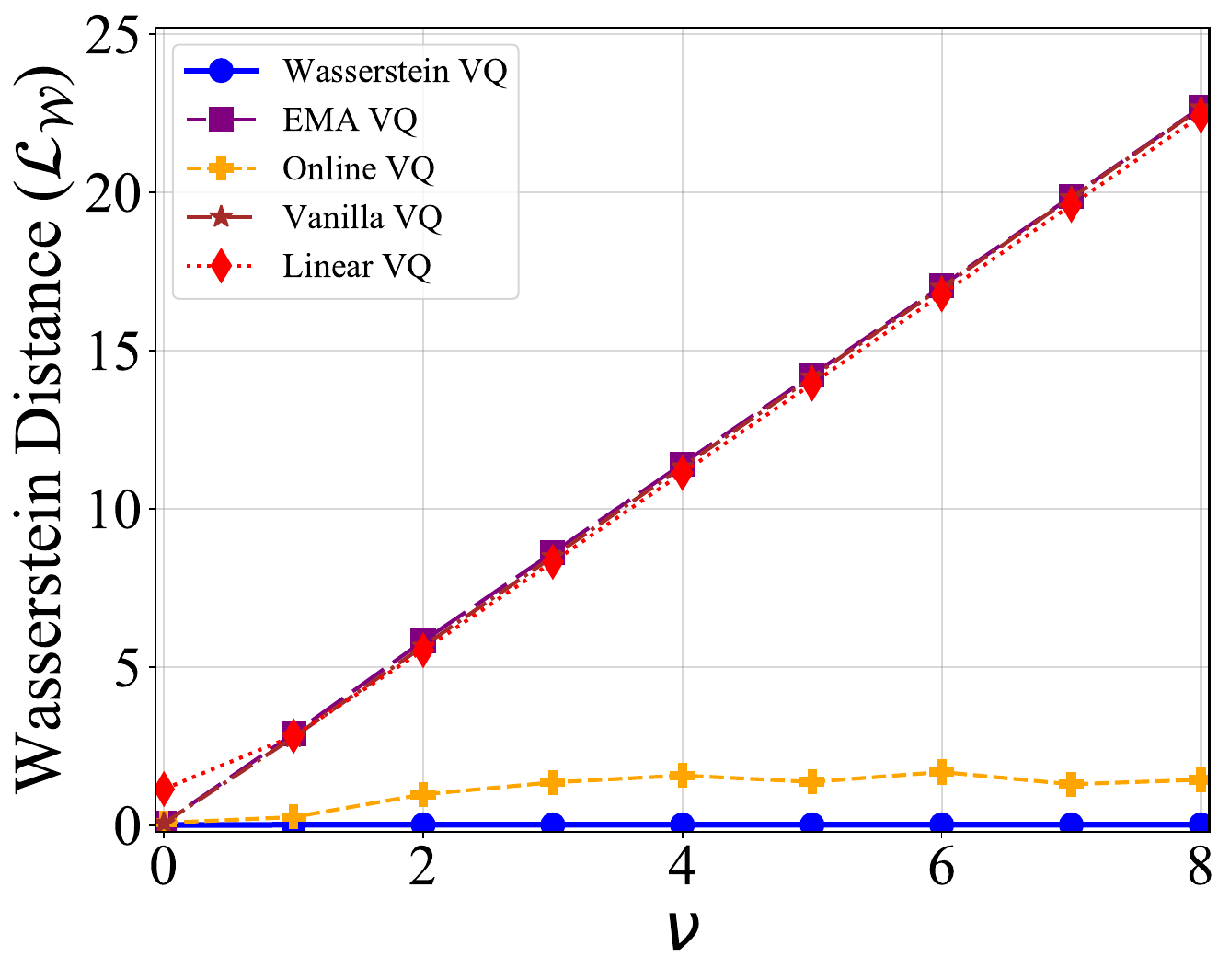}
        }
    \vspace{-2pt}
	\caption{The performance metrics $(\cE, \cU, \cC)$  for various VQ approaches. For panels (a) to (d), the codebook distribution is initialized as a Gaussian distribution, while for panels (e) to (h), the codebook distribution is initialized as a uniform distribution.}
	\label{fig:analysis on criterion triple (VQ)}
    \vspace{-3ex}
\end{figure*}

As discussed in Section~\ref{sec:distributional formulation}, the variance of the latent distribution has a substantial impact on $\mathcal{E}$. Therefore, a fair evaluation of the criterion triple should be conducted under identical latent distributions. However, in practical scenarios, either in VQ-VAE~\cite{Oord2017NeuralDR} or VQGAN~\cite{Esser2020TamingTF}, the encoder outputs exhibit significantly different distributions across quantization methods due to intrinsic differences in their algorithmic mechanisms. Moreover, these distributional discrepancies are not static but evolve dynamically throughout VQ training. As a result, directly comparing quantization errors fails to accurately reflect the intrinsic effectiveness of different VQ algorithms and may even lead to misleading conclusions. To address this, we introduce a controlled simulated experimental setting that enables a principled evaluation of the intrinsic behavior of VQ algorithms.

Specifically, we assume that the encoder output feature distributions of all VQ methods follow the same Gaussian distribution, i.e., $\bz_i \sim \mathcal{N}_d(\mu \cdot \bm 1, \bm I_d)$. Although real-world feature distributions are often more complex, this simplification isolates the effect of the quantization mechanism itself and allows for controlled and comparable comparisons across methods. For the codebook distribution, we initialize all VQ methods with the same standard Gaussian distribution (by sampling $\be_k \sim \mathcal{N}_d(\bm 0, \bm I_d)$), ensuring that the distribution variance is identical across all methods.


Our baselines include Vanilla VQ~\cite{Oord2017NeuralDR}, EMA VQ~\cite{Razavi2019GeneratingDH}, Online VQ~\cite{Zheng2023OnlineCC}, and Linear VQ, which employs a linear projection layer with frozen code vectors~\cite{Zhu2024ScalingTC,Zhu2024AddressingRC}. For all VQ algorithms except Linear VQ, the sampled code vectors are treated as trainable parameters and optimized according to their respective update rules. In the case of Linear VQ, we focus on training only the linear projection layer. Detailed experimental specifications are provided in Appendix~\ref{appendix:details part2}.

As visualized in Figures~\ref{fig:vq quantization error mean} to~\ref{fig:vq codebook perplexity mean}, \emph{Wasserstein VQ} outperforms all baselines in terms of the criterion triplet $(\cE, \cU, \mathcal{C})$, particularly when the feature distribution and the initialized codebook distribution exhibit large deviations. While existing VQ methods perform well under the idealized setting of $\mu=0$, this scenario is unrealistic: in practice, feature distributions are diverse and continuously evolving, making it infeasible to perfectly match the codebook distribution by codebook initialization. When a large initial distribution gap exists (e.g., $\mu=5$), existing methods fail to achieve effective distribution alignment and perform poorly, as depicted in Figure~\ref{fig:wasserstein distance mean}, highlighting their strong dependence on codebook initialization. In contrast, \emph{Wasserstein VQ} overcomes this limitation through explicit distributional matching regularization, thereby achieving superior performance across all criterion metrics.



We observe the same behavior when the feature and codebook distributions are uniform, with feature vectors sampled from $\unif_d(\nu-1, \nu+1)$ and code vectors initialized from $\unif_d(-1,1)$. As shown in Figures~\ref{fig:vq quantization error mean(uniform)} to~\ref{fig:wasserstein distance mean(uniform)}, \emph{Wasserstein VQ} again performs the best, suggesting that its effectiveness extends beyond Gaussian assumptions and exhibits distribution-agnostic behavior.

%% file: sections/experiment.tex
\section{Experiments}\label{sec:experiments}
In this section, we empirically demonstrate the effectiveness of the proposed distributional matching framework in visual tokenization tasks.
Experiments are conducted within the VQ-VAE~\cite{Oord2017NeuralDR} and VQGAN~\cite{Esser2020TamingTF} frameworks. 

\subsection{Evaluation on VQ-VAE Framework}
\label{sec:vqvae experiments}
\paragraph{Datasets and Baselines} Experiments are conducted on four benchmark datasets: two low-resolution datasets, i.e., CIFAR-10~\cite{Krizhevsky2009LearningML} and SVHN~\cite{Netzer2011ReadingDI},  and two high-resolution datasets FFHQ~\cite{Karras2018ASG} and ImageNet~\cite{Deng2009ImageNetAL}. We evaluate our approach against representative VQ methods:  Vanilla VQ~\cite{Oord2017NeuralDR}, EMA VQ~\cite{Razavi2019GeneratingDH}, which uses exponential moving average updates and is also referred to as $k$-means, Online VQ, which employs $k$-means++ in CVQ-VAE~\cite{Zheng2023OnlineCC}, and enhanced  straight-through estimator STE++~\cite{Huh2023StraighteningOT}. For experimental settings, see Appendix~\ref{appendix:experimental details}. 

\paragraph{Metrics}
We employ multiple evaluation metrics, including the codebook utilization rate ($\mathcal{U}$), codebook perplexity ($\mathcal{C}$), peak signal-to-noise ratio (PSNR), patch-level structural similarity index (SSIM), and pixel-level reconstruction loss (Rec. Loss).
We exclude the raw quantization error ($\mathcal{E}$) from the main VQ-VAE experiments, as $\mathcal{E}$ is highly sensitive to feature variance, which is not controlled under different training dynamics. Consequently, a direct comparison of $\mathcal{E}$ does not faithfully reflect the intrinsic effectiveness of VQ algorithms, as discussed in Section~\ref{sec:distributional formulation} and Section~\ref{sec:atomic setting}.


\begin{table*}[!t]
  \centering
  \vspace{-1ex}
  \caption{Comparison of VQ-VAEs trained on FFHQ dataset following~\cite{Oord2017NeuralDR}.}
  \label{tab:vqvae ffhq}
  \vspace{-1ex}
  \resizebox{0.9\textwidth}{!}{%
  \begin{tabular}{@{}lccccccc@{}} \hline
    Approaches & Tokens & Codebook Size & $\mathcal{U}$ ({\color{green!60!black}\bfseries$\pmb{\uparrow}$}) & $\mathcal{C}$ ({\color{green!60!black}\bfseries$\pmb{\uparrow}$}) & PSNR({\color{green!60!black}\bfseries$\pmb{\uparrow}$}) & SSIM({\color{green!60!black}\bfseries$\pmb{\uparrow}$}) & Rec. Loss({\color{green!60!black}\bfseries$\pmb{\downarrow}$}) \\\hline
    Vanilla VQ  & 256 & $16384$ & 3.8\% & 527.2 & 27.83 & 73.8 & 0.0119 \\ 				
    STE++ & 256 & $16384$ & 3.4\% & 476.7 & 27.54 & 72.3 & 0.0129 \\
    EMA VQ & 256 & $16384$ & 14.0\% & 1795.7 & 28.39 & 74.8 & 0.0106 \\
    Online VQ  & 256 & $16384$ & 11.7\% & 1115.3 & 27.68 & 72.6 & 0.0125 \\
    \textbf{Wasserstein VQ} & 256 & $16384$ & \textbf{100\%} &  \textbf{15713.3} & \textbf{29.03} & \textbf{76.6} & \textbf{0.0093} \\\hline
    Vanilla VQ & 256 & $50000$ & 1.2\% & 516.8 & 27.83 & 73.6 & 0.0120  \\ 				
    STE++ & 256 & $50000$ & 1.0\% & 447.2 & 27.49 & 72.4 & 0.0131\\
    EMA VQ & 256 & $50000$ & 10.3\% & 4075.7 & 28.61 & 75.3 & 0.0101 \\
    Online VQ & 256 & $50000$ & 6.0\% & 1642.9 & 28.37 & 74.6 & 0.0107  \\
    \textbf{Wasserstein VQ} & 256 & $50000$ & \textbf{100\%} & \textbf{47496.4} & \textbf{29.24} & \textbf{77.0} & \textbf{0.0089} \\\hline
    Vanilla VQ  & 256 & $100000$ & 0.6\% & 481.0 & 27.86 & 74.2 & 0.0118  \\ 				
    STE++ & 256 & $100000$ & 0.5\% & 450.7 & 27.52 & 72.4 & 0.0130\\
    EMA VQ & 256 & $100000$ & 2.7\% & 2087.5 & 28.43 & 74.8 & 0.0105  \\
    Online VQ & 256 & $100000$ & 3.6\% & 1556.8 & 27.12 & 71.1 & 0.0142 \\
    \textbf{Wasserstein VQ} & 256 & $100000$ & \textbf{100\%} & \textbf{93152.7} & \textbf{29.53} & \textbf{78.0} & \textbf{0.0083} \\\hline
  \end{tabular}
  \vspace{-4ex}
}
\end{table*}

\begin{table*}[!t]
  \centering
  \vspace{-1ex}
  \caption{Comparison of VQ-VAEs trained on ImageNet dataset following~\cite{Oord2017NeuralDR}.}
  \label{tab:vqvae imagenet}
  \vspace{-1ex}
  \small
  \resizebox{0.9\textwidth}{!}{%
  \begin{tabular}{@{}lccccccc@{}}
    \toprule
    \textbf{Methods} & \textbf{Tokens} & \textbf{Codebook Size} & $\mathcal{U}$ ({\color{green!60!black}\bfseries$\pmb{\uparrow}$}) & $\mathcal{C}$ ({\color{green!60!black}\bfseries$\pmb{\uparrow}$}) & \textbf{PSNR} ({\color{green!60!black}\bfseries$\pmb{\uparrow}$}) & \textbf{SSIM} ({\color{green!60!black}\bfseries$\pmb{\uparrow}$}) & \textbf{Rec. Loss} ({\color{green!60!black}\bfseries$\pmb{\downarrow}$}) \\
    \midrule
    Vanilla VQ & 256 & $16{,}384$ & 2.5\% & 360.7 & 24.44 & 57.5 & 0.0294 \\
    STE++ & 256 & $16{,}384$ & 6.5\% & 889.9 & 24.88 & 58.9 & 0.0270 \\
    EMA VQ & 256 & $16{,}384$ & 14.5\% & 1861.5 & 24.98 & 59.2 & 0.0267 \\
    Online VQ & 256 & $16{,}384$ & 22.2\% & 1465.6 & 24.88 & 58.6 & 0.0273 \\
    \textbf{Wasserstein VQ} & 256 & $16{,}384$ & \textbf{100.0\%} & \textbf{15539.1} & \textbf{25.47} & \textbf{61.2} & \textbf{0.0242} \\
    \midrule
    Vanilla VQ  & 256 & $50{,}000$ & 0.9\% & 378.7 & 24.40 & 57.7 & 0.0295 \\
    STE++ & 256 & $50{,}000$ & 2.0\% & 851.7 & 24.89 & 59.0 & 0.0270 \\
    EMA VQ  & 256 & $50{,}000$ & 16.8\% & 6139.3 & 25.37 & 60.9 & 0.0246 \\
    Online VQ & 256 & $50{,}000$ & 9.9\% & 2241.7 & 25.09 & 59.7 & 0.0260 \\
    \textbf{Wasserstein VQ} & 256 & $50{,}000$ & \textbf{100.0\%} & \textbf{46133.2} & \textbf{25.72} & \textbf{62.3} & \textbf{0.0230} \\
    \midrule
    Vanilla VQ & 256 & $100{,}000$ & 0.4\% & 337.0 & 24.43 & 57.4 & 0.0295 \\
    STE++ & 256 & $100{,}000$ & 0.9\% & 730.6 & 24.86 & 59.1 & 0.0269 \\
    EMA VQ  & 256 & $100{,}000$ & 3.0\% & 2170.0 & 25.13 & 60.1 & 0.0257 \\
    Online VQ  & 256 & $100{,}000$ & 4.1\% & 1709.9 & 24.95 & 59.1 & 0.0267 \\
    \textbf{Wasserstein VQ} & 256 & $100{,}000$ & \textbf{100.0\%} & \textbf{93264.7} & \textbf{25.88} & \textbf{63.0} & \textbf{0.0223} \\
    \bottomrule
  \end{tabular}
}
\vspace{-2ex}
\end{table*}

\paragraph{Main Results} 
As shown in Tables~\ref{tab:vqvae ffhq},~\ref{tab:vqvae imagenet} in the main text, and Tables~\ref{tab:vqvae cifar10},~\ref{tab:vqvae svhn} in Appendix~\ref{appendix:vqvae cifar-10 and svhn}, our proposed \emph{Wasserstein VQ} consistently outperforms all baselines across datasets, achieving superior performance on nearly all metrics under a wide range of settings. 
The improvements are particularly pronounced in codebook utilization ($\mathcal{U}$) and codebook perplexity ($\mathcal{C}$), where \emph{Wasserstein VQ} attains near-complete codebook utilization and substantially higher codebook perplexity. This indicates that the learned codebooks more faithfully reflect the structure and density of the feature space.
As vector quantization fundamentally acts as a compression mechanism from continuous latent representations to a discrete codebook, improved alignment between feature and code distributions directly translates into reduced information loss. Consistent with this interpretation, \emph{Wasserstein VQ} achieves the lowest reconstruction loss and improved PSNR/SSIM across datasets.

\begin{wrapfigure}[20]{r}{0.50\textwidth}
    \vspace{-6mm}
    \centering
    \includegraphics[width=0.98\linewidth]{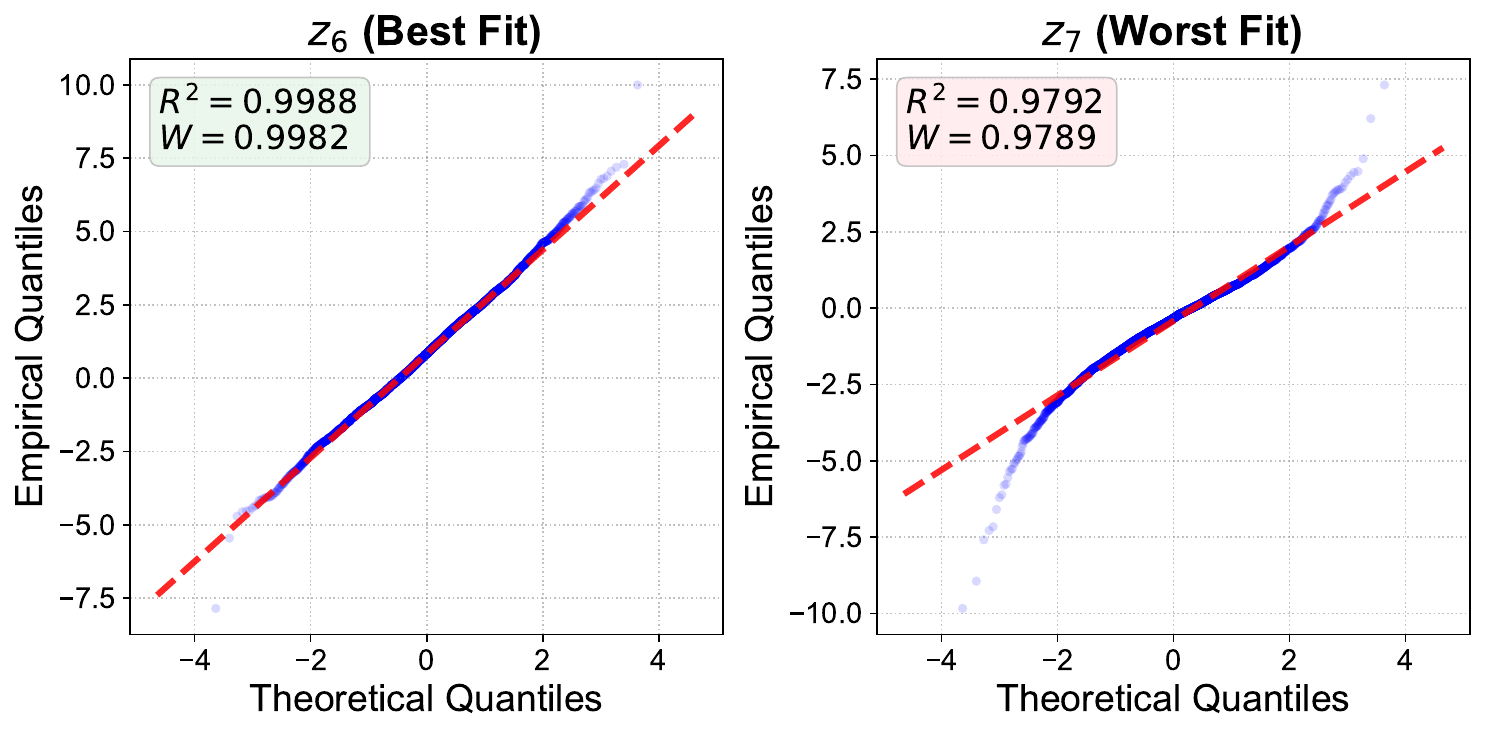} 
    \vspace{-2mm}
    \caption{\textbf{Gaussianity Validation via Q-Q Plots.} We visualize the feature dimensions with the highest $z_6$ (Left) and lowest ($z_7$, Right) conformance to the Gaussian hypothesis. \textbf{Note:} For visual clarity, plots display a random subset of 5,000 points, while the reported statistics ($R^2, W$) are computed on the \textbf{full validation set ($N=409,600$)}. The strong alignment with the red identity line ($y=x$) confirms that the learned features are predominantly Gaussian, with deviations confined to sparse heavy tails.}
    \label{fig:qq_main}
    \vspace{-6mm}
\end{wrapfigure} 

\paragraph{Gaussian Approximation Validation}
To empirically assess the Gaussian approximation underlying our Wasserstein objective, we analyze the distributional properties of the learned latent features.
Using the FFHQ dataset, we extract latent vectors from a subset of 1,600 images ($d=8$), resulting in a total of $N=409{,}600$ samples.
Figure~\ref{fig:qq_main} presents Quantile--Quantile (Q--Q) plots for the feature dimensions exhibiting the highest ($z_6$) and lowest ($z_7$) agreement with a Gaussian reference.
Under an ideal Gaussian model, the samples (blue points) align with the theoretical reference line ($y=x$).
As shown, the best-fitting dimension ($z_6$) demonstrates near-perfect alignment, with a coefficient of determination of $R^2 = 0.9988$.
Importantly, even for the least-fitting dimension ($z_7$), the Q--Q plot evaluated over the distribution achieves a high goodness-of-fit score ($R^2 > 0.979$).
While mild deviations appear in the extreme tails, the central mass of the distribution, which dominates the probability mass and governs the Wasserstein objective, is well approximated by a Gaussian.
These results support the use of a Gaussian-based Wasserstein formulation as an effective approximation for distributional matching.
A more comprehensive univariate and multivariate analysis, including Mahalanobis distance based normality assessment across all dimensions, is provided in Appendix~\ref{sec:normality_tests}.

\begin{table*}[!t]
  \centering
  \caption{Analysis of generalization ability from ImageNet to FFHQ dataset.}
  \label{tab:generalization ability}
  \vspace{-1ex}
  \resizebox{\textwidth}{!}{%
  \begin{tabular}{@{}lccccccc@{}} \hline
    Approaches  & Tokens & Codebook Dim & $\mathcal{U}$ ({\color{green!60!black}\bfseries$\pmb{\uparrow}$}) & $\mathcal{C}$ ({\color{green!60!black}\bfseries$\pmb{\uparrow}$}) & PSNR({\color{green!60!black}\bfseries$\pmb{\uparrow}$}) & SSIM({\color{green!60!black}\bfseries$\pmb{\uparrow}$}) & Rec. Loss ({\color{green!60!black}\bfseries$\pmb{\downarrow}$})\\\hline
    Vanilla VQ & 256 & 16384 & 1.7\% & 332.9 & 27.18 & 71.6 & 0.0138\\ 							
    EMA VQ & 256 & 16384 & 12.5\% & 1292.1 & 27.69 & 72.3 & 0.0126 \\ 					
    Online VQ & 256 & 16384 & 33.2\% & 2794.9 & 27.91 & 72.7 & 0.0120 \\ 				
    \textbf{Wasserstein VQ} & 256 & 16384 & 99.2\% & 13975.4 & 28.62 & 75.0 &  0.0103 \\ 	
    \hline
    \vspace{-3ex}
  \end{tabular}
}
\end{table*}

\begin{table*}[!t]
  \centering
  \vspace{-1ex}
  \caption{The computational overhead of various VQ approaches.}
  \label{tab:computational overhead}
  \vspace{-1ex}
  \resizebox{\textwidth}{!}{%
  \begin{tabular}{@{}lcccccc@{}} \hline
    Approaches  & Codebook Size & Times (second)  & Codebook Size & Times (second)  & Codebook Size & Times (second) \\\hline
    Vanilla VQ & 16384 & \textbf{0.184} & 50000 & \textbf{0.408}  & 100000 & \textbf{0.655} \\ 
    EMA VQ & 16384 & 0.265 & 50000 & 0.807 & 100000 & 1.502\\  			
    Online VQ & 16384 & 1.810 & 50000 & 5.438 & 100000 & 10.64\\ 	
    \textbf{Wasserstein VQ} & 16384 & \underline{0.294}  & 50000 & \underline{0.525} & 100000& \underline{0.774} \\ 	
    \hline
    \vspace{-6ex}
  \end{tabular}
}
\end{table*}

\paragraph{Representation Visualization} 
To examine the learned distribution of feature and code vectors across different VQ methods trained on the FFHQ dataset (with a fixed codebook size of $K=8192$), we randomly sample 3,000 feature vectors and 1,000 code vectors and visualize them via scatter plots. As shown in Figures~\ref{fig:codebook_feature 1} and~\ref{fig:codebook_feature 2}, in both Vanilla VQ and EMA VQ, most code vectors are concentrated near the origin, rendering them largely inactive. While Online VQ mitigates this central clustering, its code vectors are pushed toward the extremes of the feature space, as illustrated in Figure~\ref{fig:codebook_feature 3}. This distributional mismatch leads to increased information loss and reduced codebook utilization. By contrast, \emph{Wasserstein VQ} achieves substantially better alignment between feature and code vector distributions, significantly reducing information loss and improving codebook utilization.

\begin{figure*}[!t]
    \vspace{-2mm}
	\centering
	\subfloat[Vanilla VQ]{ 
		\label{fig:codebook_feature 1}  
		\includegraphics[width=0.225\textwidth]{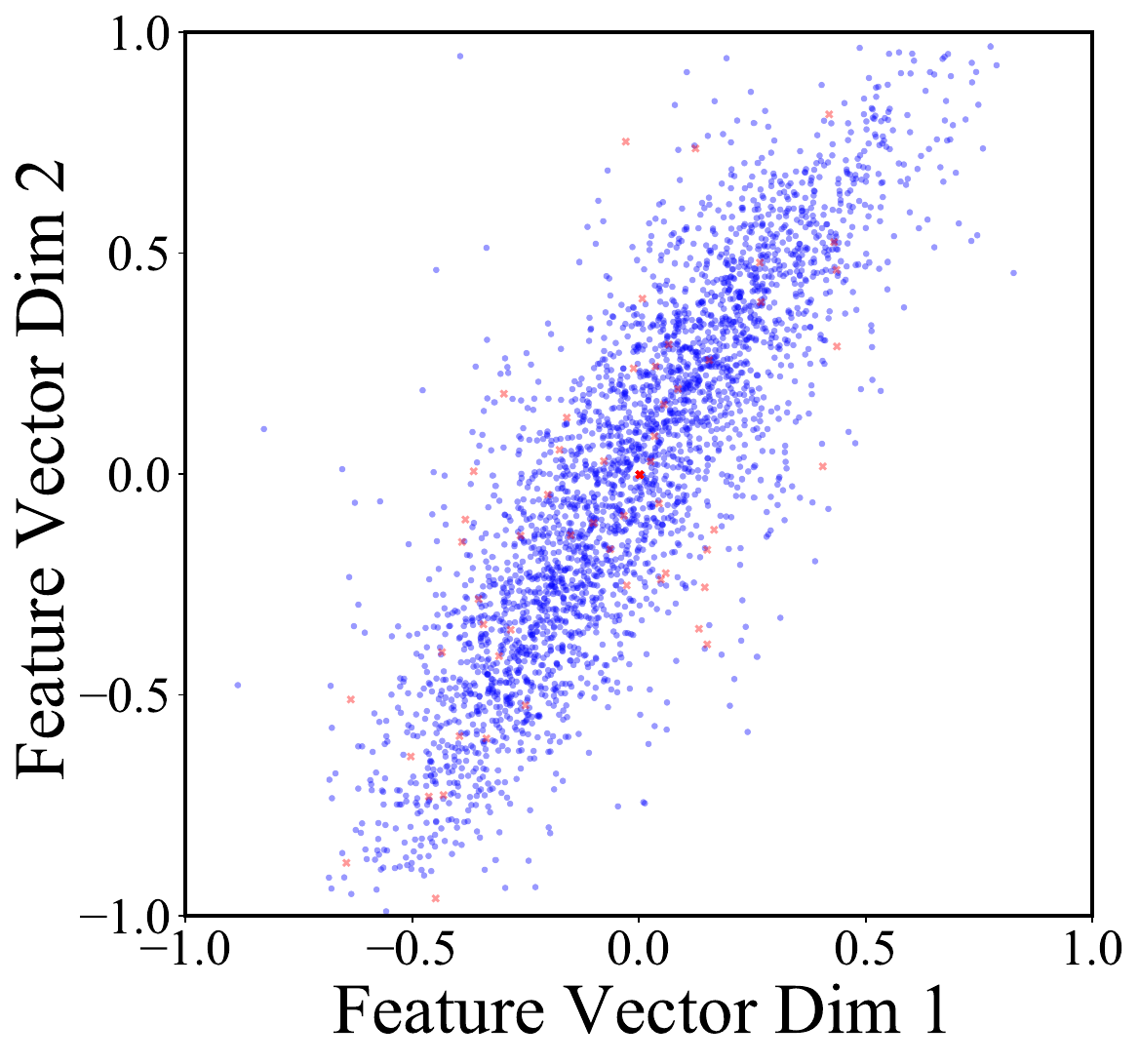}
        }
	\hspace{0.03cm}
	\subfloat[EMA VQ]{
		\label{fig:codebook_feature 2}
		\includegraphics[width=0.225\textwidth]{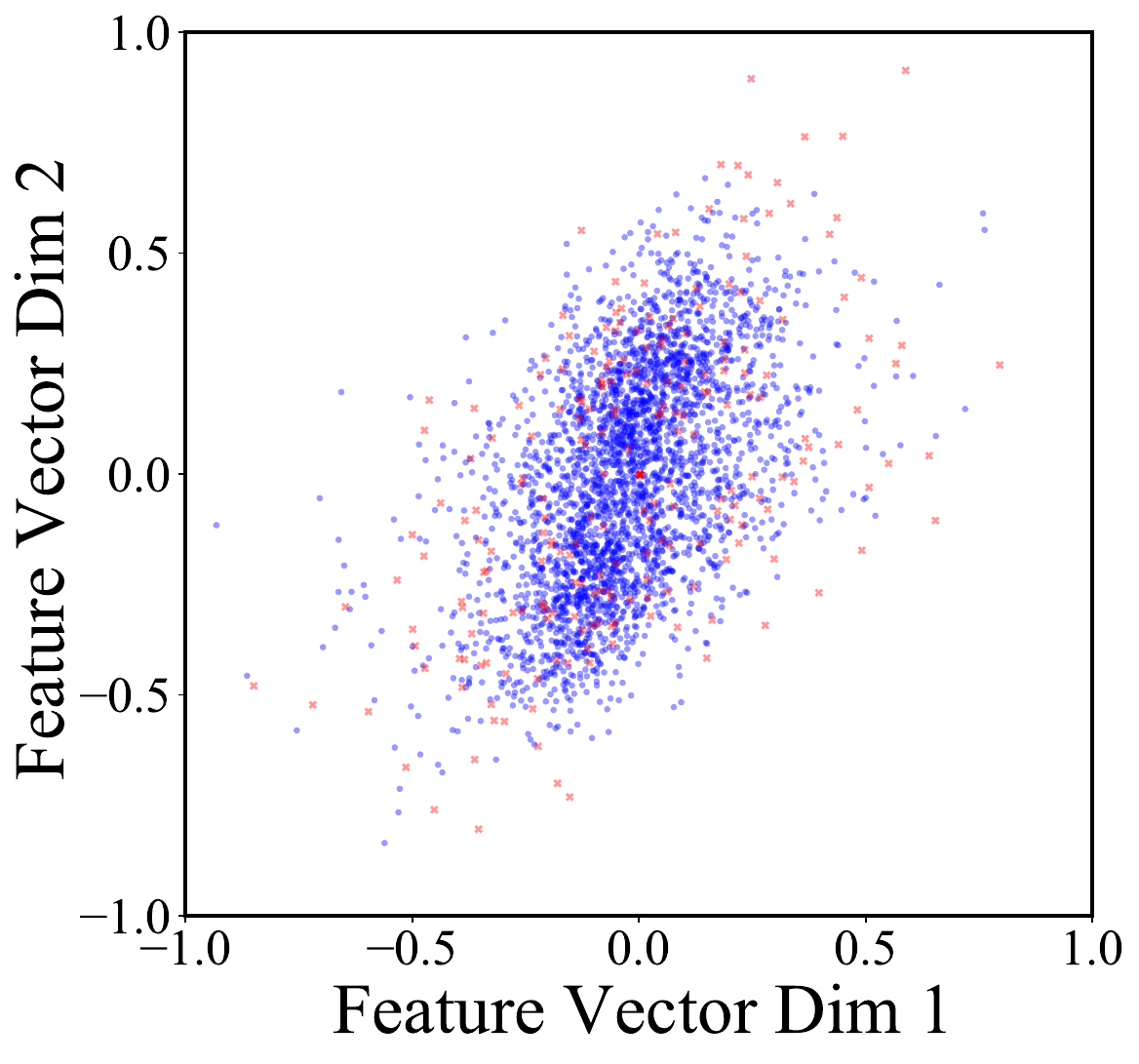}
	}
        \hspace{0.03cm}
	\subfloat[Online VQ]{ 
		\label{fig:codebook_feature 3}  
		\includegraphics[width=0.225\textwidth]{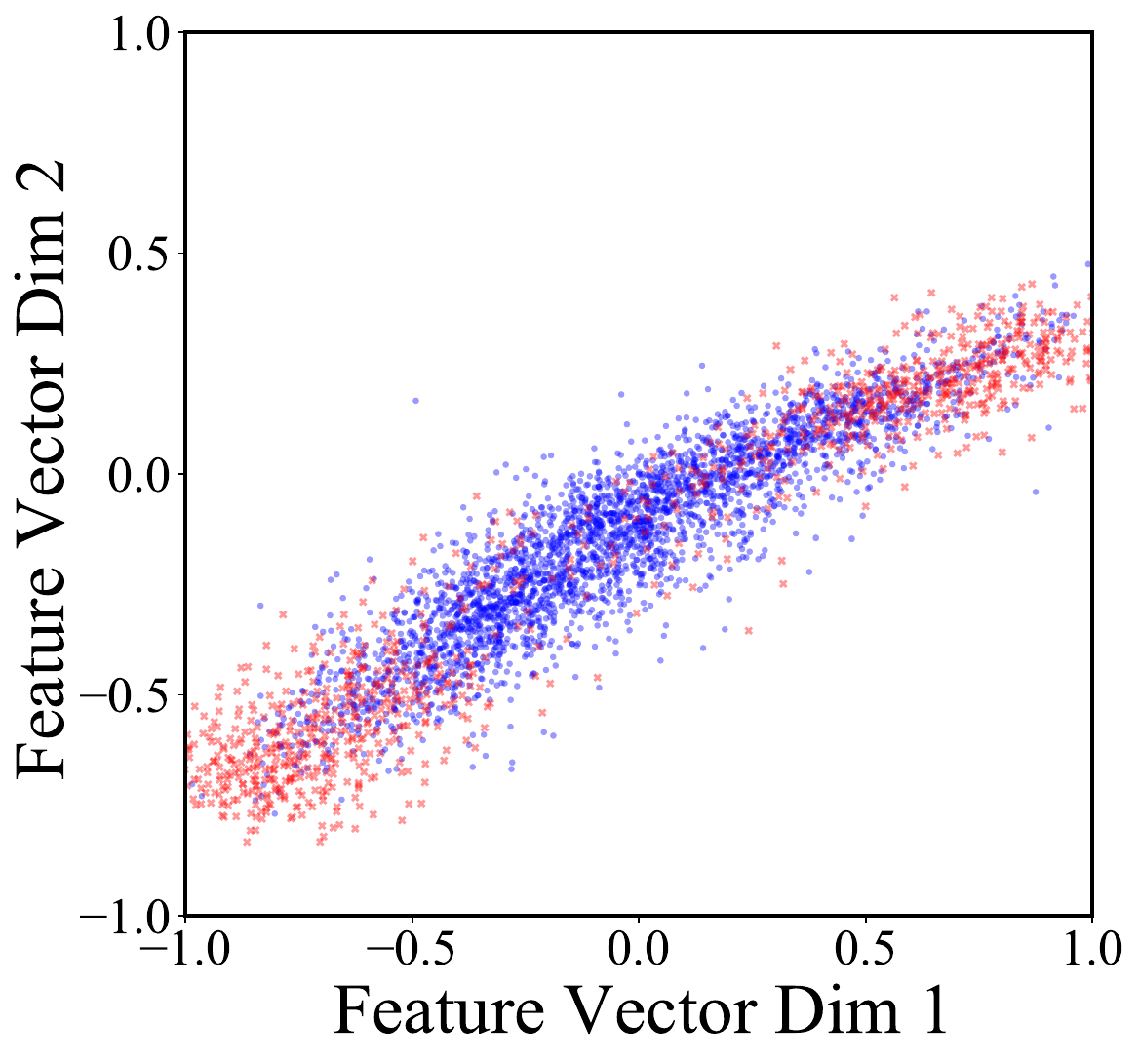}
	}
        \hspace{0.03cm}
	\subfloat[Wasserstein VQ]{
		\label{fig:codebook_feature 4}
		\includegraphics[width=0.225\textwidth]{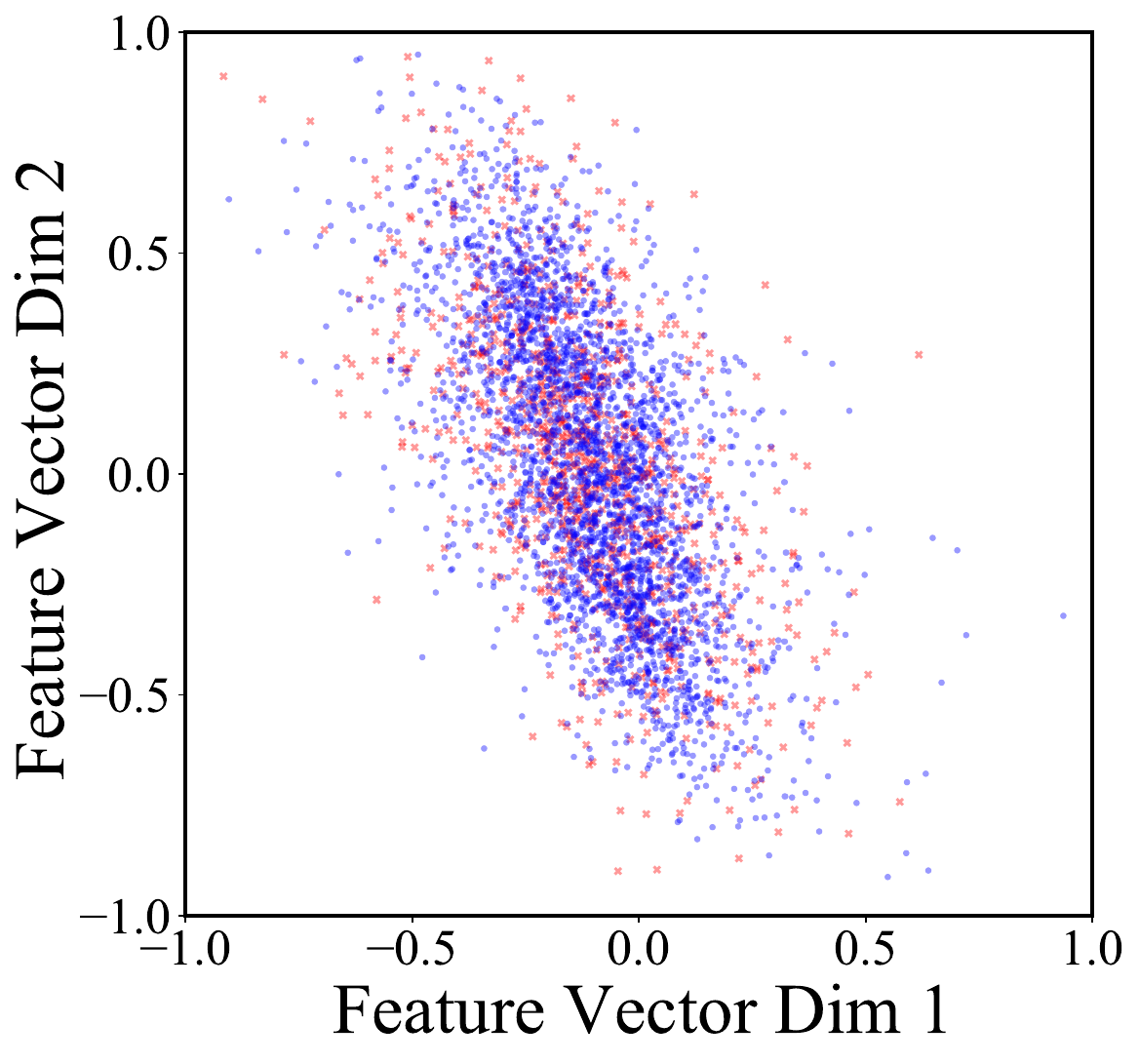}}
    \vspace{-2mm}
	\caption{Visualization of feature and codebook distributions. Blue $\cdot$ and red $\times$ represent the feature and code vectors, respectively.}
	\label{fig:feature codebook visualization}
    \vspace{-3mm}
\end{figure*}

\paragraph{Cross-Dataset Generalization}
To evaluate the model’s generalization capability, we conducted a cross-dataset transfer experiment. Specifically, we used the pre-trained checkpoint from Table~\ref{tab:vqvae imagenet} (trained on ImageNet) and evaluated its performance on the FFHQ dataset, which is out-of-distribution relative to ImageNet. The results are summarized in Table~\ref{tab:generalization ability}.

Among all methods, Wasserstein VQ achieves the highest performance in this transfer setting, indicating that the proposed distribution alignment objective does not impair generalization. Moreover, Wasserstein VQ’s strong performance on the out-of-distribution FFHQ dataset suggests that enforcing alignment may contribute to improved robustness under distribution shifts. Overall, these findings provide evidence that models trained with the proposed regularization generalize effectively to out-of-distribution data.

\paragraph{Computational Overhead Analyses} We evaluate the runtime efficiency by measuring the forward and backward pass times of the VQ module. As shown in Table~\ref{tab:computational overhead} in Appendix~\ref{appendix:computational overhead}, the runtime of Wasserstein VQ is comparable to Vanilla VQ, with only negligible overhead arising from the covariance computation.  This confirms that the closed-form Wasserstein objective offers a highly efficient pathway to optimal codebook utilization without incuring a significant time overhead. More implementation details see Appendix~\ref{appendix:computational overhead}.

\paragraph{Sensitivity Analysis on $\gamma$} 
We conduct a sensitivity analysis with respect to $\gamma$ on the FFHQ dataset by varying $\gamma \in \{0, 10^{-5}, 10^{-4}, 10^{-3}, 10^{-2}, 10^{-1}, 1\}$. As reported in Table~\ref{tab:sensitivity analysis} in Appendix~\ref{appendix:sensitivity analysis}, when $\gamma = 0$, Wasserstein VQ yields the worst performance, as it degenerates into the vanilla VQ formulation without distributional matching. As $\gamma$ increases from $10^{-5}$ to $10^{-3}$, the performance of Wasserstein VQ consistently improves. When $\gamma$ reaches $10^{-2}$, Wasserstein VQ achieves full ($100\%$) codebook utilization along with competitive quantitative results. Moreover, within the range $\gamma \in [10^{-2}, 1]$, the performance of Wasserstein VQ remains stable, indicating that the method is not sensitive to the precise choice of $\gamma$ once it exceeds a moderate threshold.

\paragraph{Effect of Codebook Size on FFHQ} To investigate the impact of codebook size on VQ performance, we vary the codebook size $K$  across a wide range: $K \in \{1024, 2048, 4096, 8192, 16384, 50000, 100000\}$. As shown in Table~\ref{tab:vqvae ffhq} in the main text and Table~\ref{tab:vqvae ffhq codebook size} in Appendix~\ref{appendix:ablation study}, the vanilla VQ model suffers from severe codebook collapse even at relatively small sizes, such as $K=1024$. In contrast, improved variants like EMA VQ and Online VQ handle smaller codebooks effectively, but still exhibit codebook collapse when $K$ becomes very large (e.g., $K \geq 50000$).  In contrast, \emph{Wasserstein VQ} consistently maintains 100\% codebook utilization across all codebook sizes, highlighting the effectiveness of enforcing distributional alignment via the quadratic Wasserstein distance in mitigating codebook collapse.

\paragraph{Effect of Codebook Dimensionality} 
We investigate the impact of codebook dimensionality on VQ performance by conducting experiments on CIFAR-10 with $d$ ranging from 2 to 32. As reported in Table~\ref{tab:vqvae cifar-10 codebook dimension} (Appendix~\ref{appendix:ablation study}), our proposed Wasserstein VQ consistently outperforms all baselines across all dimensionalities. We also observe a manifestation of the curse of dimensionality: performance generally degrades as $d$ increases. Vanilla VQ experiences the most severe decline, followed by EMA VQ and Online VQ, whereas Wasserstein VQ exhibits only minimal reduction in codebook utilization.


\begin{table*}[!t]
\centering
\vspace{-1ex}
\caption{
    {Reconstruction performance on the ImageNet-1K dataset. The suffixes "–a," "–b," and "–c" correspond to decoder adaptation for 5, 10, and 15 epochs, respectively. The best-performing result is highlighted in bold. Results marked with $^{\dagger}$ are cited from VQGAN-LC \cite{Zhu2024ScalingTC}, those with $^{\star}$ from Llama GEN \cite{Sun2024AutoregressiveMB}, and those with $^{\ddagger}$ from VQ-Transplant \cite{Fang2025VQTransplant}. }
    }
\vspace{-1ex}
\small
\resizebox{0.97\textwidth}{!}{
\begin{tabular}{@{}lcccccccc@{}}
\toprule
\textbf{Methods} & \textbf{Tokens} & \textbf{Codebook Size} & $\mathcal{U}$ ({\color{green!60!black}\bfseries$\pmb{\uparrow}$}) & \textbf{r-FID} ({\color{green!60!black}\bfseries$\pmb{\downarrow}$}) & \textbf{r-IS} ({\color{green!60!black}\bfseries$\pmb{\uparrow}$}) & \textbf{LPIPS} ({\color{green!60!black}\bfseries$\pmb{\downarrow}$}) & \textbf{PSNR} ({\color{green!60!black}\bfseries$\pmb{\uparrow}$}) & \textbf{SSIM} ({\color{green!60!black}\bfseries$\pmb{\uparrow}$})\\
\midrule
DQVAE$^{\dagger}$ & 256  & $1{,}024$& - & 4.08  & - & - & - & - \\
DiVAE$^{\dagger}$ & 256 & $16{,}384$ & - & 4.07 & - & - & - & - \\
RQVAE$^{\dagger}$ & 256 & $16{,}384$ & - &  3.20 & - & - & - & -\\
RQVAE$^{\dagger}$ & 512 & $16{,}384$ & - & 2.69 & - & - & - & -\\
RQVAE$^{\dagger}$ & $1{,}024$& $16{,}384$ & - & 1.83 & - & - & -  & -\\
DF-VQGAN$^{\dagger}$ & 256 & $12{,}288$ & - & 5.16 & - & - & - & - \\
DF-VQGAN$^{\dagger}$ & $1{,}024$& $8{,}192$& - & 1.38 & - & - & - & - \\
Llama GEN$^{\star}$ & 256 & $16{,}384$ & 97.0\% & 2.19 & -  & - &  20.79 & 67.5 \\
\midrule
\multirow{3}{*}{VQGAN$^{\dagger}$}  & 256 & $16{,}384$ & 3.4\% & 5.96 & - & 0.17 & 23.3& 52.4 \\
& 256 &  $50{,}000$ & 1.1\% & 5.44 & - & 0.17 & 22.5 & 52.5 \\
& 256 &  $100{,}000$ & 0.5\% &  5.44 & - & 0.17 & 22.3 & 52.5 \\
\midrule
\multirow{3}{*}{VQGAN-FC$^{\dagger}$} & 256 & $16{,}384$ & 11.2\%  & 4.29 & - & 0.17 & 22.8 & 54.5 \\ 
& 256 & $50{,}000$ & 3.6\%  & 4.96 & - & 0.15 & 23.1 & 54.7 \\
& 256 & $100{,}000$ & 1.9\%  & 4.65 & - & 0.15 & 22.9 & 55.1 \\
\midrule
\multirow{3}{*}{VQGAN-EMA$^{\dagger}$} & 256 & $16{,}384$ & 83.2\% & 3.41 & -  & 0.14 & 23.5 & 56.6 \\
& 256 & $50{,}000$ &  40.2\% & 3.88 & -  & 0.14 & 23.2 & 55.9 \\
& 256 & $100{,}000$ & 24.2\% & 3.46 & -  &  0.13 & 23.4 & 56.2 \\
\midrule
\multirow{4}{*}{VQGAN-LC$^{\dagger}$} & 256 & $16{,}384$ & \textbf{99.9\%}  & 3.01 & - & 0.13 & 23.2 & 56.4 \\
 & 256 & $50{,}000$ & \textbf{99.9\%}  & 2.75 & - & 0.13 &  23.8 &  58.4 \\
 & 256 & $100{,}000$ & \textbf{99.9\%}  & 2.62 & - &  0.12 &  23.8 & 58.9 \\
 & $1{,}024$& $100{,}000$ & 99.5\%  & 1.29 & - &  \textbf{0.07} &  \textbf{27.0} &  \textbf{71.6} \\\midrule
\midrule
\multirow{3}{*}{\textbf{Wasserstein VQ-a}$^{\ddagger}$} & 512 & $16{,}384$ & 99.8\% & 1.04 & 191.3 & 0.114 & 24.36 & 64.0 \\
& 512 & $32{,}768$&  99.7\% & 0.98 &  193.9 & 0.111 & 24.37 & 64.3 \\
& 512 & $65{,}536$& 99.6\% & 0.92 & 195.5 & 0.106 & 24.68 & 65.4 \\\midrule
\multirow{3}{*}{\textbf{Wasserstein VQ-b}} & 512 & $16{,}384$ & 99.8\% & 0.98 & 192.9 & 0.114 & 24.34 & 63.7\\
& 512 & $32{,}768$& 99.8\% & 0.88 & 196.2 & 0.109 & 24.60 & 64.7 \\
& 512 & $65{,}536$& 99.6\% & 0.81 & \textbf{198.7} & 0.105 & \textbf{24.77} & \textbf{65.5}\\\midrule
\multirow{3}{*}{\textbf{Wasserstein VQ-c}} & 512 & $16{,}384$ & 99.8\% & 0.90 & 194.1 & 0.114 & 24.27 & 63.4\\
& 512 & $32{,}768$& 99.8\% & 0.85 & 196.4 & 0.109 & 24.43 & 64.1 \\
& 512 & $65{,}536$& 99.6\% & \textbf{0.79} & 198.5 & \textbf{0.104} & 24.73 & 65.2 \\
\bottomrule
\end{tabular}}
\vspace{-3ex}
\label{tab:reconstruction_IN-main table}
\end{table*}

\begin{table}[!t]
\centering
\caption{
    {Reconstruction performance on the ImageNet-1K dataset for multi-scale quantization algorithms. The suffixes "–a," "–b," and "–c" correspond to decoder adaptation for 5, 10, and 15 epochs, respectively. For each codebook size, the best-performing result is highlighted in bold. $^{\ddagger}$: results are cited from VQ-Transplant \cite{Fang2025VQTransplant}. }
    }
\small
\resizebox{0.97\textwidth}{!}{
\begin{tabular}{@{}lcccccc@{}}
\toprule
\textbf{Methods} & \textbf{Tokens} & \textbf{Codebook Size} & $\mathcal{E}$ ({\color{green!60!black}\bfseries$\pmb{\downarrow}$}) & $\mathcal{U}$ ({\color{green!60!black}\bfseries$\pmb{\uparrow}$}) & $\mathcal{C}$ ({\color{green!60!black}\bfseries$\pmb{\uparrow}$}) & \textbf{r-FID} ({\color{green!60!black}\bfseries$\pmb{\downarrow}$}) \\\midrule
VAR$^{\ddagger}$ & 680 & $4{,}096$& 0.283 & \textbf{100\%} & $2{,}981.4$ & 0.92\\\midrule
\multirow{2}{*}{Vanilla VAR-a$^{\ddagger}$} & 680 & $4{,}096$& 0.305 & 38.2\% & $1{,}300$.4 & 1.25 \\
& 680 & $8{,}192$& 0.309 & 22.9\% & $1{,}422$.6 & 1.30\\\midrule
\multirow{2}{*}{EMA VAR-a$^{\ddagger}$} & 680 & $4{,}096$& 0.321 & 99.9\% & $1{,}806$.6 &  1.69 \\
& 680 & $8{,}192$&  0.312 & 99.8\% & $2{,}498$.8 &  1.15\\\midrule
\multirow{2}{*}{Online VAR-a$^{\ddagger}$} & 680 & $4{,}096$&  0.276 &  99.0\% & $1{,}950$.9 & 1.05\\
& 680 & $8{,}192$& 0.269 & 73.9\% & $3{,}588$.6 & 1.00\\\midrule
\multirow{2}{*}{\textbf{Wasserstein VAR-a}$^{\ddagger}$} & 680 & $4{,}096$& \textbf{0.255} & \textbf{100\%} &  \textbf{$3{,}286.2$} & 0.93 \\
& 680 & $8{,}192$&  \textbf{0.240}  & \textbf{100\%} & \textbf{$6{,}518.2$}
 & 0.83\\\midrule
\multirow{2}{*}{\textbf{Wasserstein VAR-b}} & 680 & $4{,}096$& \textbf{0.255} & \textbf{100\%} & \textbf{$3{,}286.2$} & 0.88 \\
& 680 & $8{,}192$&  \textbf{0.240}  & \textbf{100\%} & \textbf{$6{,}518.2$}
 & 0.79 \\\midrule
\multirow{2}{*}{\textbf{Wasserstein VAR-c}} & 680 & $4{,}096$& \textbf{0.255} & \textbf{100\%} & \textbf{$3{,}286.2$} & \textbf{0.81} \\
& 680 & $8{,}192$&  \textbf{0.240}  & \textbf{100\%} & \textbf{$6{,}518.2$}
 & \textbf{0.73}\\
\bottomrule
\end{tabular}}
\vspace{-4ex}
\label{tab:reconstruction_IN-multi-scale}
\end{table}

\subsection{Evaluation on the VQGAN Framework}

\paragraph{Datasets and Baselines} We evaluate our proposed method against a comprehensive set of baselines on the ImageNet and FFHQ datasets. For ImageNet, the compared methods include DQVAE~\cite{Huang2023TowardsAI}, DF-VQGAN~\cite{Ni2022NWALIPLI}, DiVAE~\cite{Shi2022DiVAEPI}, RQVAE~\cite{Lee2022AutoregressiveIG}, VQGAN~\cite{Esser2020TamingTF}, VQGAN-FC~\cite{Yu2021VectorquantizedIM}, VQGAN-EMA~\cite{Razavi2019GeneratingDH}, VQGAN-LC~\cite{Zhu2024ScalingTC}, Llama GEN~\cite{Sun2024AutoregressiveMB}, and VAR~\cite{Tian2024VisualAM}.
For FFHQ, we compare against RQVAE~\cite{Lee2022AutoregressiveIG}, VQGAN~\cite{Esser2020TamingTF}, VQGAN-FC~\cite{Yu2021VectorquantizedIM}, VQGAN-EMA~\cite{Razavi2019GeneratingDH}, VQ-WAE~\cite{Vuong2023VectorQW}, MQVAE~\cite{Huang2023NotAI}, and VQGAN-LC~\cite{Zhu2024ScalingTC}. Implementation details see Appendix~\ref{appendix:experimental details}.

\paragraph{Metrics} We report standard reconstruction and perceptual quality metrics, including Fréchet Inception Distance (r-FID)~\cite{Heusel2017GANsTB},
Reconstruction Inception Score (r-IS)~\cite{Salimans2016ImprovedTF},
Learned Perceptual Image Patch Similarity (LPIPS)~\cite{Zhang2018TheUE},
Peak Signal-to-Noise Ratio (PSNR), and Structural Similarity Index Measure (SSIM).

\begin{table*}[!t]
\centering
\caption{{Reconstruction performance on the FFHQ dataset. Results marked with $^{\dagger}$ are cited from VQGAN-LC \cite{Zhu2024ScalingTC}, and those with $^{\ddagger}$ from VQ-Transplant~\cite{Fang2025VQTransplant}.}}
\vspace{-1ex}
\label{tab:vqgan tranplant ffhq}
\resizebox{0.96\textwidth}{!}{%
\begin{tabular}{@{}lcccccccc@{}}
\toprule
VQs & Tokens & Codebook Size $K$ & $\mathcal{E}$({\color{green!60!black}\bfseries$\pmb{\downarrow}$}) & $\mathcal{U}$ ({\color{green!60!black}\bfseries$\pmb{\uparrow}$}) & PSNR({\color{green!60!black}\bfseries$\pmb{\uparrow}$}) & SSIM({\color{green!60!black}\bfseries$\pmb{\uparrow}$}) & LPIPS ({\color{green!60!black}\bfseries$\pmb{\downarrow}$}) & r-FID({\color{green!60!black}\bfseries$\pmb{\downarrow}$}) \\
\midrule
RQVAE$^{\dagger}$ & 256 & 2048 & - & - & 22.9 & 67.0 & 0.13 & 7.04  \\
VQ-WAE$^{\dagger}$ & 256 & 1024 & - & -  & 22.5 & 66.5 & 0.12 & 4.20 \\
MQVAE$^{\dagger}$ & 256 & 1024 & - & 78.2\% & - & - & -  & 4.55 \\
VQGAN$^{\dagger}$ & 256 & 16384 & - & 2.3\% & 24.4 & 63.3 & 0.12 & 5.25  \\
VQGAN-FC$^{\dagger}$ & 256 & 16384  & - & 10.9\% & 24.8 & 64.6 & 0.11 & 4.86  \\
VQGAN-EMA$^{\dagger}$  & 256 & 16384  & - & 68.2\% & 25.4 & 66.1 & 0.10 & 4.79   \\
VQGAN-LC$^{\dagger}$ & 256 & 100000  & - & 99.5\% & 26.1 & 69.4 & 0.08 & 3.81   \\\midrule
\textbf{Wasserstein VQ$^{\ddagger}$} & 512 & 16384 & \textbf{0.153} & \textbf{99.7\%}   &  \textbf{27.25} & \textbf{75.4} & \textbf{0.075} & \textbf{1.81} \\
\textbf{Wasserstein VQ$^{\ddagger}$} & 512 & 32768  & \textbf{0.142} & \textbf{99.7\%}  & \textbf{27.33} & \textbf{75.7} & \textbf{0.072} & \textbf{1.21} \\
\bottomrule
\vspace{-4ex}
\end{tabular}
}
\end{table*}

\begin{figure*}[!t]
    \centering
    \includegraphics[width=0.96\linewidth]{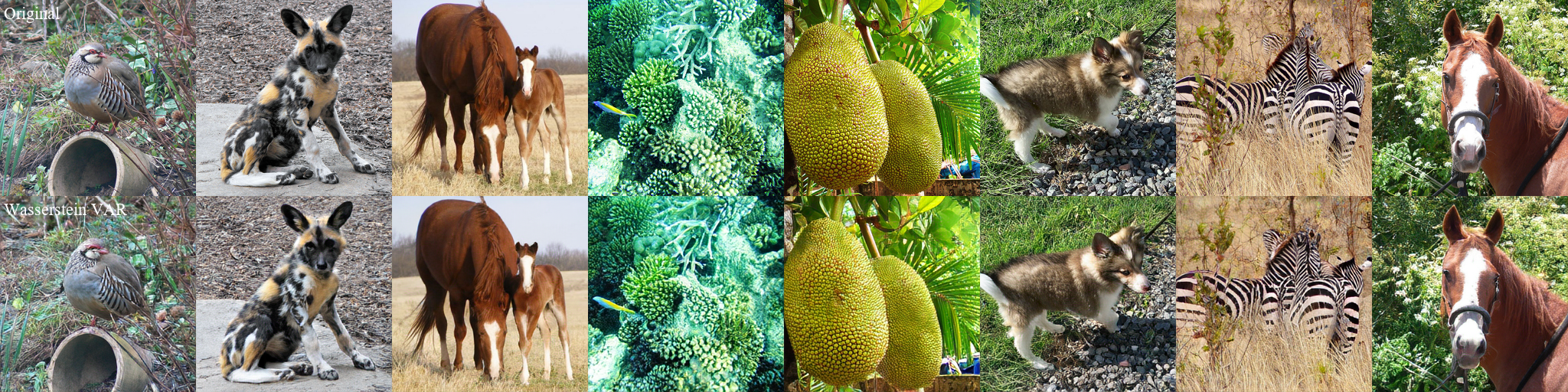}
    \vspace{-1ex}
    \caption{Visualization of reconstructed ImageNet Images. The top row displays the original input images with a resolution of $256 \times 256$ pixels, while the bottom row shows the reconstructed images from the \emph{Wasserstein VAR}.}
    \label{fig:visual reconstruction imagenet}
    \vspace{-2ex}
\end{figure*}

\begin{figure*}[!t]
    \centering    \includegraphics[width=0.96\linewidth]{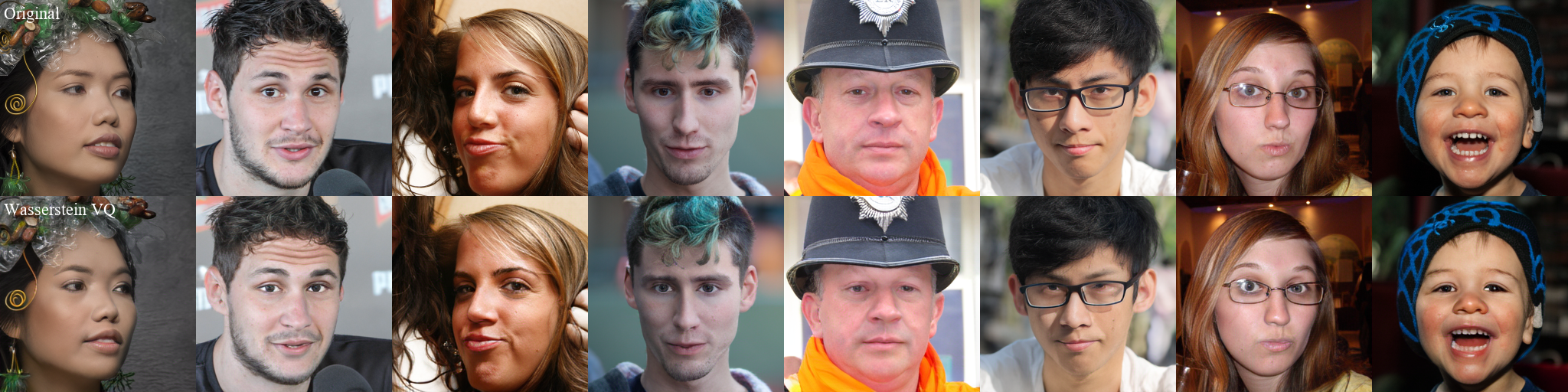}
    \vspace{-1ex}
    \caption{Visualization of reconstructed FFHQ Images. The top row displays the original input images with a resolution of $256 \times 256$ pixels, while the bottom row shows the reconstructed images from the \emph{Wasserstein VQ}.}
    \label{fig:visual reconstruction ffhq}
    \vspace{-3ex}
\end{figure*}

\paragraph{Main Results} As shown in Tables~\ref{tab:reconstruction_IN-main table} and~\ref{tab:vqgan tranplant ffhq}, our proposed \emph{Wasserstein VQ} achieves the lowest r-FID within the VQGAN framework. 
We further evaluate multi-scale quantization algorithms in Table~\ref{tab:reconstruction_IN-multi-scale}. All experiments in this table are conducted under fully controlled conditions: the encoder–decoder architecture is fixed to VAR~\cite{Tian2024VisualAM}, and, importantly, the latent features are kept identical across all methods. This ensures that measurements of codebook utilization, perplexity, and quantization error are not influenced by variations in the latent features, as discussed in Section~\ref{sec:distributional formulation}.

Under these controlled conditions, \emph{Wasserstein VAR} consistently demonstrates the strongest overall performance, with particularly notable improvements in quantization error. Reconstruction quality further improves as the number of decoder adaptation epochs increases, reflecting stronger adversarial refinement. For instance, when $K=8192$, the r-FID decreases from \textbf{0.83} (Wasserstein VAR-a) to \textbf{0.73} (Wasserstein VAR-c).



\paragraph{Reconstruction Visualization}
To further demonstrate the reconstruction performance, we visualize the reconstructed FFHQ and ImageNet images in Figure~\ref{fig:visual reconstruction ffhq} and~\ref{fig:visual reconstruction imagenet}. As shown, the images reconstructed by our method are nearly identical to the original inputs. In particular, on the FFHQ dataset, only very minor differences can be observed in fine details. 
\vspace{-1ex}

\begin{table}[!t]
\centering
\caption{
    {Reconstruction performance on the ImageNet-1K dataset for multi-scale quantization algorithms. The suffixes “–a,” “–b,” and “–c” correspond to decoder adaptation for 5, 10, and 15 epochs, respectively. For each codebook size, the best-performing result is highlighted in bold. }
    }
\resizebox{0.95\textwidth}{!}{
\begin{tabular}{@{}lcccccccc@{}}
\toprule
Methods & Tokens  & Codebook Size $K$ & $\mathcal{E}$ ({\color{green!60!black}\bfseries$\pmb{\downarrow}$}) & $\mathcal{U}$ ({\color{green!60!black}\bfseries$\pmb{\uparrow}$}) & $\mathcal{C}$ ({\color{green!60!black}\bfseries$\pmb{\uparrow}$}) & r-FID ({\color{green!60!black}\bfseries$\pmb{\downarrow}$}) \\\midrule
\multirow{2}{*}{MMD VAR-a} & 680 & 4096 & \textbf{0.255} & \textbf{100\%} & \textbf{3757.3}  &   \textbf{0.91}\\
& 680 & 8192 &  \textbf{0.234}  & \textbf{100\%} & \textbf{7539.4} &  \textbf{0.81}\\
\multirow{2}{*}{\textbf{Wasserstein VAR-a}} & 680 & 4096 & \textbf{0.255} & \textbf{100\%} &  3286.2 & 0.93 \\
& 680 & 8192 &  0.240  & \textbf{100\%} & 6518.2
 & 0.83\\\midrule
\multirow{2}{*}{MMD VAR-b} & 680 & 4096 & \textbf{0.255} & \textbf{100\%} & \textbf{3757.3}  &  \textbf{0.87}\\
& 680 & 8192 &   \textbf{0.234} & \textbf{100\%} & \textbf{7539.4} 
 & \textbf{0.78}\\
\multirow{2}{*}{\textbf{Wasserstein VAR-b}} & 680 & 4096 & \textbf{0.255} & \textbf{100\%} & 3286.2 & 0.88 \\
& 680 & 8192 &  0.240 & \textbf{100\%} & 6518.2
 & 0.79 \\\midrule
\multirow{2}{*}{MMD VAR-c} & 680 & 4096 & \textbf{0.255} & \textbf{100\%} & \textbf{3757.3}  &   0.82\\
& 680 & 8192 &  \textbf{0.234}  & \textbf{100\%} & 
 \textbf{7539.4} & 0.75\\
\multirow{2}{*}{\textbf{Wasserstein VAR-c}} & 680 & 4096 & \textbf{0.255} & \textbf{100\%} & 3286.2 & \textbf{0.81} \\
& 680 & 8192 &  0.240  & \textbf{100\%} & 6518.2
 & \textbf{0.73}\\
\bottomrule
\end{tabular}}
\vspace{-1.5ex}
\label{tab:mmd vs wasserstein}
\end{table}

\section{Discussion on Two Distributional Matching Approaches: Wasserstein VQ vs. MMD VQ}
\label{sec:mmd-vs-wasserstein}

We compare Wasserstein VQ and MMD VQ from three perspectives: visual tokenization performance, computational efficiency, and robustness to non-Gaussian distributions. First, we assess visual tokenization performance under the VQGAN framework. As shown in Table~\ref{tab:mmd vs wasserstein}, we evaluate performance with decoder adaptation conducted for 5, 10, and 15 epochs. Overall, the two methods achieve very similar performance, with Wasserstein VQ slightly outperforming MMD VQ when the decoder adaptation is extended to 15 epochs.

Second, we analyze the computational cost of the two approaches. Following the setup in Appendix~\ref{appendix:computational overhead}, we measure runtime efficiency by recording forward and backward pass times over 100 iterations. We consider two scenarios: (i) fixing the number of data samples at $N = 8192$ while varying the codebook size $K$ from 128 to 32,768, and (ii) fixing the codebook size at $K = 8192$ while varying the number of data samples $N$ from 128 to 32,768, providing complementary comparisons. As illustrated in Figures~\ref{fig:computational-codebook} and~\ref{fig:computational-data}, MMD VQ exhibits a rapidly increasing computational cost, reflecting its inefficiency. This difference arises because Wasserstein VQ performs distribution matching by aligning first- and second-order statistics, whereas MMD VQ relies on pairwise distances between all elements. Consequently, the computational complexity of MMD VQ grows substantially, especially when both the number of data samples and the codebook size are large.

\begin{figure}[!t]
	\centering
	\subfloat[Codebook Size]{ 
		\label{fig:computational-codebook}  \includegraphics[width=0.475\textwidth]{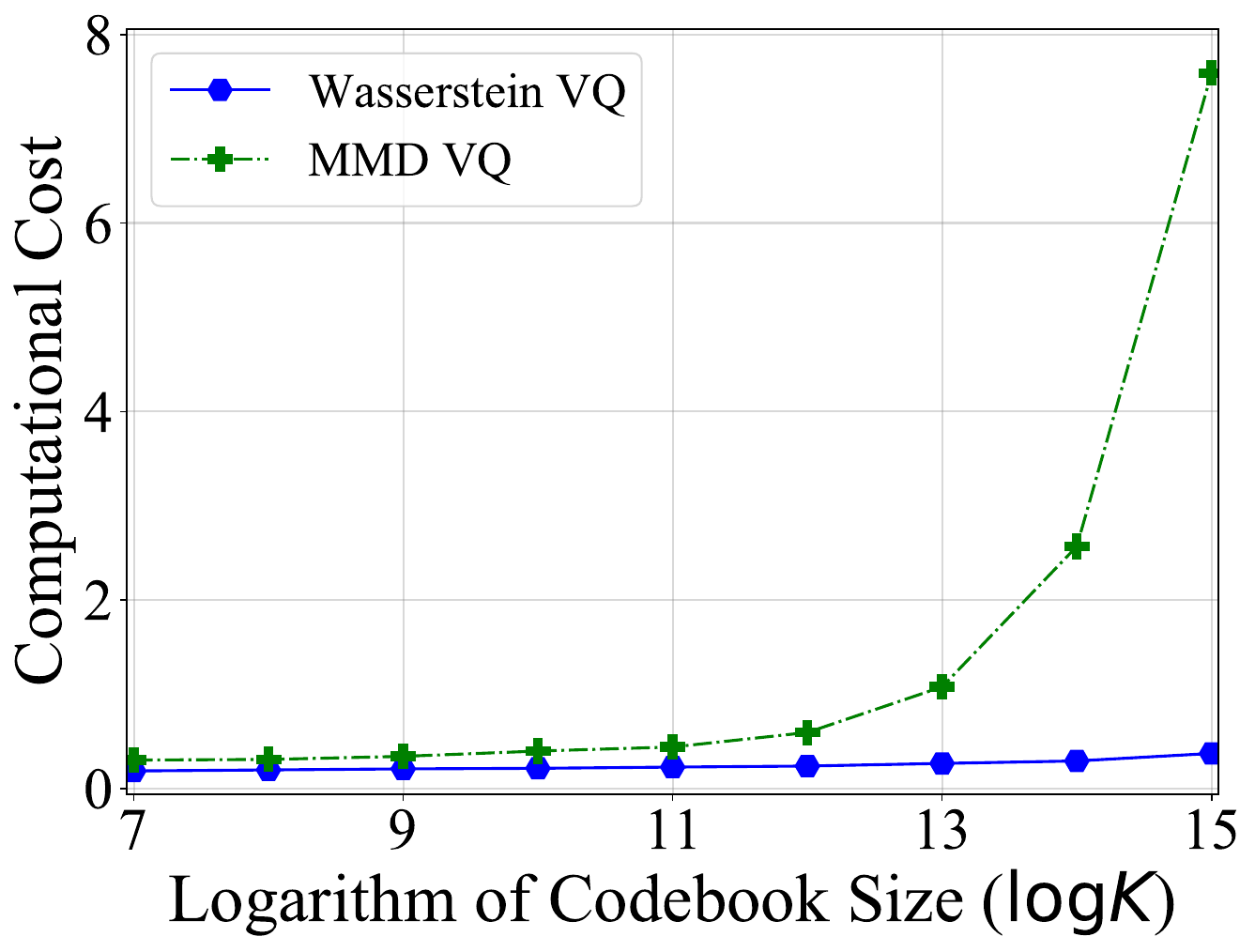}
        }
	\subfloat[Data Sample Size]{
		\label{fig:computational-data}\includegraphics[width=0.475\textwidth]{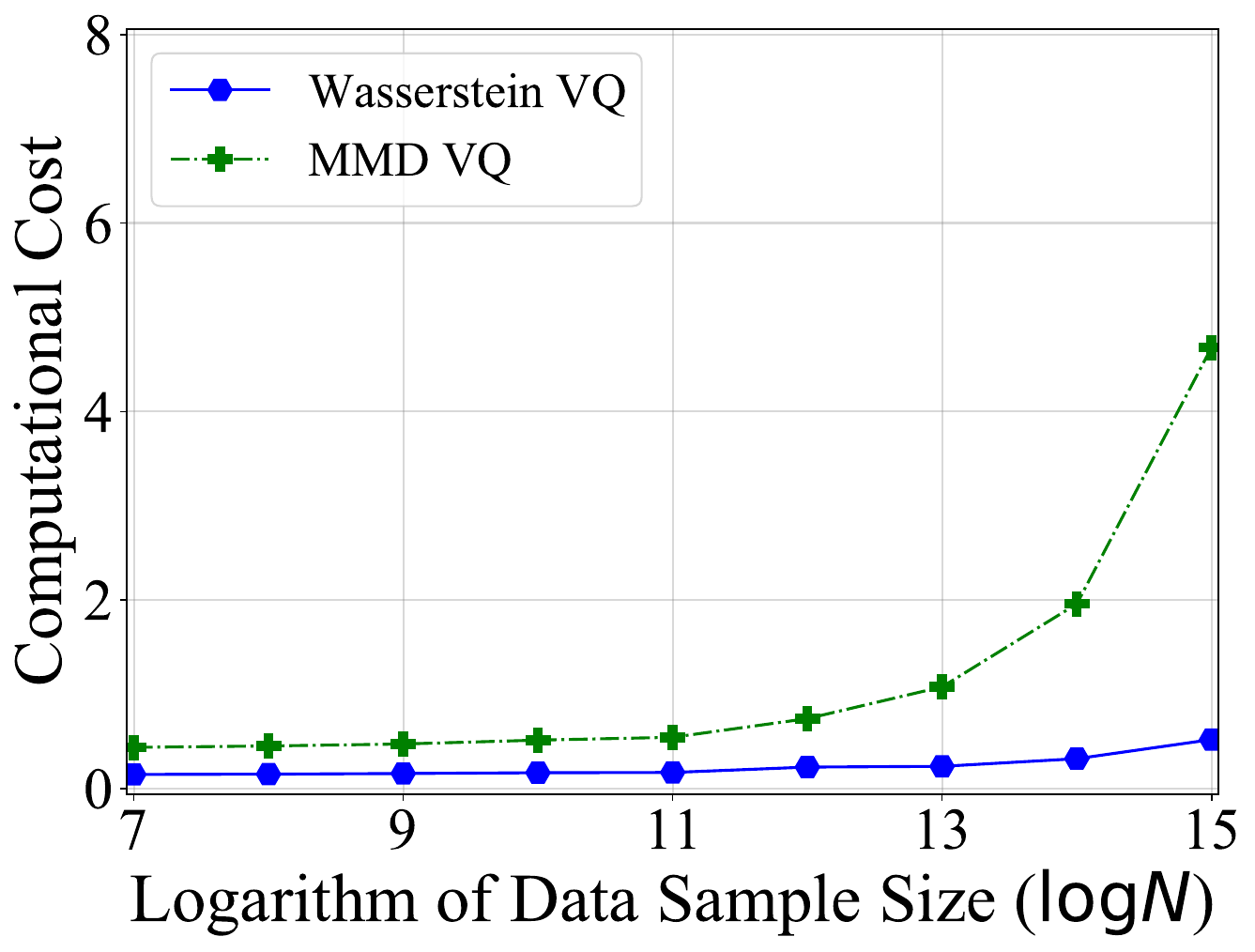}
	}
    \vspace{-2mm}
	\caption{Computational overhead (seconds) comparison between Wasserstein VQ and MMD VQ.}
	\label{fig:computational cost comparison}
    \vspace{-2mm}
\end{figure}

Third, we evaluate the robustness of Wasserstein VQ and MMD VQ under non-Gaussian latent distributions. While Section~\ref{sec:vqvae experiments} shows that the latent distributions in visual tokenization tasks are approximately Gaussian, this experiment allows us to explicitly assess robustness beyond the Gaussian regime. To simulate a non-Gaussian distribution, we follow the setting in Section~\ref{sec:atomic setting} and assume the encoder output features follow a bimodal mixture:
\begin{equation}  
z_i \sim 0.5 \cdot \mathcal{N}(\zeta \mathbf{1}, I) + 0.5 \cdot \mathcal{N}(-\zeta  \mathbf{1}, I), \nonumber
\end{equation}
where the mode separation parameter $\zeta \in \{0.0, 0.5, \dots, 4.0\}$ controls deviation from Gaussianity, with larger $\zeta$ indicating stronger non-Gaussianity. The codebook is initialized with a standard Gaussian distribution, and code vectors are treated as trainable parameters. After 10,000 training steps, we evaluate Wasserstein VQ and MMD VQ in terms of quantization error and codebook utilization rate. As shown in Figures~\ref{fig:comparison quantization error} and~\ref{fig:comparison codebook utilization}, for small $\zeta$ (nearly Gaussian), the two methods perform comparably. As $\zeta$ increases, MMD VQ consistently achieves lower quantization error and higher codebook utilization, whereas Wasserstein VQ degrades due to its limited moment-matching assumption. These synthetic experiments empirically demonstrate that MMD VQ is more robust to non-Gaussian feature distributions.

\paragraph{Summary.}  
Wasserstein VQ and MMD VQ exhibit complementary strengths and limitations. Wasserstein VQ is computationally efficient and performs stably under nearly Gaussian latent distributions, benefiting from its first- and second-order moment-matching formulation, which scales well with large datasets and codebooks. However, it is less robust when feature distributions deviate from Gaussianity. In contrast, MMD VQ excels under non-Gaussian distributions, consistently achieving lower quantization error and higher codebook utilization, but at the cost of substantially higher computational overhead. In practice, Wasserstein VQ is preferred for efficiency and standard Gaussian-like scenarios, while MMD VQ provides an advantage when robustness to non-Gaussian features is critical.
\begin{figure}[!t]
	\centering
    \vspace{-4mm}
    \subfloat[Quantization error]{
		\label{fig:comparison quantization error}\includegraphics[width=0.475\textwidth]{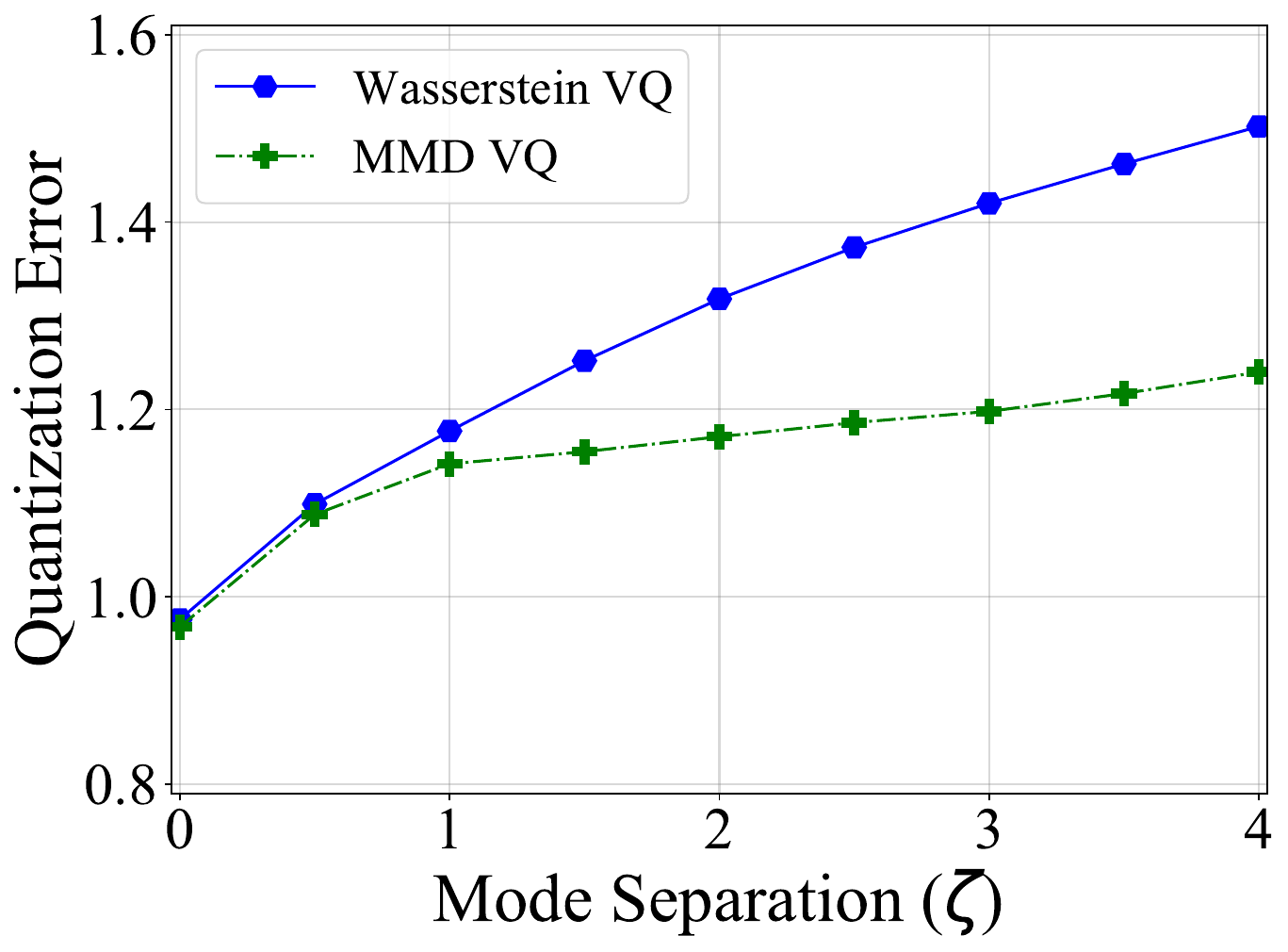}
	}
	\subfloat[Codebook utilization]{ 
		\label{fig:comparison codebook utilization}  \includegraphics[width=0.475\textwidth]{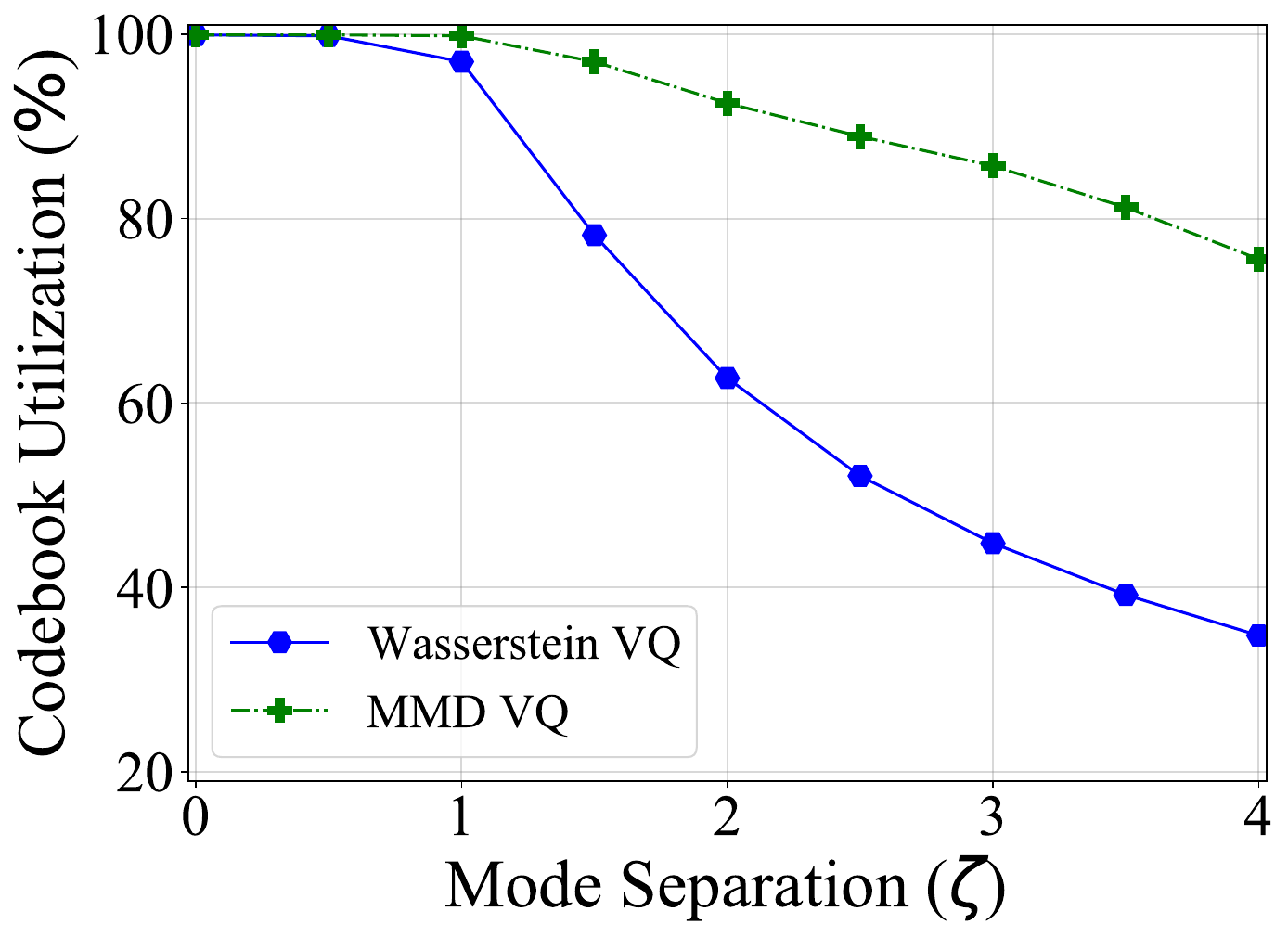}
        }
    \vspace{-1mm}
	\caption{ comparison between Wasserstein VQ and MMD VQ.}
	\label{fig:comparison non-gaussian}
    \vspace{-5mm}
\end{figure}

%% file: sections/conclusion.tex
\section{Conclusion}\label{sec:conclusion}

We identify a fundamental distributional mismatch between feature and code vectors as a common cause of training instability and codebook collapse in vector quantization. To address this issue, we propose a general distributional matching framework that formalizes desirable VQ behavior through principled criteria. Within this framework, we show that aligning feature and code distributions provides a unified mechanism for improving training stability, codebook utilization, and representation efficiency. We instantiate the framework using a Wasserstein-based objective with an efficient closed-form solution under a mild Gaussian approximation, and further demonstrate that a nonparametric alternative based on maximum mean discrepancy achieves comparable performance. More broadly, this work offers a unified perspective on vector quantization and opens new avenues for designing stable and efficient discrete representation learning methods.

\section{Limitation and Discussion}
One limitation of this work is that we do not directly validate the effectiveness of the proposed distributional matching framework through downstream generative results, due to limited computational resources. Recent studies have shown that improved reconstruction performance does not necessarily translate into better generative performance. For instance, increasing the codebook size or token length in discrete visual tokenizers~\cite{yu2024language}, or increasing the latent resolution or dimensionality in continuous tokenizers, can significantly improve reconstruction quality while simultaneously degrading generative performance~\cite{Yao2025ReconstructionVG}. 

In these cases, larger codebooks or longer token sequences in discrete tokenizers substantially increase the complexity of training autoregressive models, due to longer sequences and larger vocabularies. Similarly, higher latent resolution or dimensionality in continuous tokenizers makes the training of diffusion models more challenging, leading to reduced generative quality. 

In contrast, distributional matching improves tokenizer quality without introducing additional burdens to downstream generative models. For example, compared with VAR~\cite{Tian2024VisualAM}, Wasserstein VAR achieves improved reconstruction without increasing the codebook size or token length. As a result, the improved reconstruction performance induced by distributional matching is more likely to translate into improved generative performance.

%% file: sections/appendix.tex
\section{Optimal Support of The Codebook Distribution}

\begin{proof}[Proof of Theorem \ref{thm:1}]
    First, we assume $\overline{\supp(\cP_B)}=\overline{\supp(\cP_A)}$. Then for any $\bz\in \supp(\cP_A)$, there exist a sequence of points in $\supp(\cP_B)$ that converge to $\bz$. Let $\{\be_k\}_{k=1}^K$ be $K$ code vectors independently generated from $\cP_B$. 
    Then the empirical distribution of $\{\be_k\}_{k=1}^K$ tends to $\cP_B$ as the  size $K$ tends to infinity. 
    Since $\Omega=\supp(\cP_A)$ is a bounded region,  we have the following: 
\begin{align}    \sup_{\bz\in\overline{\supp(\cP_A)}}\min_{k}\norm{\bz-\be_k}^2= \sup_{\bz\in\overline{\supp(\cP_B)}}\min_{k}\norm{\bz-\be_k}^2\overset{p}{\to}0,\quad \textnormal{as }K\to\infty.
\end{align}
    This quantity is an upper bound on the quantization error $\cE(\{\bz_i\};\{\be_k\})$. Thus,
    \begin{align}
    \sup_{\{\bz_i\}\subseteq\Omega}\cE\left(\{\bz_i\}_{i=1}^N;\{\be_k\}_{k=1}^K\right)\leq \sup_{\bz\in\overline{\Omega}}\min_{k}\norm{\bz-\be_k}^2\overset{p}{\to}0, \quad \textnormal{as }K\to\infty.
    \end{align}
    This demonstrates that $\cP_B$ has vanishing quantization error asymptotically. Furthermore, for any $K$ code vectors $\{\be_k\}_{k=1}^K$ independently drawn from $\cP_B$, we have $\{\be_k\}_{k=1}^K\subseteq\overline{\Omega}$. Since the empirical distribution of $\{\bz_i\}_{i=1}^N$ tends to $\cP_A$ as the feature sample size $N$ tends to infinity, we can easily show that for any fixed $\{\be_k\}_{k=1}^K\subseteq\overline{\Omega}$, the codebook utility rate satisfies
    \begin{align}
    \cU\left(\{\bz_i\}_{i=1}^N,\{\be_k\}_{k=1}^K\right)\overset{p}{\to}1,\quad\textnormal{as }N\to\infty.
    \end{align}
    This shows that $\{\be_k\}_{k=1}^K$ attains full utilization asymptotically, and thus $\cP_B$ attains full utilization asymptotically. 

    On the other hand, we assume $\cP_B$ attains full utilization and vanishing quantization error asymptotically. Then we first claim that $\overline{\supp(\cP_A)}\subseteq\overline{\supp(\cP_B)}$. Since $\cP_B$ has vanishing quantization error asymptotically, then for any $\bz\in\supp(\cP_A)$, there exist a sequence of points in $\supp(\cP_B)$ that converge to $\bz$. This implies that $\supp(\cP_A)\subseteq\overline{\supp(\cP_B)}$ and thus $\overline{\supp(\cP_A)}\subseteq\overline{\supp(\cP_B)}$. 
    
    To show $\overline{\supp(\cP_B)}=\overline{\supp(\cP_A)}$, it remains to show $\supp(\cP_B)\subseteq\overline{\supp(\cP_A)}$. In fact, if $\supp(\cP_B)\subseteq\overline{\supp(\cP_A)}$ does not hold, then there exists an open region $
    \cR\subseteq \supp(\cP_B)-\overline{\supp(\cP_A)}
    $ 
    such that $\cP_B(\cR)>0$ and 
    \begin{align}
    \min_{\bz\in\supp(\cP_A),\bz'\in\cR}\norm{\bz-\bz'}\geq\epsilon_0
    \end{align}
    for some $\epsilon_0>0$. Since $\supp(\cP_A)\subseteq\overline{\supp(\cP_B)}$, then there exists a sufficiently large $K_0$ such that the event  
    \begin{align}\label{equ:codebookselection}
   \!\!\! \left\{\textnormal{Generating}\{\be_k\}_{k=1}^{K_0} \textnormal{ i.i.d. from }\cP_B \textnormal{~s.t.~}  \{\be_k\}\subseteq\supp(\cP_A), \sup_{\bz\in\supp(\cP_A)}\min_k\norm{\bz-\be_k}<\epsilon_0\right\}
    \end{align}
    has some positive probability $C>0$. Then with a positive probability of at least {$C\cdot\cP_B(\cR)$}, we can pick the first $K_0$ code vectors from Equation (\ref{equ:codebookselection}) and the $(K_0+1)$th code vector from $\cR$. For any such codebook of size $K_0+1$, we know the $(K_0+1)$th code vector will never be used regardless of the choice of the feature set $\{\bz_i\}$. Therefore, the codebook utilization 
    \begin{align}
    \sup_{\{\bz_i\}}\cU\left(\{\be_k\}_{k=1}^{K_0+1};\{\bz_i\}\right)\leq \frac{K_0}{K_0+1}<1.
    \end{align}
    This contradicts the property that $\cP_B$ attains full utilization asymptotically. Thus, $\supp(\cP_B)\subseteq\overline{\supp(\cP_A)}$ must hold. This concludes the proof.  
\end{proof} 

\section{Understanding Codebook Collapse Through the Lens of Voronoi Partition}
\label{appendix:understanding codebook collapse by Voronoi Partition}

\begin{figure*}[!h]
    \vspace{-2mm}
	\centering
	\subfloat[]{         
        \label{fig:voronoi partition a}  
		\includegraphics[width=0.24\textwidth]{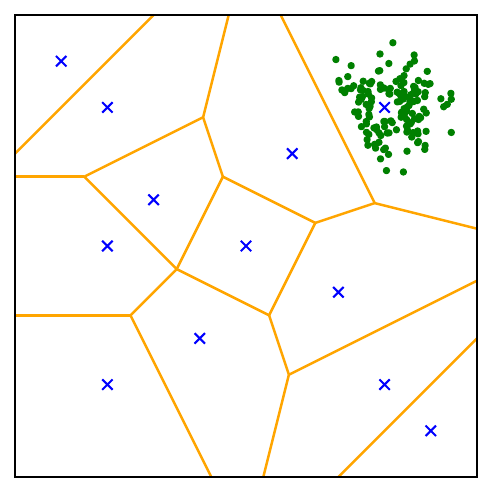}
        }
	\subfloat[]{         
        \label{fig:voronoi partition b}  
		\includegraphics[width=0.24\textwidth]{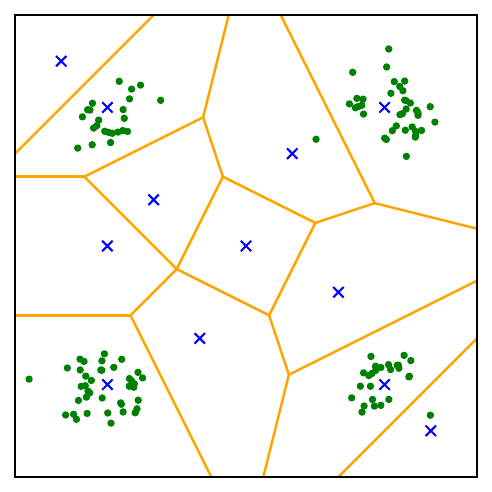}
        }
	\subfloat[]{         
        \label{fig:voronoi partition c}  
		\includegraphics[width=0.24\textwidth]{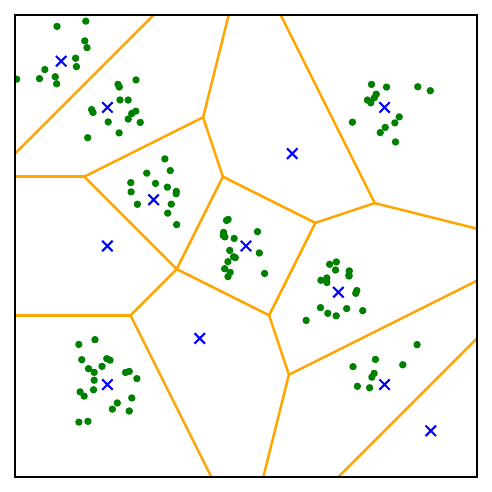}
        }
    \subfloat[]{         
        \label{fig:voronoi partition d}  
		\includegraphics[width=0.24\textwidth]{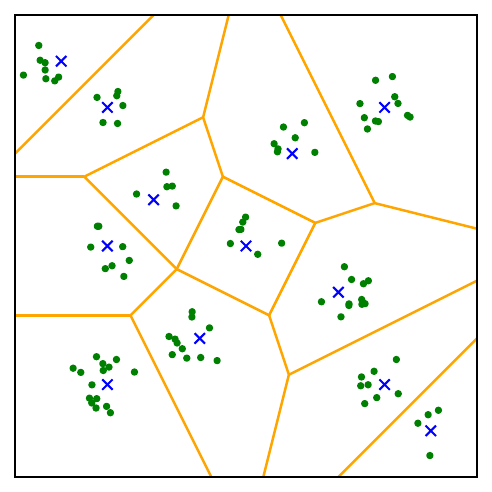}
        }
    \vspace{-1ex}
	\caption{Visualization of the Voronoi partition. The symbols $\cdot$ and $\times$ represent the feature and code vectors, respectively.}
    \vspace{-3ex}
	\label{fig:visualization of voronoi partition}
\end{figure*}

\subsection{Voronoi Partition: Definition and Connection to Codebook Collapse}
\label{appendix:voronoi partition definition}

Let $\mathcal{X}$ be a metric space with distance function $d(\cdot, \cdot)$, and let $\{\be_k\}_{k=1}^K$ denote a set of code vectors. The Voronoi cell (or Voronoi region) $\mathcal{R}_k$ associated with code vector $\be_k$ is the set of all points in $\mathcal{X}$ whose distance to $\be_k$ is no greater than their distance to any other code vector $\be_j$ ($j \neq k$):
\begin{align}
\mathcal{R}_k = \{x \in \mathcal{X} \;|\; d(x, \be_k) \le d(x, \be_j), \forall j \neq k\}.
\end{align}
The Voronoi diagram is the collection of all cells $\{\mathcal{R}_k\}_{k=1}^K$. Figure~\ref{fig:visualization of voronoi partition} illustrates a Voronoi diagram for 12 code vectors, partitioning the metric space into 12 regions according to $\mathcal{R}_k$. When $d$ is the $\ell_2$ distance, vector quantization (VQ) can equivalently be expressed in terms of Voronoi regions:
\begin{align}
\label{eq:vq voronoi partition}
\forall \bz_i \in \mathcal{R}_j, \quad \bz_i' = \argmin_{\be \in \{\be_k\}} \norm{\bz_i - \be} = \be_j,
\end{align}
where $\bz_i$ is an arbitrary feature vector. Equation~\ref{eq:vq voronoi partition} provides an alternative viewpoint for nearest neighbor search: identify the region $\mathcal{R}_j$ containing $\bz_i$ and select the corresponding code vector $\be_j$.

\paragraph{Relation to Codebook Collapse} 
Codebook collapse occurs in its most severe form when all feature vectors fall within the same Voronoi region. As illustrated in Figure~\ref{fig:voronoi partition a}, all features occupy a single region, leading to the utilization of only one code vector. To prevent collapse, feature vectors should be distributed across all regions as uniformly as possible (Figure~\ref{fig:voronoi partition d}).

\subsection{Limitations of Existing Vector Quantization Methods}
\label{appendix:VQ codebook collapse}

We analyze why standard VQ methods often fail to prevent codebook collapse, using Vanilla VQ~\cite{Oord2017NeuralDR} and $k$-means-based VQ~\cite{Razavi2019GeneratingDH} as examples. Both share a similar assignment step but differ in their update mechanisms. 

\paragraph{Assignment Step} 
Given feature vectors $\{\bz_i\}_{i=1}^N$ and code vectors $\{\be_k\}_{k=1}^K$, at iteration $t$, both methods partition the feature space into Voronoi cells and assign features to the nearest code vectors:
\begin{align}
\mathcal{R}_k^{(t-1)} = \{x \in \mathcal{X} \;|\; \left\Vert x- \be_k^{(t-1)} \right\Vert_2^{2} \le \left\Vert x- \be_j^{(t-1)} \right\Vert_2^{2}, \forall j \neq k\}, \quad \mathcal{S}_{k}^{(t)} = \{\bz_i \;|\; \bz_i \in \mathcal{R}_k^{(t-1)}\}. \nonumber
\end{align}

\paragraph{Update Step in Vanilla VQ} Code vectors are updated via gradient descent on the reconstruction loss:
\begin{align}
\mathcal{L} = \frac{1}{N}\sum_{k=1}^{K}\sum_{\bz_m \in \mathcal{S}_{k}^{(t)}} \left\Vert \bz_m - \be_k^{(t-1)} \right\Vert_2^{2}.
\end{align}

\paragraph{Update Step in $k$-means-based VQ} Code vectors are updated using an exponential moving average:
\begin{align}
\be_k^{(t)} = \alpha \be_k^{(t-1)} + (1-\alpha) \frac{1}{|\mathcal{S}_{k}^{(t)}|} \sum_{\bz_m \in \mathcal{S}_{k}^{(t)}} \bz_m.
\end{align}

\paragraph{Codebook Collapse in Vanilla and $k$-means VQ} 
Despite different update strategies, both methods can suffer from codebook collapse because the assignment step does not guarantee that all Voronoi cells receive feature vectors (Figures~\ref{fig:voronoi partition a}--\ref{fig:voronoi partition c}). Larger codebooks exacerbate the issue, leaving some cells unassigned and their corresponding code vectors underutilized.

\paragraph{Connection to Distribution Matching and Mitigation} 
As demonstrated in Section~\ref{sec:atomic setting}, both VQ methods are sensitive to codebook initialization. Collapse is mitigated only if the codebook distribution approximates the feature distribution. However, in practice, feature distributions are typically unknown and dynamically evolving. To address this, we introduce a \textbf{distributional matching constraint} that aligns the codebook distribution with the feature distribution, ensuring complete codebook utilization.

\section{The Relationship Between Gradient Gap and Quantization Error}
\label{appendix:gradient gap quantization error}

In this section, we provide a formal derivation of the relationship between the \textbf{Quantization Error} (Criterion 1) and the \textbf{Gradient Gap} caused by the Straight-Through Estimator (STE). This analysis theoretically supports our claim that minimizing quantization error via distribution matching inherently mitigates training instability.

\textbf{Problem Setup}: Let $\mathcal{L}$ denote the global loss function. In the VQ process, the encoder outputs a continuous latent vector $z_e$, which is discretized to a code $z_q$. 
During backpropagation, the non-differentiable quantization operation is bypassed using the STE, which approximates the gradient of the encoder output $\frac{\partial \mathcal{L}}{\partial z_e}$ using the gradient of the quantized code $\frac{\partial \mathcal{L}}{\partial z_q}$:
\begin{equation}
    \text{STE Approximation:} \quad \frac{\partial \mathcal{L}}{\partial z_e} \leftarrow \frac{\partial \mathcal{L}}{\partial z_q}
\vspace{-3pt}
\end{equation}
We define the \textbf{Gradient Gap} ($\mathcal{G}$) as the magnitude of the discrepancy between the true gradient and the STE approximated gradient:
\begin{equation}
    \mathcal{G} \triangleq \left\| \frac{\partial \mathcal{L}}{\partial z_e} - \frac{\partial \mathcal{L}}{\partial z_q} \right\|
\vspace{-3pt}
\end{equation}

\textbf{Taylor Expansion Analysis}: To analyze this gap, we perform a Taylor expansion of the gradient function around the quantized point $z_q$. Let us consider the gradient function with respect to the latent variable $x$. Expanding $\frac{\partial \mathcal{L}}{\partial z_e}$ at the point $x = z_q$, we obtain:

\begin{equation}
    \frac{\partial \mathcal{L}}{\partial z_e} = \left. \frac{\partial \mathcal{L}}{\partial x} \right|_{x=z_q} + \left. \frac{\partial^2 \mathcal{L}}{\partial x^2} \right|_{x=z_q} (z_e - z_q) + \mathcal{O}(\|z_e - z_q\|^2)
    \label{eq:taylor}
\vspace{-3pt}
\end{equation}

where $\left. \frac{\partial \mathcal{L}}{\partial x} \right|_{x=z_q}$ is exactly the gradient backpropagated from the quantized code, i.e., $\frac{\partial \mathcal{L}}{\partial z_q}$, and $\mathbf{H} = \left. \frac{\partial^2 \mathcal{L}}{\partial x^2} \right|_{x=z_q}$ is the Hessian matrix, representing the local curvature of the loss function at $z_q$.

\textbf{First-Order Approximation and Upper Bound}: By retaining only the first-order term (as is common in gradient gap analysis) and ignoring the higher-order terms $\mathcal{O}(\|z_e - z_q\|^2)$, the relationship can be approximated as:

\begin{equation}
    \frac{\partial \mathcal{L}}{\partial z_e} \approx \frac{\partial \mathcal{L}}{\partial z_q} + \mathbf{H} (z_e - z_q)
\vspace{-3pt}
\end{equation}

Substituting this back into the definition of the Gradient Gap $\mathcal{G}$, we get:
\begin{equation}
    \mathcal{G} = \left\| \frac{\partial \mathcal{L}}{\partial z_e} - \frac{\partial \mathcal{L}}{\partial z_q} \right\| \approx \| \mathbf{H} (z_e - z_q) \|
\vspace{-3pt}
\end{equation}

By applying the definition of the consistent matrix norm (specifically, the spectral norm for the symmetric Hessian), we derive the upper bound:

\begin{equation}
    \mathcal{G} \le \| \mathbf{H} \|_2 \cdot \| z_e - z_q \|
    \label{eq:bound}
\vspace{-3pt}
\end{equation}
where $\| \mathbf{H} \|_2$ represents the spectral norm of the Hessian (indicative of the maximum curvature) and $\| z_e - z_q \|$ corresponds to the \textbf{Quantization Error}.

The derivation in Eq. \ref{eq:bound} explicitly demonstrates that the Gradient Gap $\mathcal{G}$ is linearly bounded by the Quantization Error. In methods like Vanilla VQ, a large Quantization Error implies a looser bound on $\mathcal{G}$, leading to inaccurate gradient signals passed to the encoder. This large discrepancy causes the ``Training Instability'' observed in previous works. By minimizing the Wasserstein distance, our method explicitly minimizes the Quantization Error (Criterion 1). This tightens the upper bound on $\mathcal{G}$, ensuring that the gradient passed to the encoder ($\frac{\partial \mathcal{L}}{\partial z_q}$) is a faithful approximation of the true gradient ($\frac{\partial \mathcal{L}}{\partial z_e}$).

\section{Complementary Roles of Criterion~\ref{criteria:cu} and~\ref{criteria:cp} in Assessing Codebook Collapse} 
\label{appendix:explanation on criterions}

\begin{figure*}[!h]
    \vspace{-6ex}
	\centering
	\subfloat[$(50\%, 4.92)$]{         
        \label{fig:criterion sample a}  
         \includegraphics[width=0.23\textwidth]{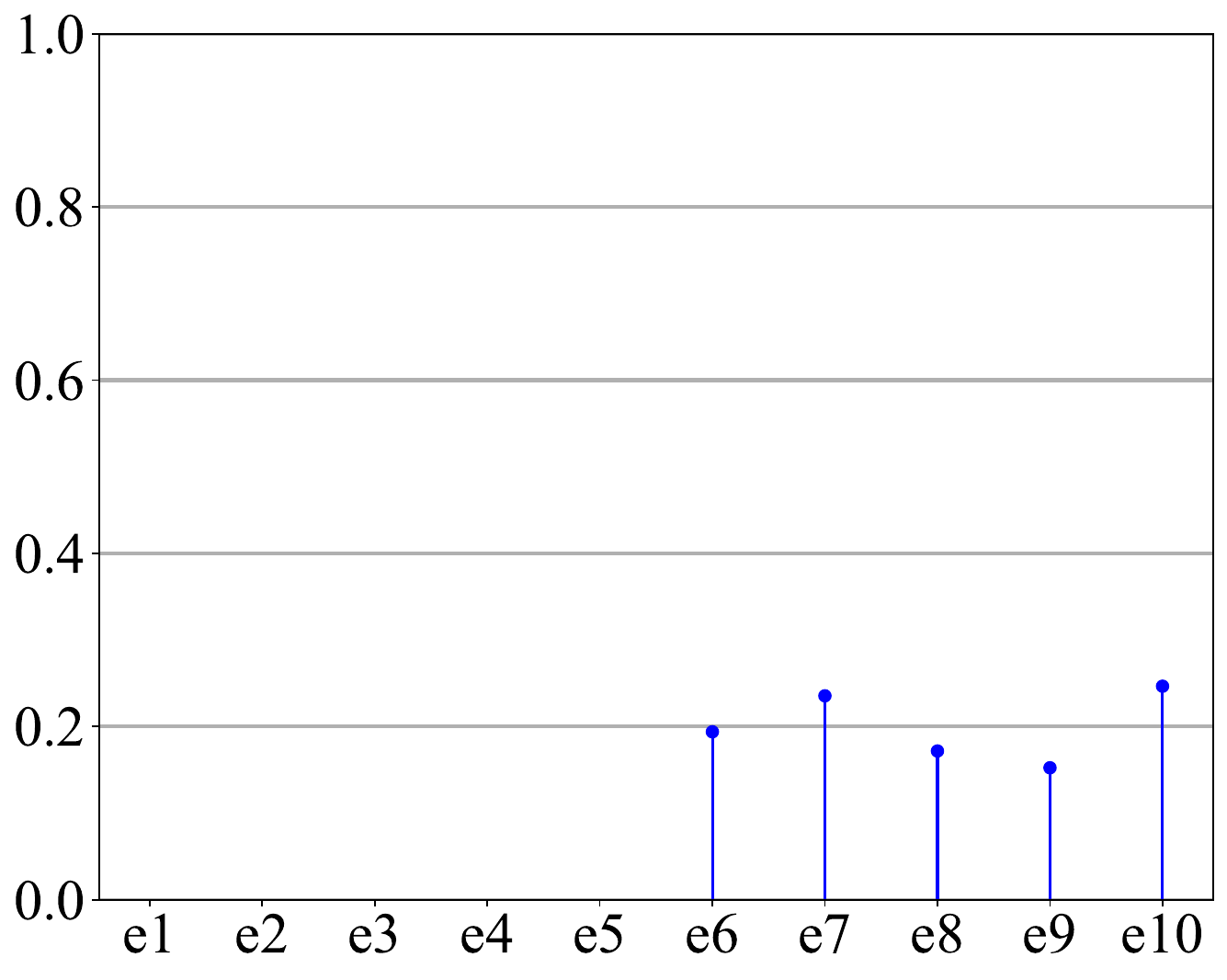}
        }
	\subfloat[$(100\%, 10.00)$]{        
        \label{fig:criterion sample b}  
		\includegraphics[width=0.23\textwidth]{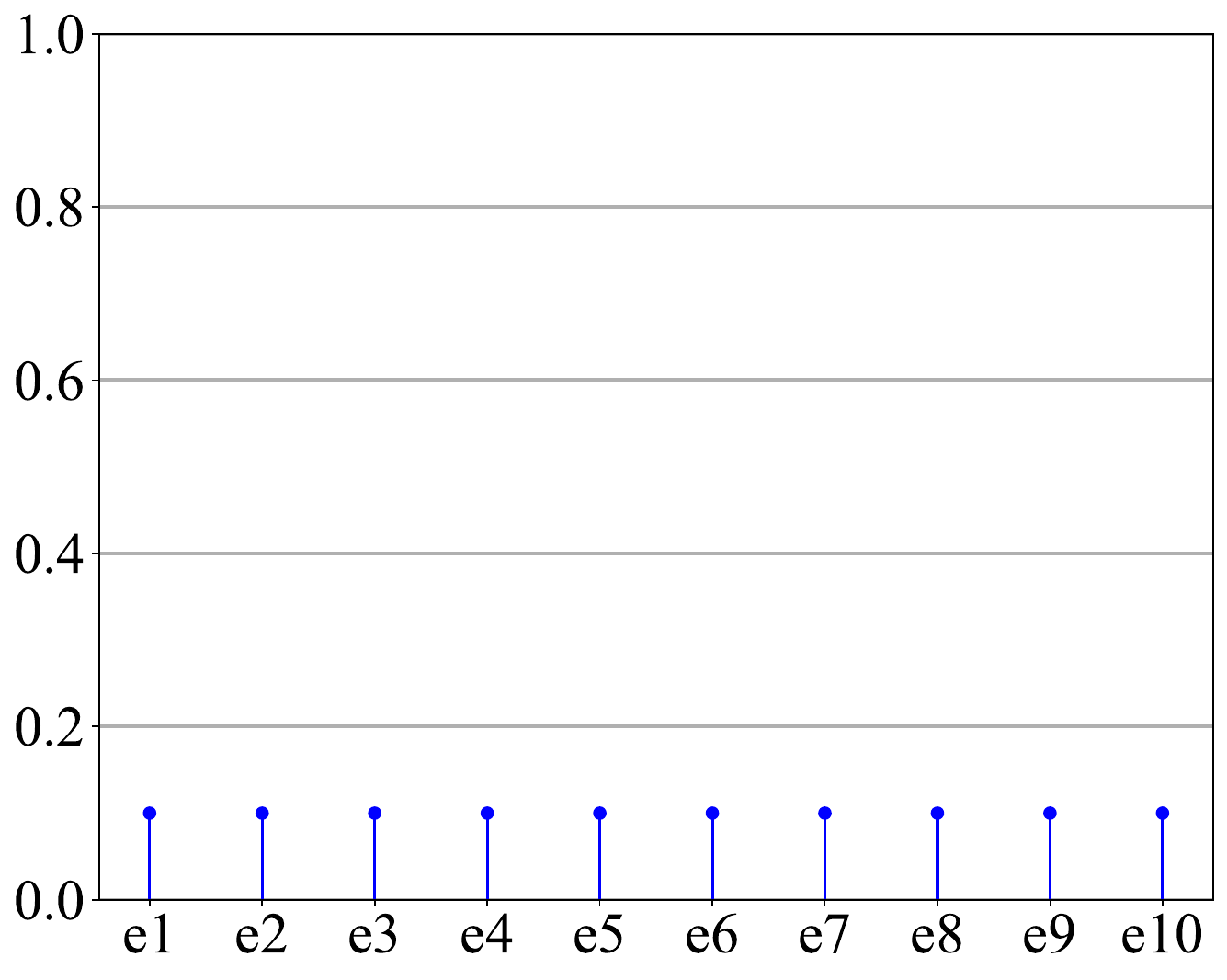}
        }
	\subfloat[$(100\%, 1.02)$]{         
        \label{fig:criterion sample c}  
		\includegraphics[width=0.23\textwidth]{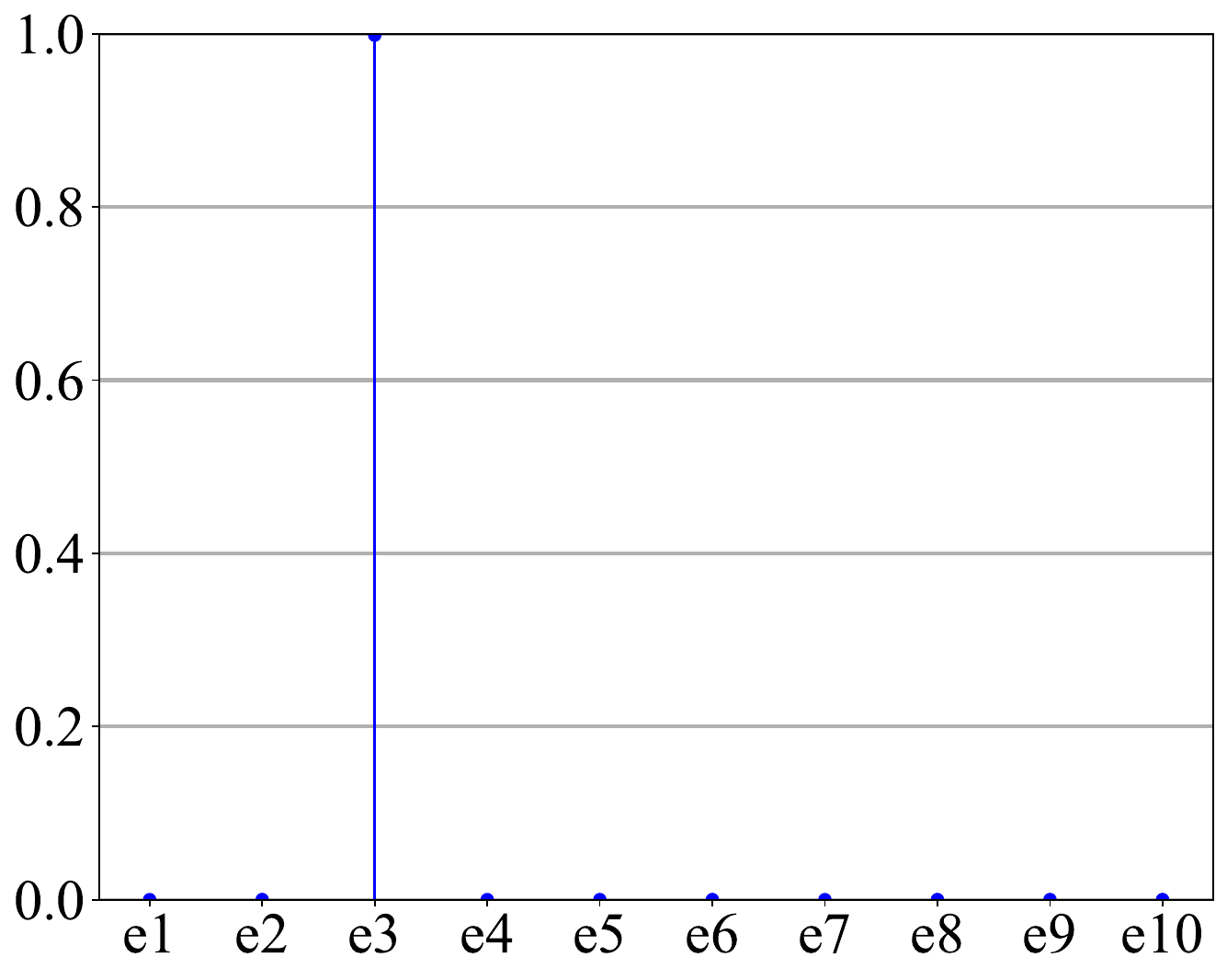}
        }
    \subfloat[$(100\%, 4.92)$]{         
        \label{fig:criterion sample d}  
		\includegraphics[width=0.23\textwidth]{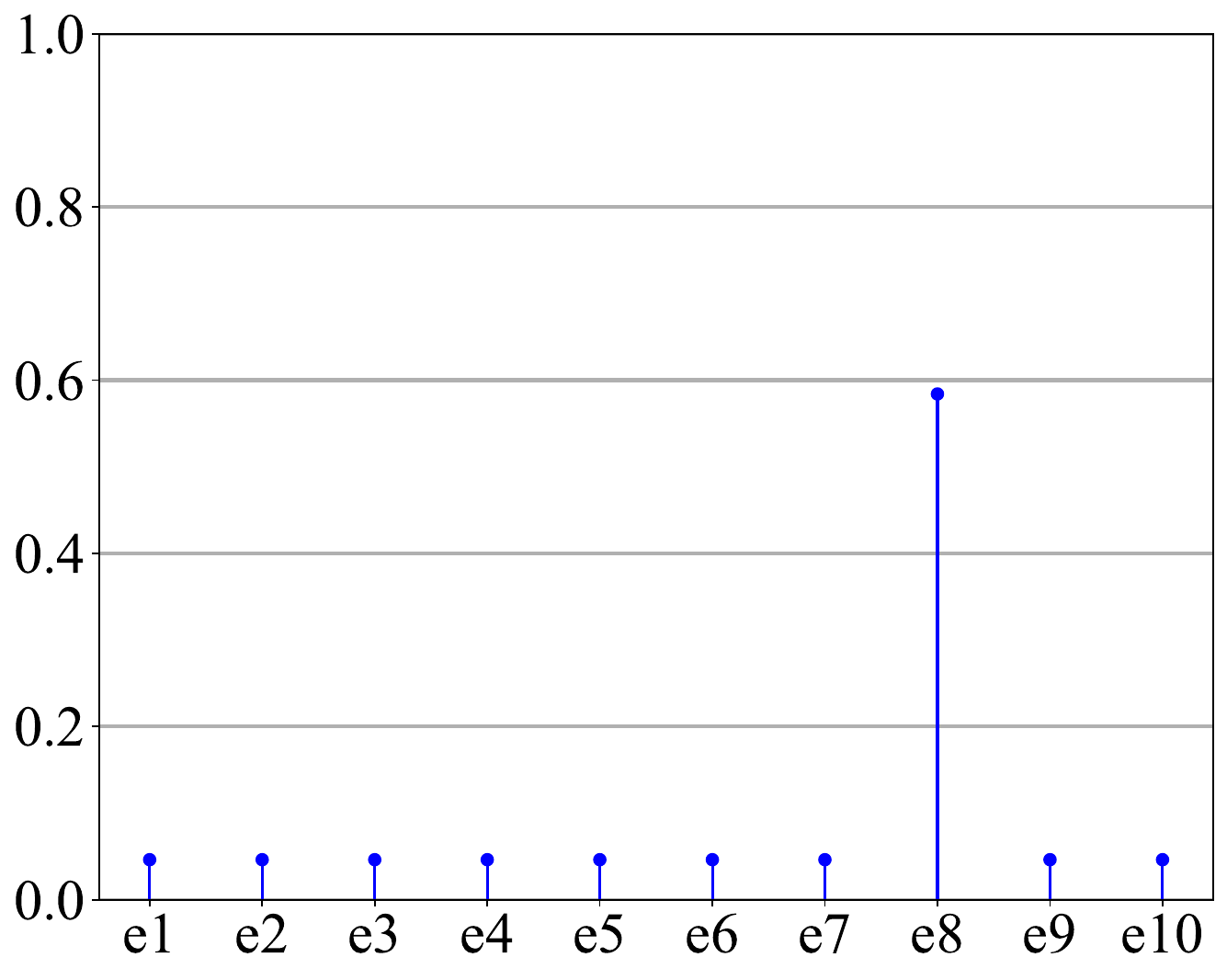}
        }
    \vspace{-1ex}
	\caption{Visualization of the evaluation criteria $(\cU, \cC)$.}
    \vspace{-2ex}
	\label{fig:visualization of criterion triple}
\end{figure*}

To clarify the complementary roles of Criterion~\ref{criteria:cu} and Criterion~\ref{criteria:cp}, which are introduced in Section~\ref{sec:distributional formulation}, we provide visual illustrations to facilitate intuitive understanding. The metric $\mathcal{U}$ is designed to quantify the \emph{completeness} of codebook utilization. As shown in Figures~\ref{fig:criterion sample a} and~\ref{fig:criterion sample b}, the corresponding values of $\mathcal{U}$ are $50\%$ and $100\%$, respectively\footnote{This difference arises because, in Figure~\ref{fig:criterion sample a}, only half of the code vectors have nonzero utilization probabilities $p_k$, as defined in Criterion~\ref{criteria:cp}, whereas in Figure~\ref{fig:criterion sample b}, all code vectors satisfy $p_k>0$.}.

However, $\mathcal{U}$ alone is insufficient to characterize the severity of codebook collapse, since it does not account for imbalance among utilized code vectors. This limitation is illustrated in Figure~\ref{fig:criterion sample c}. Although all code vectors are utilized, resulting in $\mathcal{U}=100\%$, the code vector $e_3$ overwhelmingly dominates the assignment distribution. Such extreme imbalance constitutes a degenerate form of codebook collapse, even though $\mathcal{U}$ attains its maximum value. This observation motivates the introduction of Criterion~\ref{criteria:cp}, which explicitly measures the uniformity, or conversely the imbalance, of codebook utilization.

A comparison between Figures~\ref{fig:criterion sample b} and~\ref{fig:criterion sample c} further highlights the discriminative capability of Criterion~\ref{criteria:cp}. While both cases share the same utilization completeness, namely $\mathcal{U}=100\%$, their corresponding values of $\mathcal{C}$ differ substantially, being $10.00$ and $1.02$, respectively. This contrast demonstrates that Criterion~\ref{criteria:cp} effectively captures differences in the distribution of $p_k$ that are not reflected by $\mathcal{U}$. In particular, the near minimal value of $\mathcal{C}$ in Figure~\ref{fig:criterion sample c} correctly identifies this scenario as a collapsed codebook, which aligns with our intended interpretation.

Nevertheless, Criterion~\ref{criteria:cp} alone is also insufficient for a comprehensive evaluation of codebook collapse. As illustrated by Figures~\ref{fig:criterion sample a} and~\ref{fig:criterion sample d}, identical values of $\mathcal{C}$ can correspond to markedly different levels of utilization completeness, as indicated by their substantially different values of $\mathcal{U}$. This discrepancy shows that $\mathcal{C}$ by itself cannot quantify the proportion of actively utilized code vectors.

Consequently, in this work, we adopt the joint use of Criterion~\ref{criteria:cu} and Criterion~\ref{criteria:cp} to quantitatively assess codebook collapse. Effective mitigation is achieved only when both metrics attain sufficiently large values, indicating that the codebook is both fully utilized and balanced in its assignment distribution.

\section{The Significant Impact of Distribution Variance on Quantization Error} 
\label{appendix:Impact of Distribution Variance}
As discussed in Section~\ref{sec:effects of distribution matching} and~\ref{sec:theoretical analysis}, the optimal criterion triple is achieved when the distributions $\mathcal{P}_A$ and $\mathcal{P}_B$ are identical. In this section, we further analyze the criterion triple by the lens of distribution variance under the condition that both distributions are identical. Specifically, we first sample a set of feature vectors $\{\bz_i\}_{i=1}^N$ along with a set of code vectors $\{\be_k\}_{k=1}^K$ from the distribution $\mathcal{N}_d(\m 0, \sigma^2\bm I)$ or the distribution $\unif_d(-\zeta, \zeta)$. We then calculate the evaluation criteria according to their definitions in Section~\ref{sec:distributional formulation}. As demonstrated in Table~\ref{table:distribution variance}, $\sigma$ and $\zeta$ have a substantial impact on $\cE$, while $\cU$ and $\mathcal{C}$ remains largely unaffected. 

\begin{table*}[!h]  
\centering
\vspace{-2ex}
\caption{The criterion triple influence by the distribution variance.}
\vspace{-1ex}
\resizebox{0.9\textwidth}{!}{
\begin{tabular}{c|c|c|c|c|c|c|c|c|c|c|}\hline
\multirow{2}{*}{Evaluation Criteria}  & \multicolumn{5}{|c}{$\sigma$} & \multicolumn{5}{|c}{$\zeta$}\\\cline{2-11} 
& 0.0001 & 0.001 & 0.01 & 0.1 & 1.0 & 0.0001 & 0.001 & 0.01 & 0.1 & 1.0 \\\hline
$\cE$  & 1.25e-8 & 1.25e-6 & 1.25e-4 & 1.24e-2 & 1.25 & 3.27e-9 & 3.27e-7 & 3.27e-5 & 3.27e-3 & 0.327 \\
$\cU$ & 0.9934 & 0.9938 & 0.9940 & 0.9934 & 0.9941 & 0.9993 & 0.9986 & 0.9990 & 0.9992 & 0.9989 \\
$\mathcal{C}$ & 7265.3 & 7260.3 & 7267.7 & 7255.0  & 7275.8 & 7380.2 & 7372.2 & 7387.9 & 7397.5 & 7391.6 \\\hline
\end{tabular}}
\vspace{-2ex}
\label{table:distribution variance}
\end{table*}

This experimental finding suggests that when the distribution variance of the feature vectors is uncontrollable or unknown, reporting a comparison of quantization error among various VQ methods is unreasonable. This is because the improvement in quantization error is predominantly attributed to the reduction in distribution variance rather than the effectiveness of the VQ methods. To evaluate various VQ methods in terms of the criterion triple, we establish an atomic and fair experimental setting in Section~\ref{sec:atomic setting}, where the feature distributions for all VQ methods are identical.

\section{Interpretation of Qualitative Distributional Matching Results}
\label{appendix:prototypical study}

This section interprets the experimental results presented in Figure~\ref{fig:criterion triple analysis}. The VQ process relies on nearest neighbor search for code vector selection. As evident from Figure~\ref{fig:diagram 1 group 1} to~\ref{fig:diagram 4 group 1}, actively selected code vectors are predominantly those located in close proximity to or within the feature distribution, while distant ones remain unselected. This leads to highly uneven code vector utilization $p_k$, with those closer to the feature distribution being excessively used. This elucidates the significantly low $\mathcal{U}$ and $\mathcal{C}$ observed in Figure~\ref{fig:diagram 1 group 1}. Furthermore, a notable quantization error, e.g., $\mathcal{E}=1.19$ in Figure~\ref{fig:diagram 1 group 1}, arises when the codebook and feature distributions are mismatched, forcing feature vectors outside the codebook to settle for distant code vectors. Conversely, as the disk centers align, leading to a closer match between the two distributions, an increased number of code vectors become actively engaged. Additionally, code vectors are utilized more uniformly, and feature vectors can select nearer counterparts. This accounts for the improvement of criterion triple values towards optimality as the distributions align.

Analogously, we can employ nearest neighbor search to interpret the second case. When code vectors are distributed within the range of feature vectors, as illustrated in Figure~\ref{fig:diagram 1 group 2} and Figure~\ref{fig:diagram 2 group 2}, the majority of code vectors would be actively utilized, ensuring high $\mathcal{U}$. However, the utilization of these code vectors is not uniform; code vectors on the periphery of the codebook distribution are more frequently used, leading to relatively low $\mathcal{C}$. Feature vectors on the periphery will have larger distances to their nearest code vectors, resulting in higher $\mathcal{E}$. Conversely, when feature vectors fall within the range of code vectors, as depicted in Figure~\ref{fig:diagram 3 group 2} and Figure~\ref{fig:diagram 4 group 2}, outer code vectors remain largely unused, leading to a lower $\mathcal{U}$ and $\mathcal{C}$. Since only inner code vectors are active, each feature vector can find a nearby counterpart, maintaining low $\mathcal{E}$.

\section{Supplementary Quantitative Analyses on Distribution Matching: Further Supporting the Main Findings in Section~\ref{sec:effects of distribution matching}}
\label{appendix:supplementary comprehensive quantitative analyses}
To further elucidate the effects of the distributional matching, we conduct more quantitative analyses centered around the criterion triple $(\cE, \cU, \mathcal{C})$.

\subsection{Codebook Distribution and Feature Distribution are Gaussian Distributions}
\label{appendix:analyses under gaussian distribution}

\begin{figure*}[!t]
    \vspace{-2ex}
    \centering

    \subfloat[$\cE$ {w.r.t.} $K$]{ 
        \label{fig:gaussian_codebooksize_mean_error}  
        \includegraphics[width=0.16\textwidth]{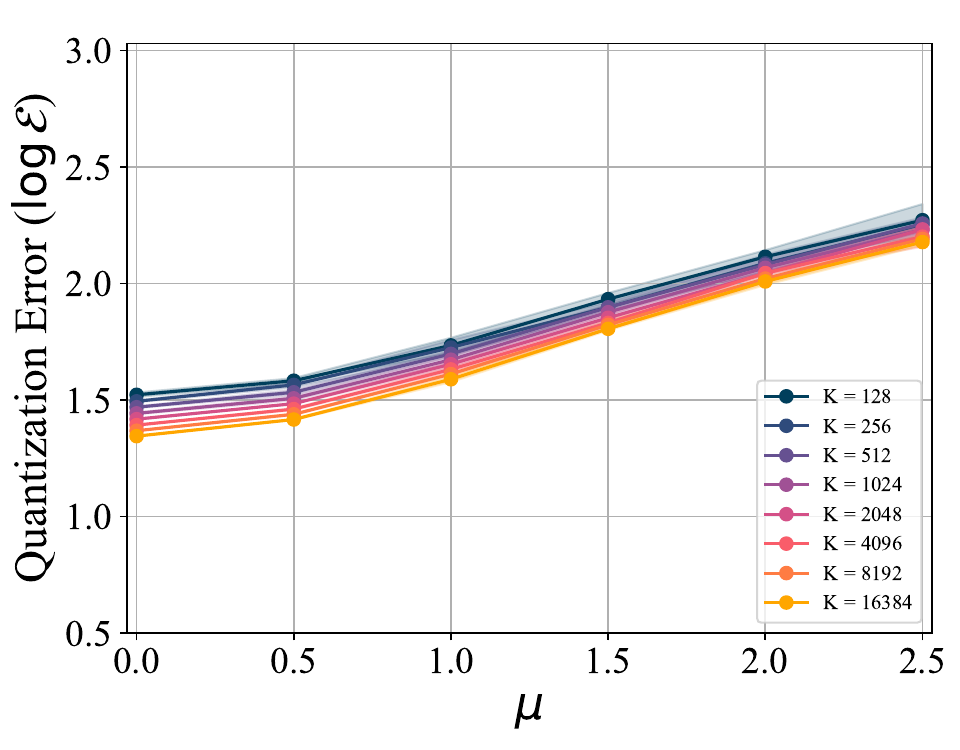}
    }
    \hspace{-0.39cm}
    \subfloat[$\cE$ {w.r.t.} $d$]{
        \label{fig:gaussian_featuredim_mean_error}
        \includegraphics[width=0.16\textwidth]{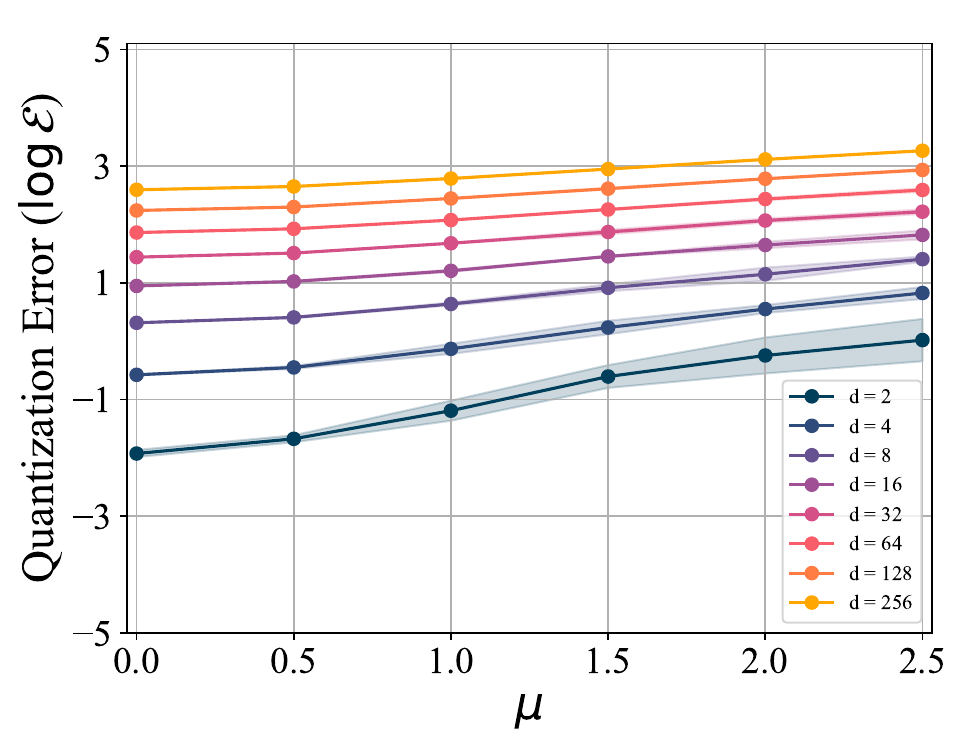}
    }
    \hspace{-0.39cm}
    \subfloat[$\cE$ {w.r.t.} $N$]{
        \label{fig:gaussian_featuresize_mean_error}
        \includegraphics[width=0.16\textwidth]{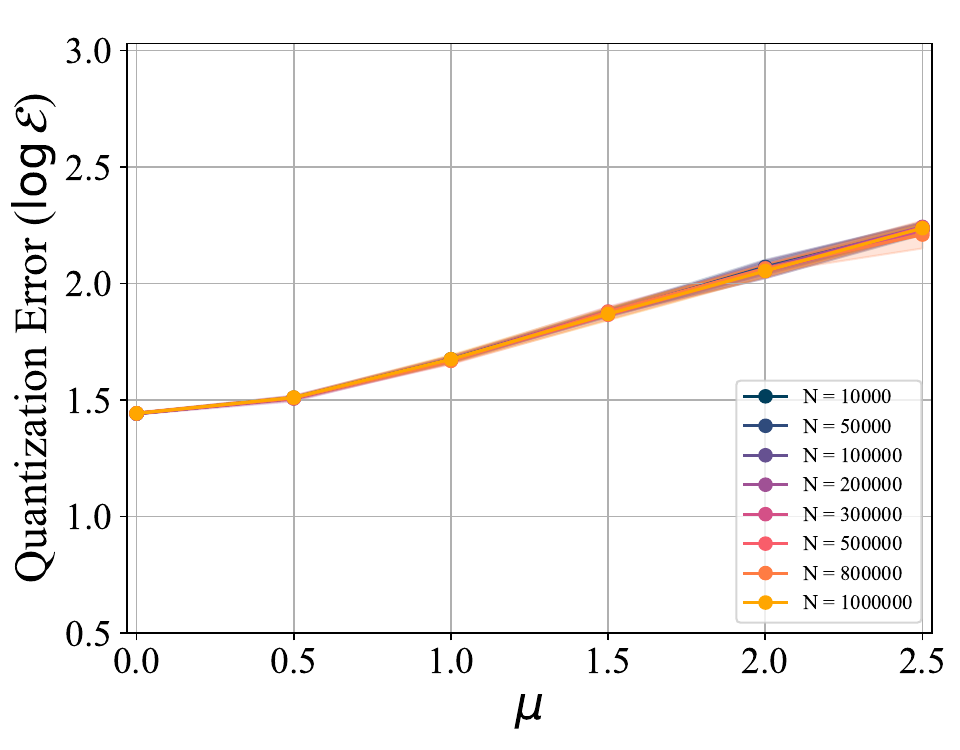}
    }
    \hspace{-0.39cm}
    \subfloat[$\cE$ {w.r.t.} $K$]{ 
        \label{fig:gaussian_codebooksize_sigma_error}  
        \includegraphics[width=0.16\textwidth]{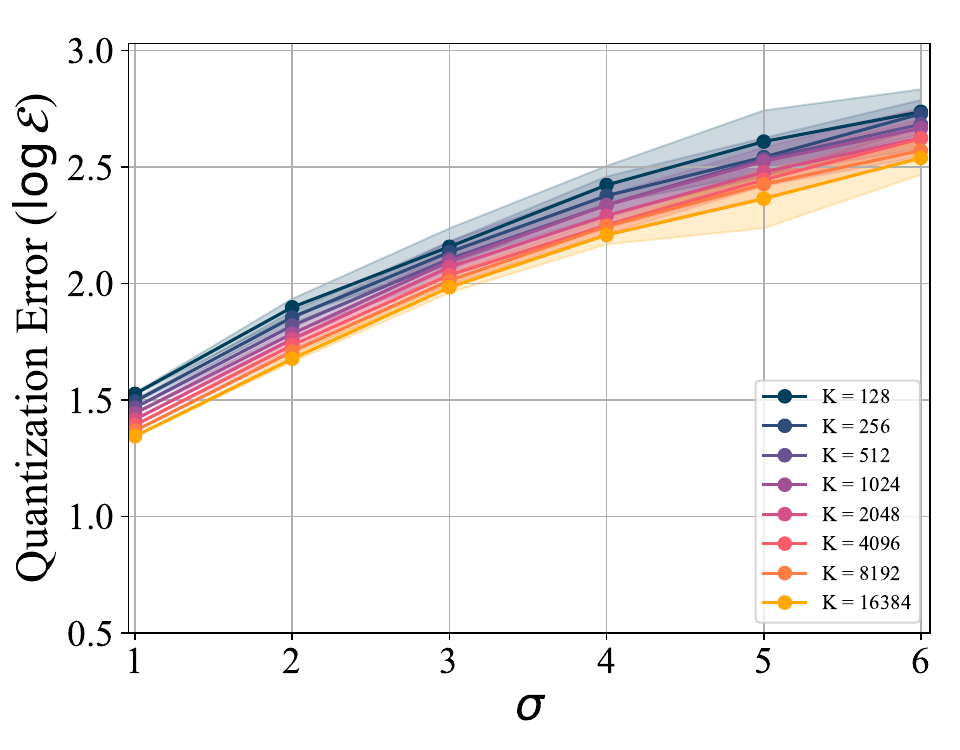}
    }
    \hspace{-0.39cm}
    \subfloat[$\cE$ {w.r.t.} $d$]{
        \label{fig:gaussian_featuredim_sigma_error}
        \includegraphics[width=0.16\textwidth]{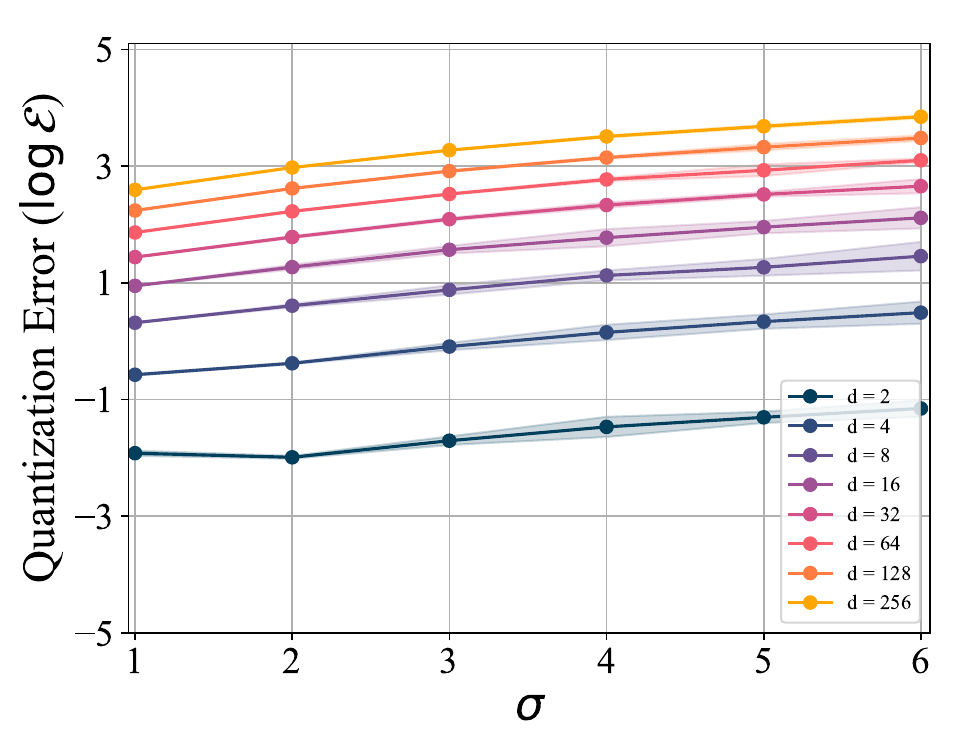}
    }
    \hspace{-0.39cm}
    \subfloat[$\cE$ {w.r.t.} $N$]{
        \label{fig:gaussian_featuresize_sigma_error}
        \includegraphics[width=0.16\textwidth]{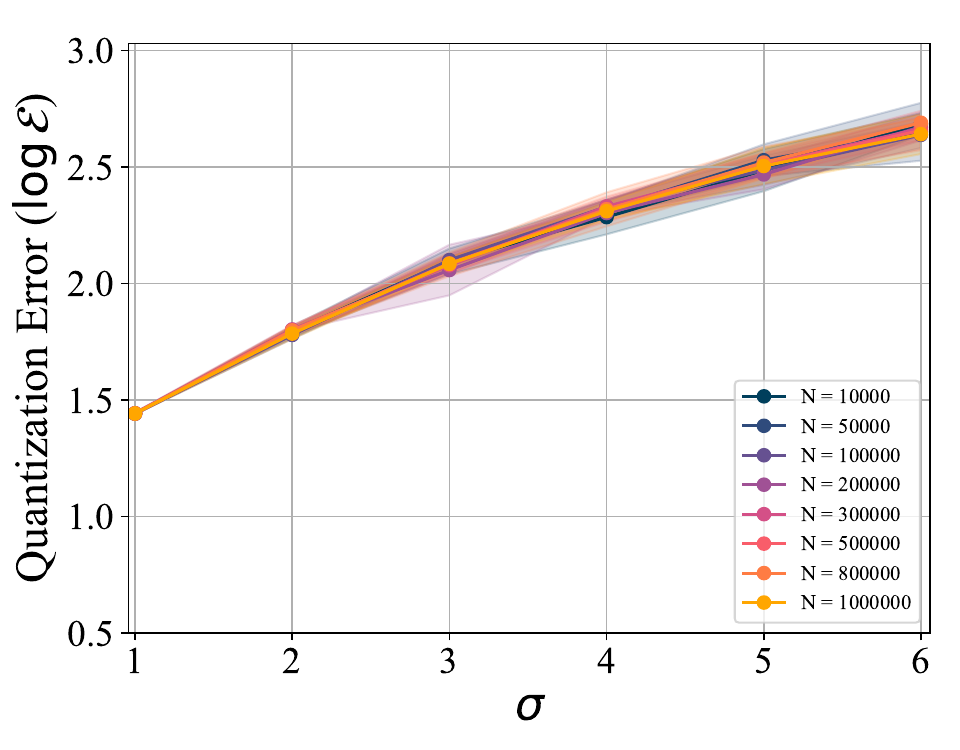}
    }

    \subfloat[$\cU$ {w.r.t.} $K$]{ 
        \label{fig:gaussian_codebooksize_mean_utilization}  
        \includegraphics[width=0.16\textwidth]{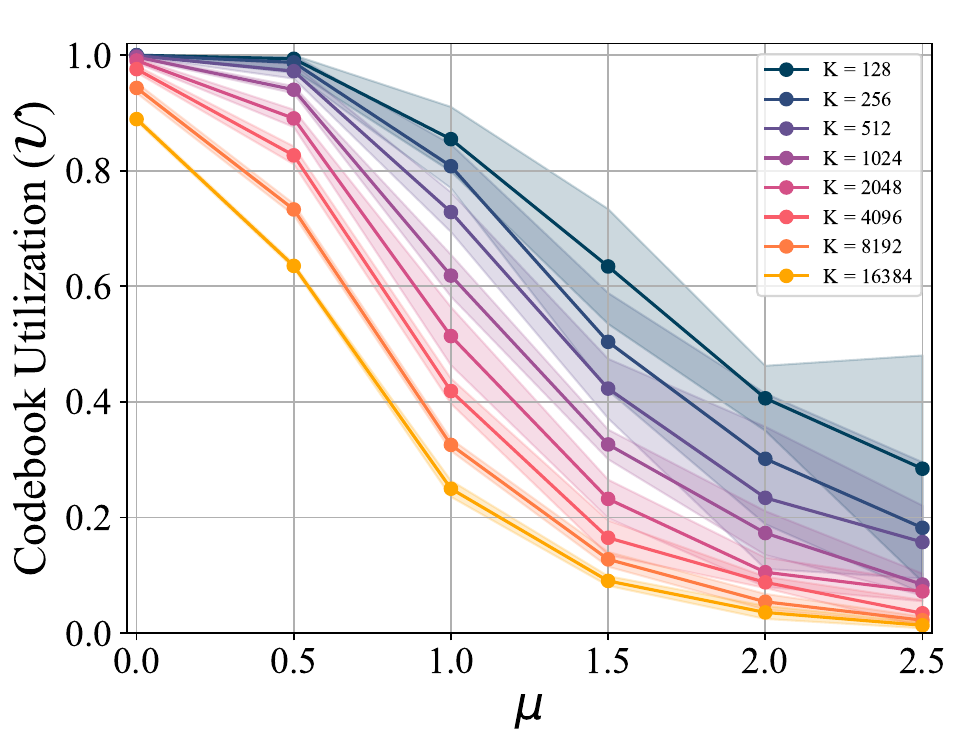}
    }
    \hspace{-0.39cm}
    \subfloat[$\cU$ {w.r.t.} $d$]{
        \label{fig:gaussian_featuredim_mean_utilization}
        \includegraphics[width=0.16\textwidth]{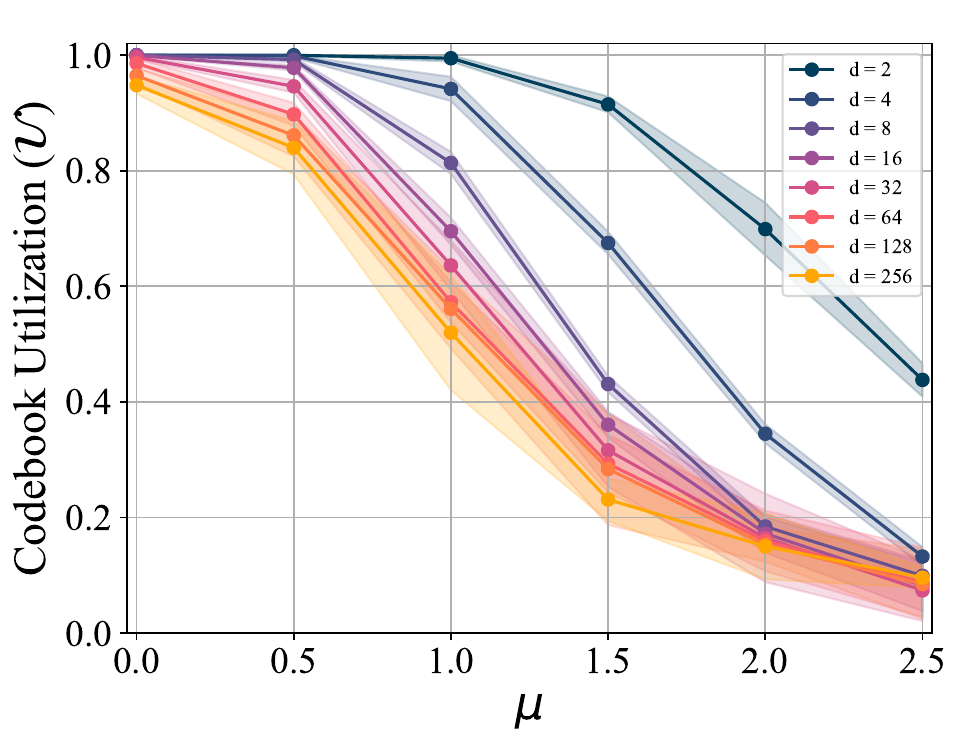}
    }
    \hspace{-0.39cm}
    \subfloat[$\cU$ {w.r.t.} $N$]{
        \label{fig:gaussian_featuresize_mean_utilization}
        \includegraphics[width=0.16\textwidth]{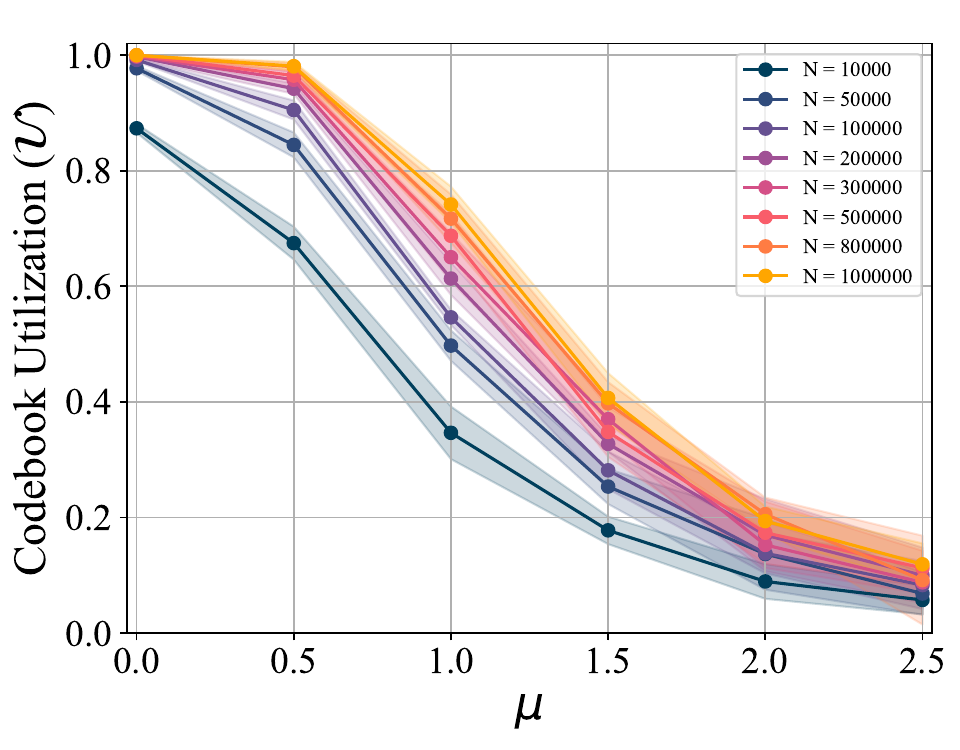}
    }
    \hspace{-0.39cm}
    \subfloat[$\cU$ {w.r.t.} $K$]{ 
        \label{fig:gaussian_codebooksize_sigma_utilization}  
        \includegraphics[width=0.16\textwidth]{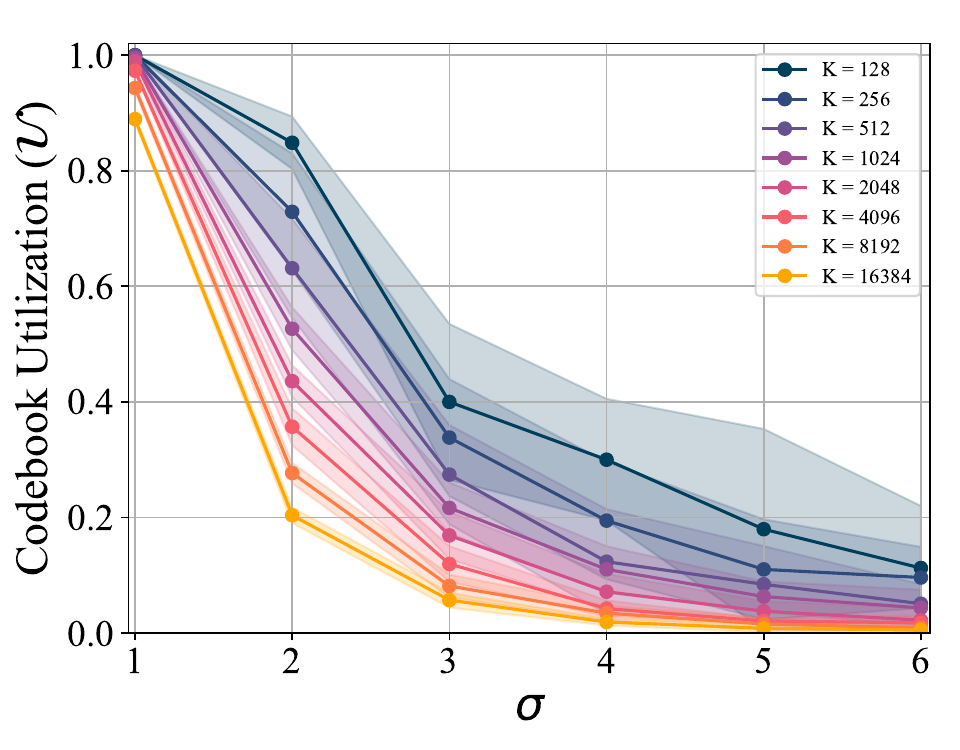}
    }
    \hspace{-0.39cm}
    \subfloat[$\cU$ {w.r.t.} $d$]{
        \label{fig:gaussian_featuredim_sigma_utilization}
        \includegraphics[width=0.16\textwidth]{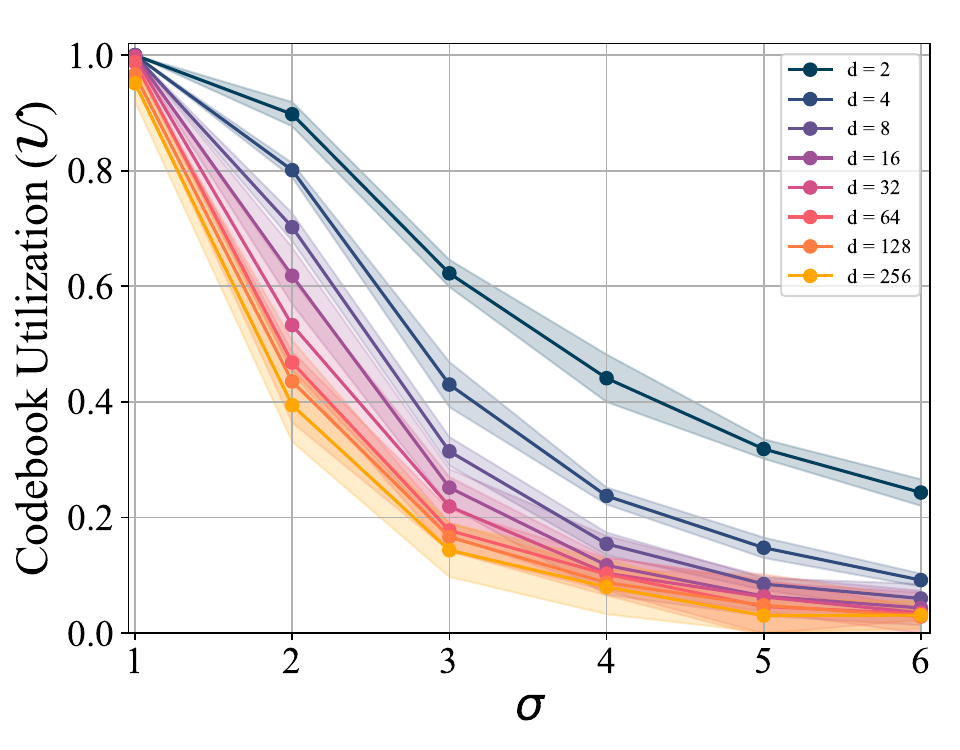}
    }
    \hspace{-0.39cm}
    \subfloat[$\cU$ {w.r.t.} $N$]{
        \label{fig:gaussian_featuresize_sigma_utilization}
        \includegraphics[width=0.16\textwidth]{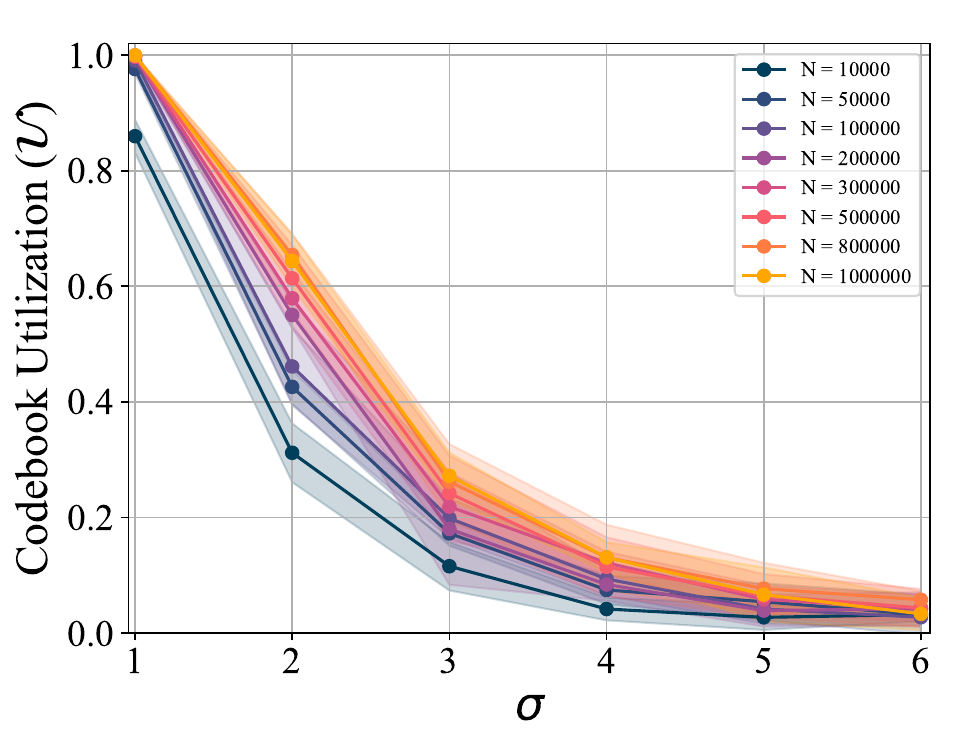}
    }

    \subfloat[$\mathcal{C}$ {w.r.t.} $K$]{ 
        \label{fig:gaussian_codebooksize_mean_perplexity}  
        \includegraphics[width=0.16\textwidth]{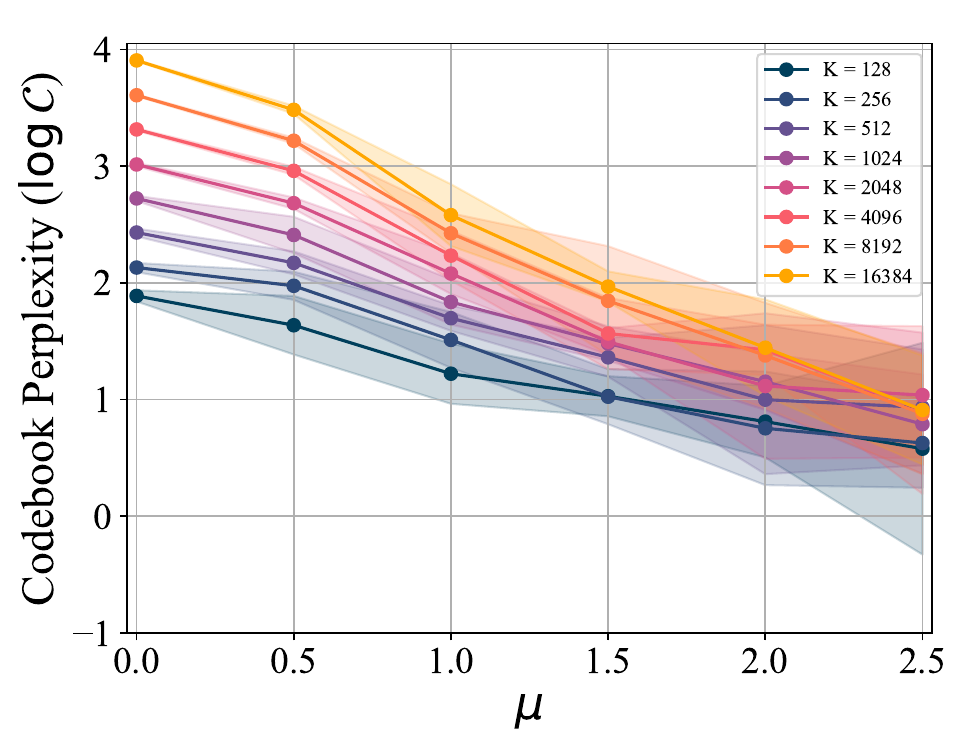}
    }
    \hspace{-0.39cm}
    \subfloat[$\mathcal{C}$ {w.r.t.} $d$]{
        \label{fig:gaussian_featuredim_mean_perplexity}
        \includegraphics[width=0.16\textwidth]{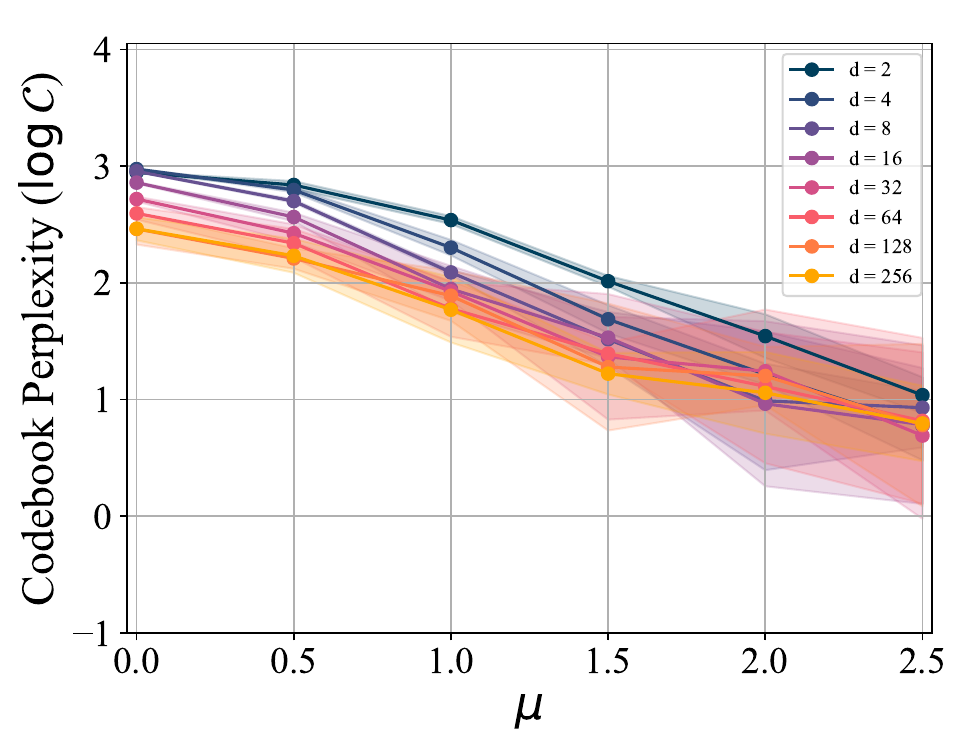}
    }
    \hspace{-0.39cm}
    \subfloat[$\mathcal{C}$ {w.r.t.} $N$]{
        \label{fig:gaussian_featuresize_mean_perplexity}
        \includegraphics[width=0.16\textwidth]{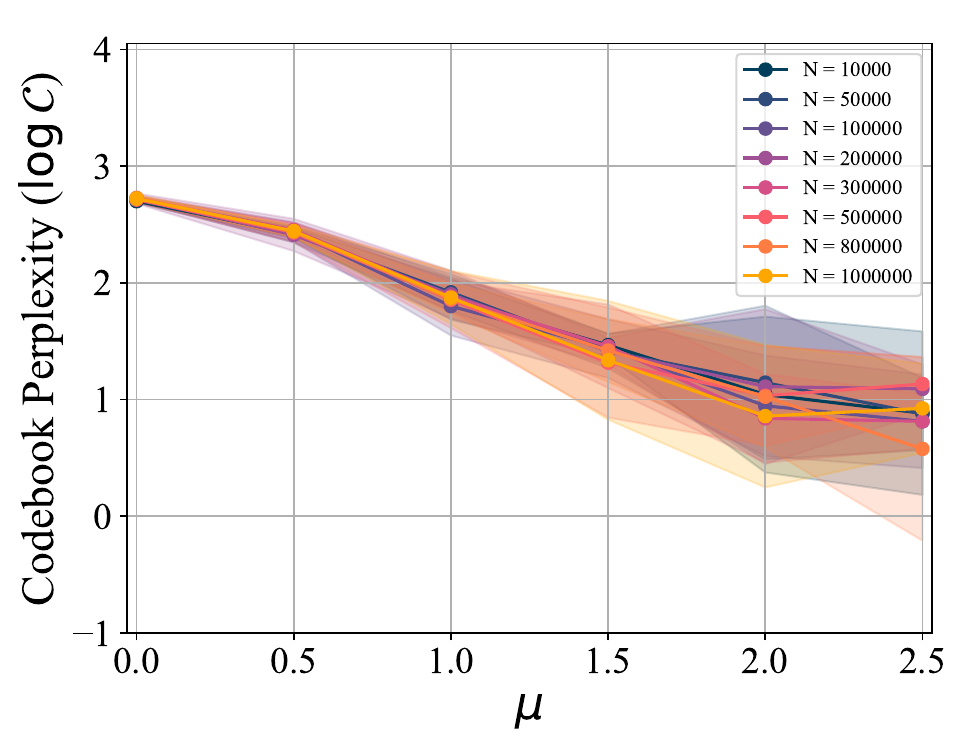}
    }
    \hspace{-0.39cm}
    \subfloat[$\mathcal{C}$ {w.r.t.} $K$]{ 
        \label{fig:gaussian_codebooksize_sigma_perplexity}  
        \includegraphics[width=0.16\textwidth]{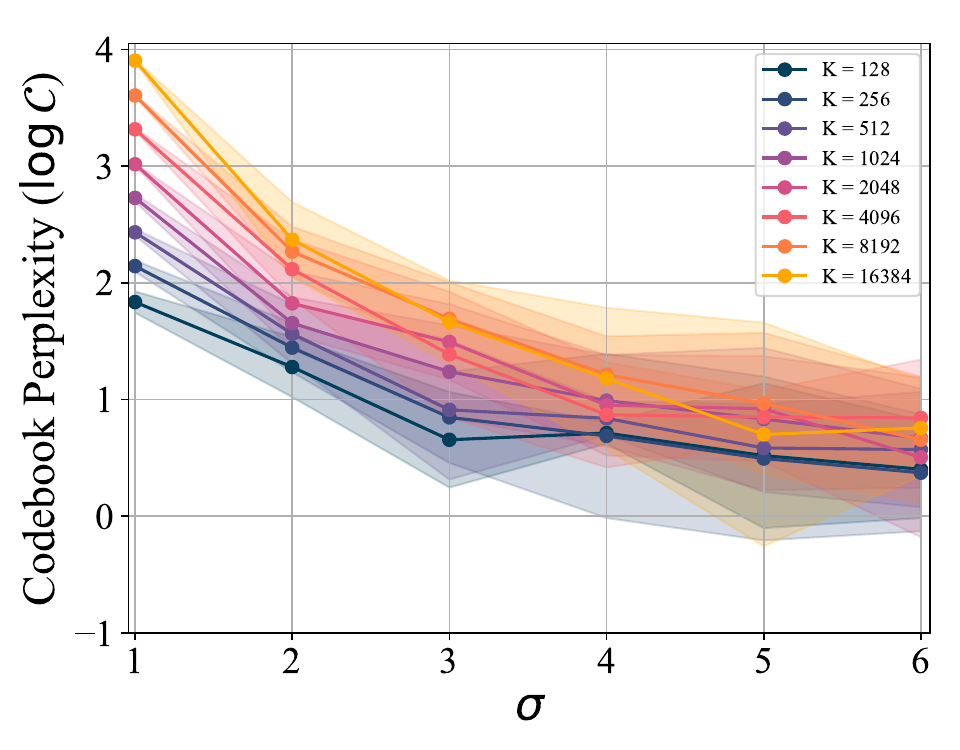}
    }
    \hspace{-0.39cm}
    \subfloat[$\mathcal{C}$ {w.r.t.} $d$]{
        \label{fig:gaussian_featuredim_sigma_perplexity}
        \includegraphics[width=0.16\textwidth]{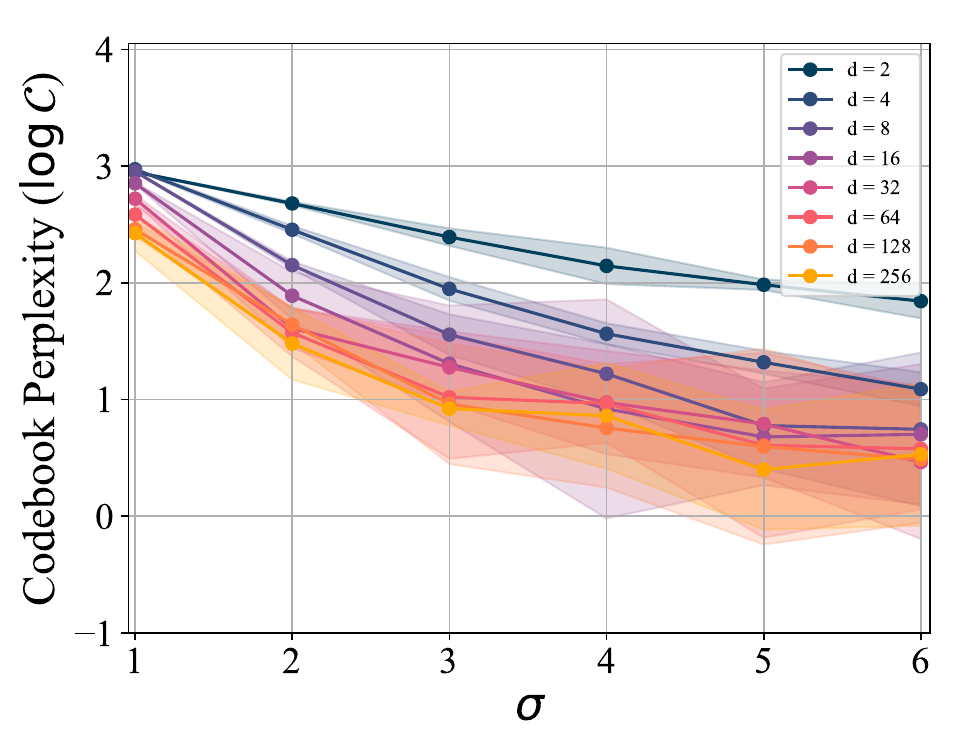}
    }
    \hspace{-0.39cm}
    \subfloat[$\mathcal{C}$ {w.r.t.} $N$]{
        \label{fig:gaussian_featuresize_sigma_perplexity}
        \includegraphics[width=0.16\textwidth]{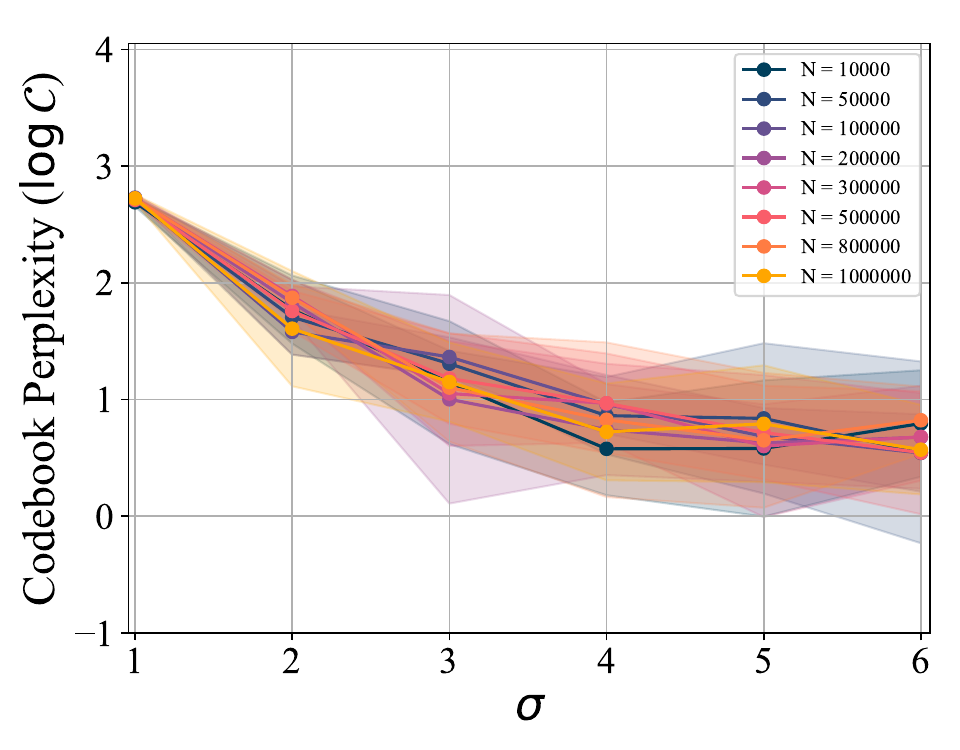}
    }
    \vspace{-1ex}
    \caption{\small{Quantitative analyses of the criterion triple when $\mathcal{P}_A$ and $\mathcal{P}_B$ are Gaussian distributions.}}
    \label{fig:quantitative analysis gaussian distribution}
    \vspace{-4ex}
\end{figure*}

\begin{figure*}[!t]
    \vspace{-1ex}
	\centering

	\subfloat[$\cE$ {w.r.t.} $K$]{ 
		\label{fig:uniform_codebooksize_mean_error}  
		\includegraphics[width=0.16\textwidth]{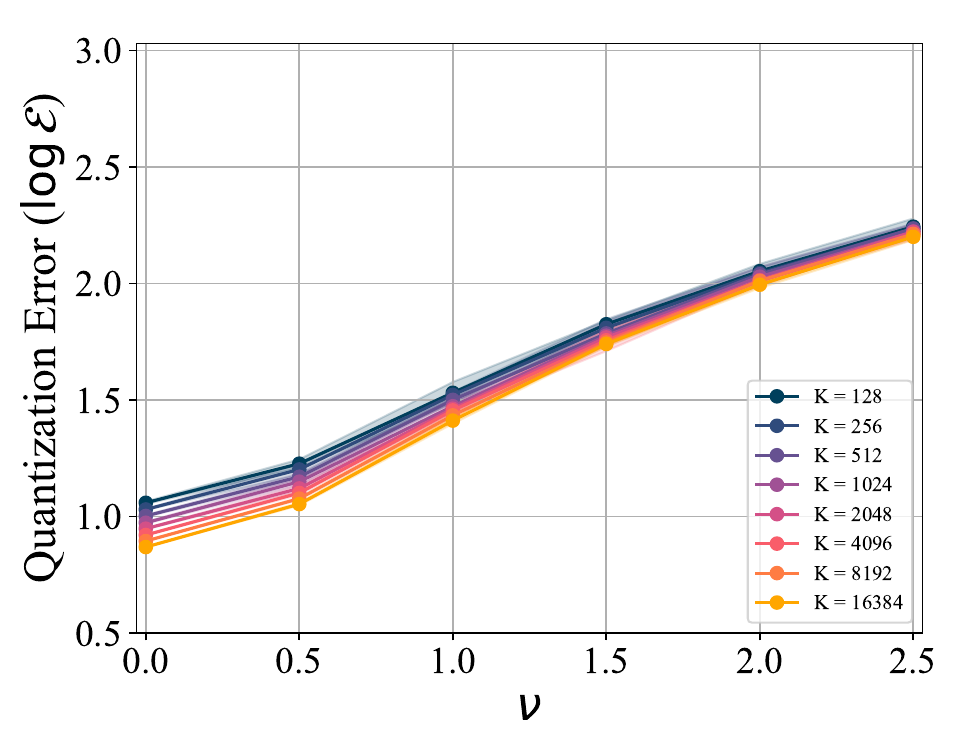}
    }
    \hspace{-0.39cm}
	\subfloat[$\cE$ {w.r.t.} $d$]{
		\label{fig:uniform_featuredim_mean_error}
		\includegraphics[width=0.16\textwidth]{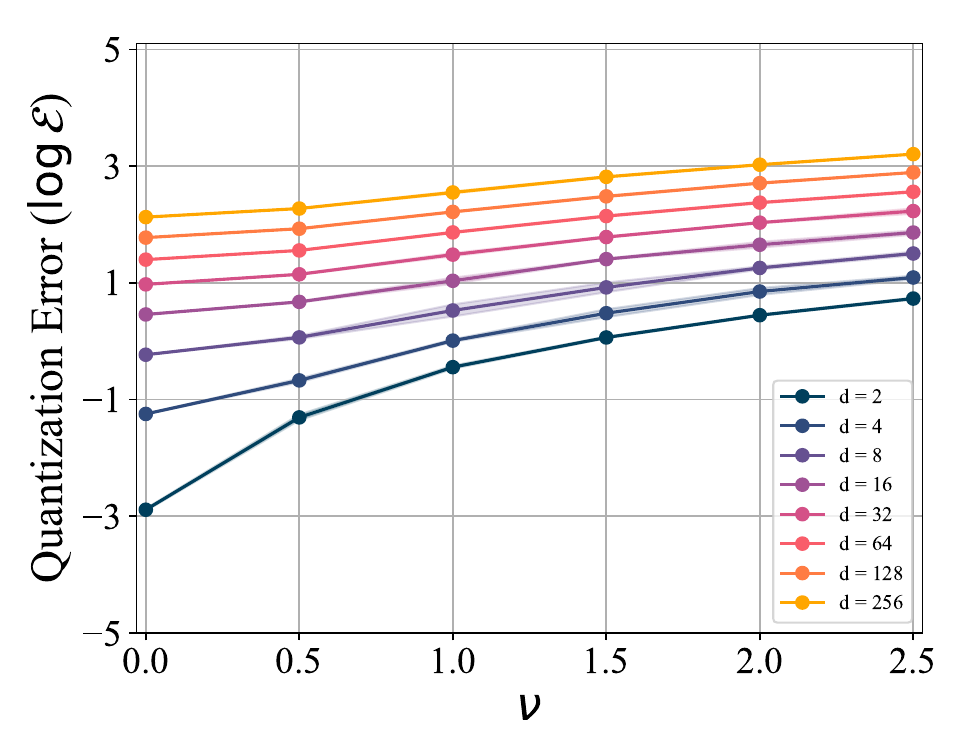}
	}
    \hspace{-0.39cm}
	\subfloat[$\cE$ {w.r.t.} $N$]{
		\label{fig:uniform_featuresize_mean_error}
		\includegraphics[width=0.16\textwidth]{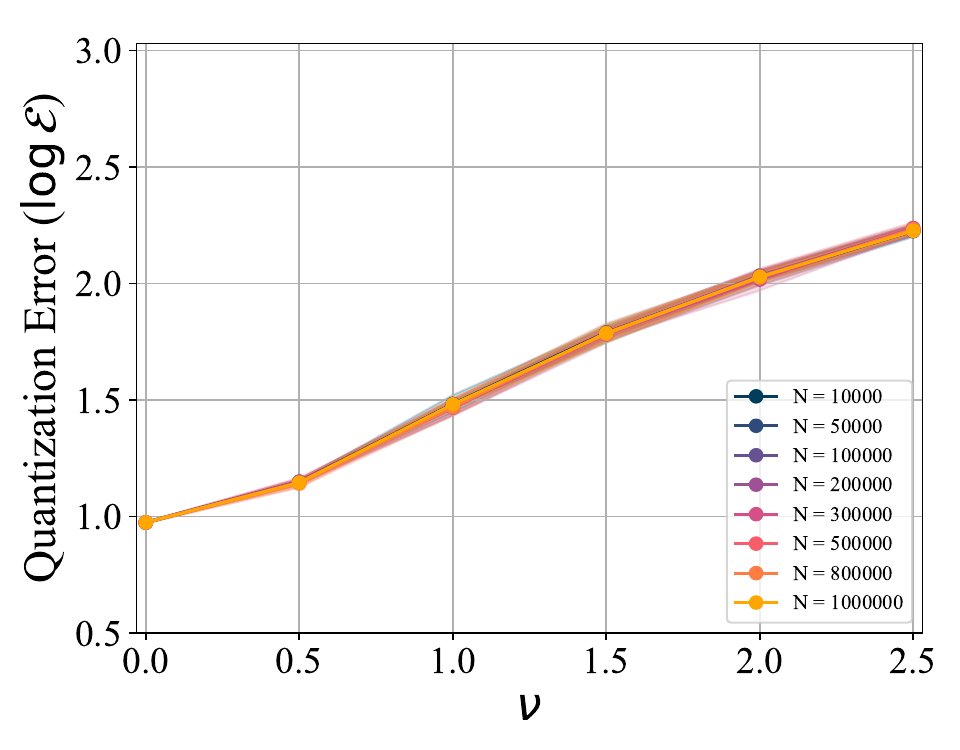}
	}
    \hspace{-0.39cm}
	\subfloat[$\cE$ {w.r.t.} $K$]{ 
		\label{fig:uniform_codebooksize_sigma_error}  
		\includegraphics[width=0.16\textwidth]{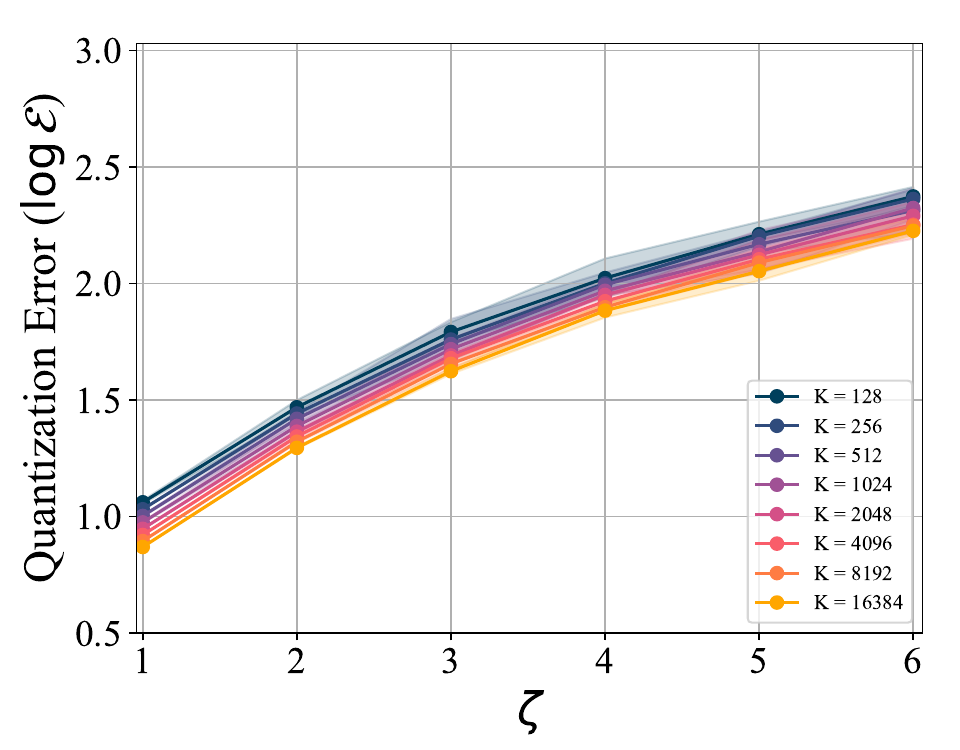}
	}
    \hspace{-0.39cm}
	\subfloat[$\cE$ {w.r.t.} $d$]{
		\label{fig:uniform_featuredim_sigma_error}
		\includegraphics[width=0.16\textwidth]{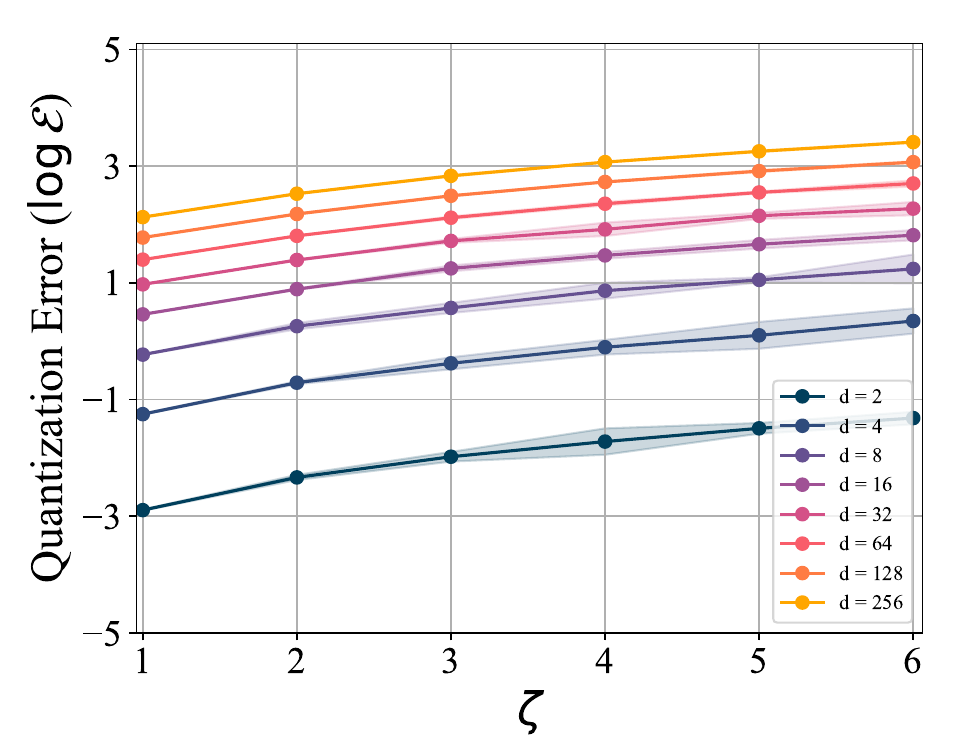}
	}
    \hspace{-0.39cm}
	\subfloat[$\cE$ {w.r.t.} $N$]{
		\label{fig:uniform_featuresize_sigma_error}
		\includegraphics[width=0.16\textwidth]{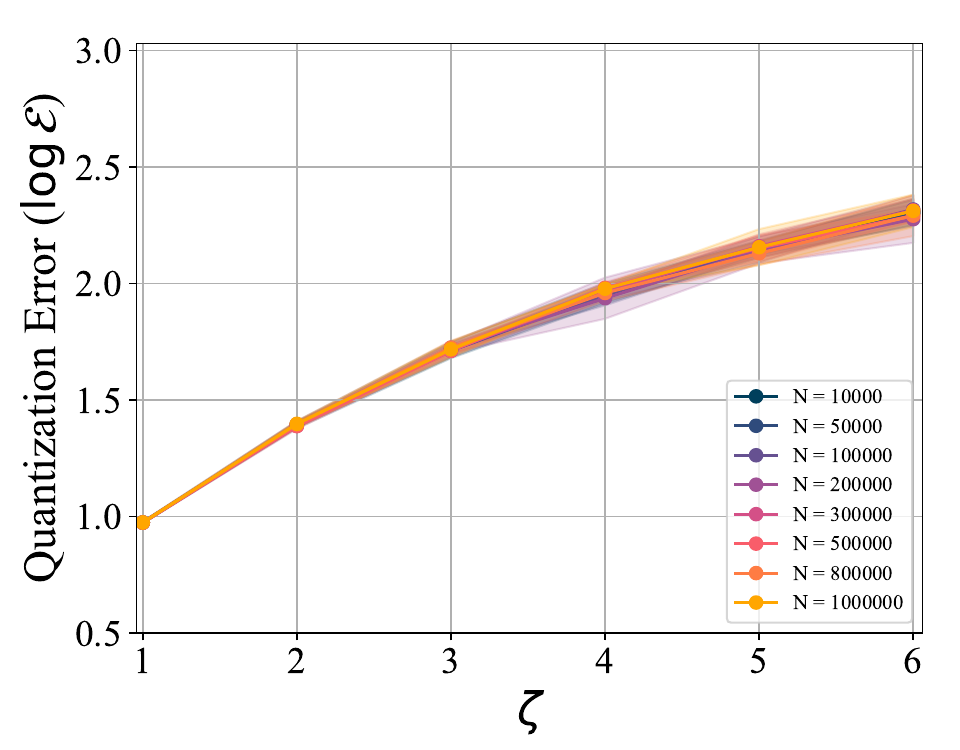}
	}


	\subfloat[$\cU$ {w.r.t.} $K$]{ 
        \label{fig:uniform_codebooksize_mean_utilization}  
		\includegraphics[width=0.16\textwidth]{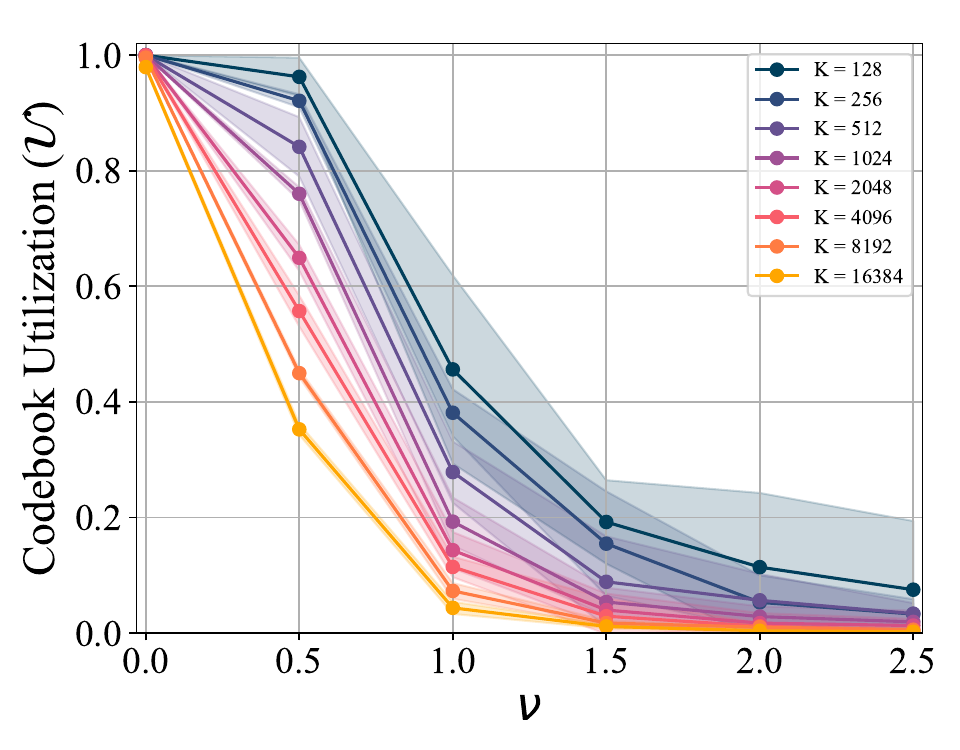}
    }
	\hspace{-0.39cm}
	\subfloat[$\cU$ {w.r.t.} $d$]{
        \label{fig:uniform_featuredim_mean_utilization}
		\includegraphics[width=0.16\textwidth]{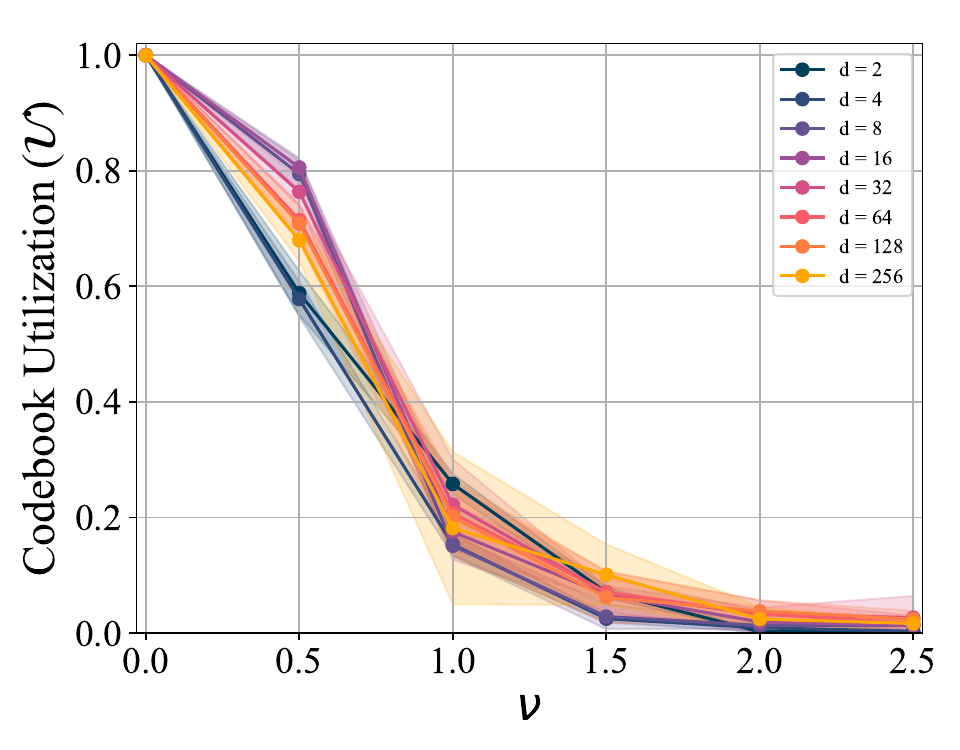}
	}
	\hspace{-0.39cm}
	\subfloat[$\cU$ {w.r.t.} $N$]{
        \label{fig:uniform_featuresize_mean_utilization}
		\includegraphics[width=0.16\textwidth]{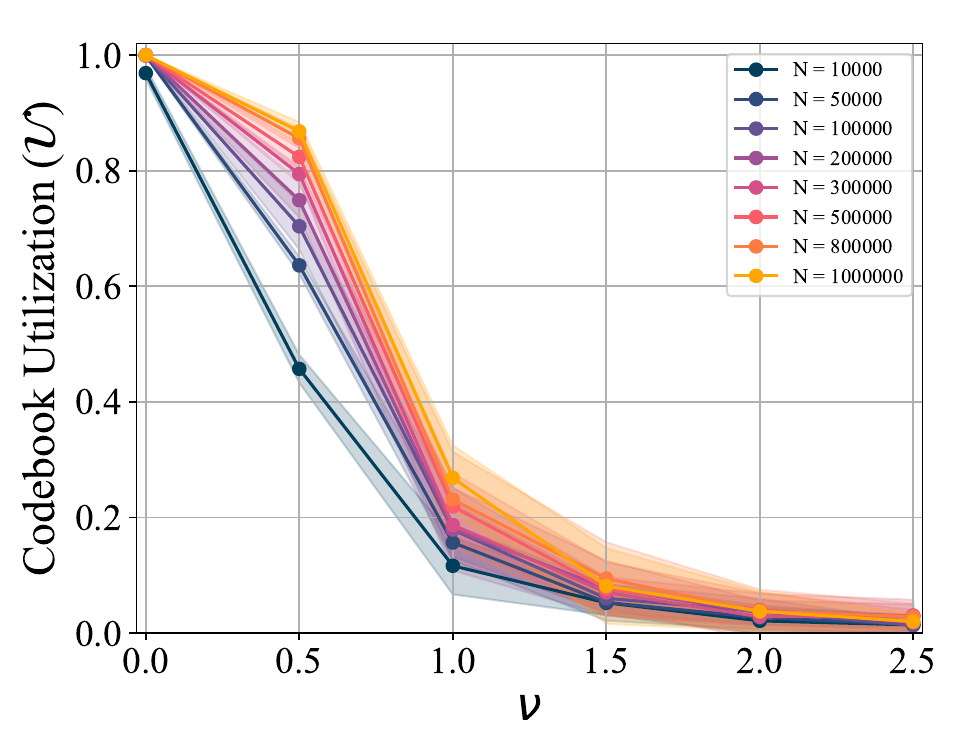}
	}
    \hspace{-0.39cm}
	\subfloat[$\cU$ {w.r.t.} $K$]{ 
        \label{fig:uniform_codebooksize_sigma_utilization}  
		\includegraphics[width=0.16\textwidth]{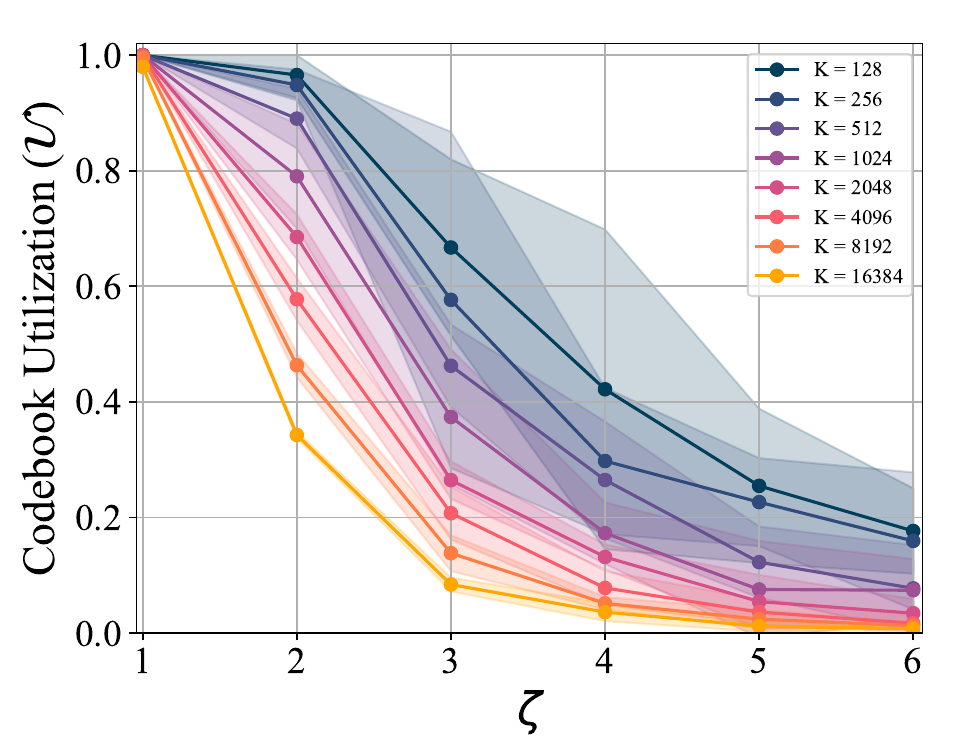}
	}
    \hspace{-0.39cm}
	\subfloat[$\cU$ {w.r.t.} $d$]{
        \label{fig:uniform_featuredim_sigma_utilization}
		\includegraphics[width=0.16\textwidth]{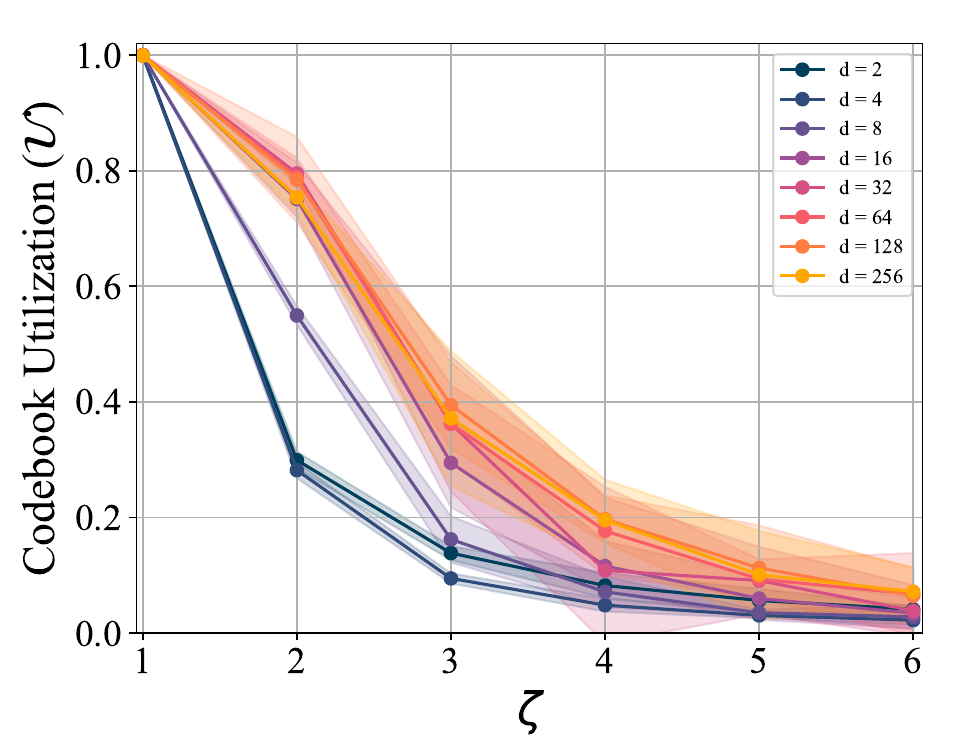}
	}
    \hspace{-0.39cm}
	\subfloat[$\cU$ {w.r.t.} $N$]{
        \label{fig:uniform_featuresize_sigma_utilization}
		\includegraphics[width=0.16\textwidth]{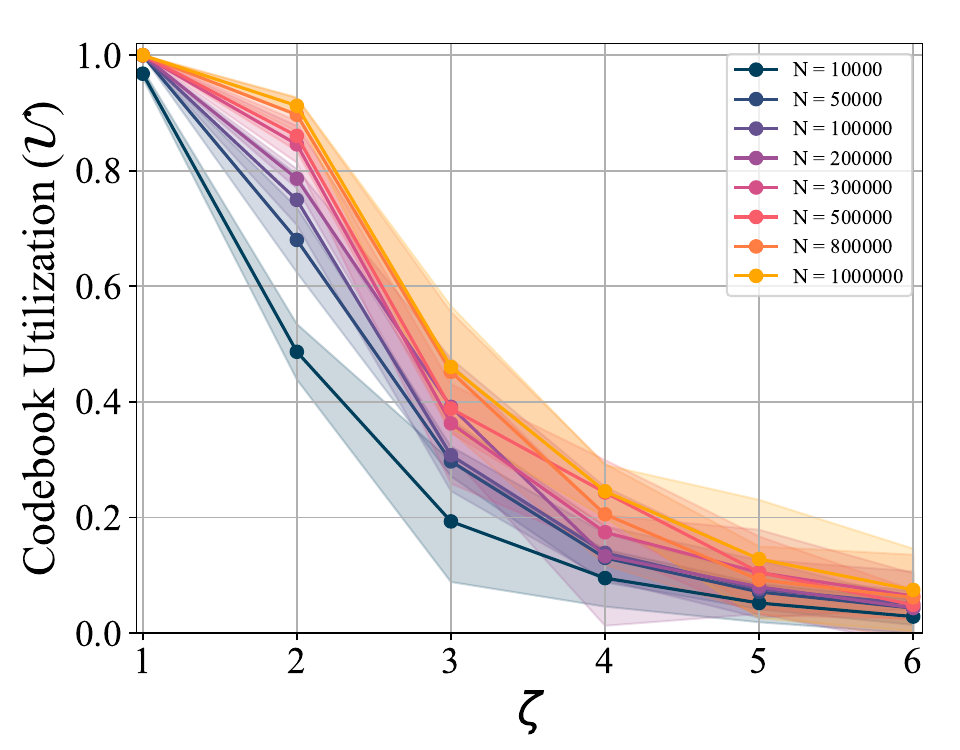}
	}


	\subfloat[$\mathcal{C}$ {w.r.t.} $K$]{ 
        \label{fig:uniform_codebooksize_mean_perplexity}  
		\includegraphics[width=0.16\textwidth]{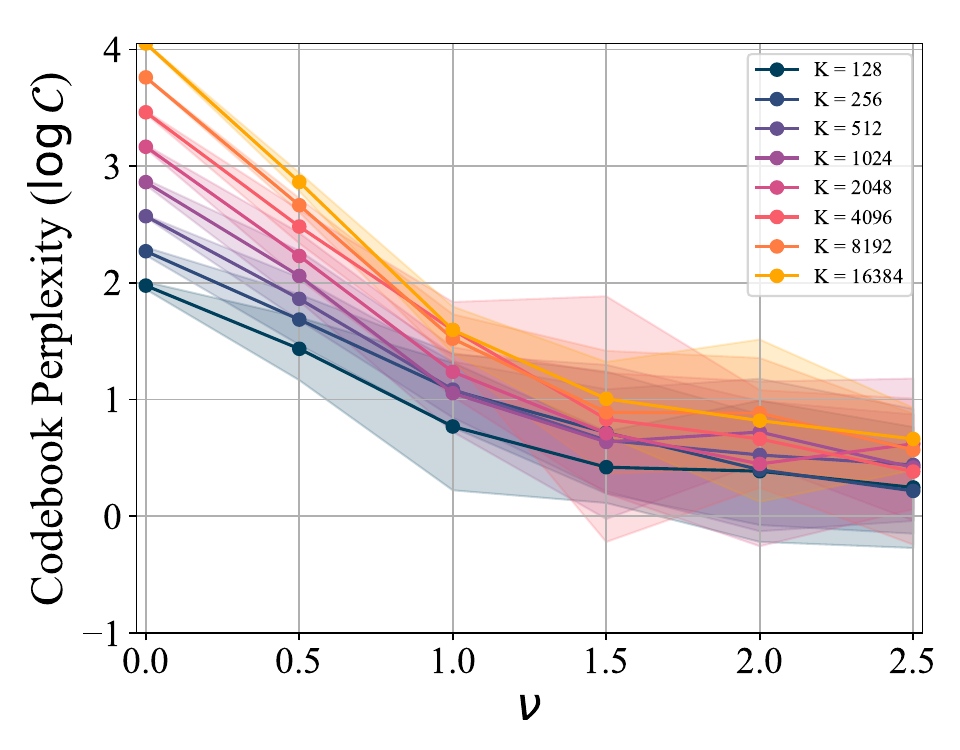}
    }
	\hspace{-0.39cm}
	\subfloat[$\mathcal{C}$ {w.r.t.} $d$]{
        \label{fig:uniform_featuredim_mean_perplexity}
		\includegraphics[width=0.16\textwidth]{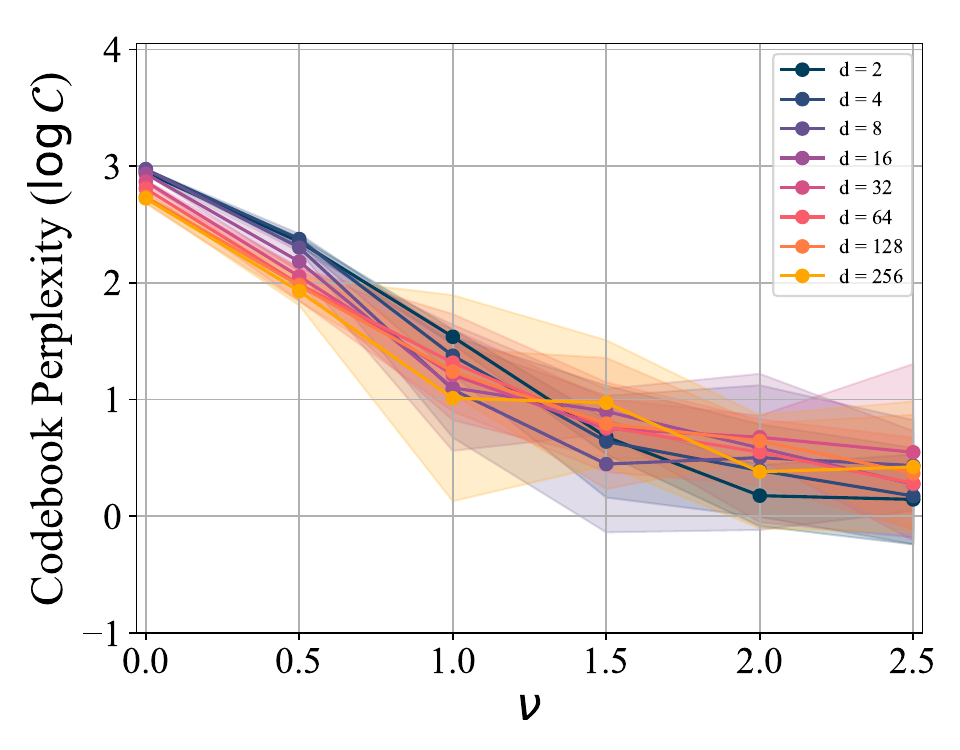}
	}
	\hspace{-0.39cm}
	\subfloat[$\mathcal{C}$ {w.r.t.} $N$]{
        \label{fig:uniform_featuresize_mean_perplexity}
		\includegraphics[width=0.16\textwidth]{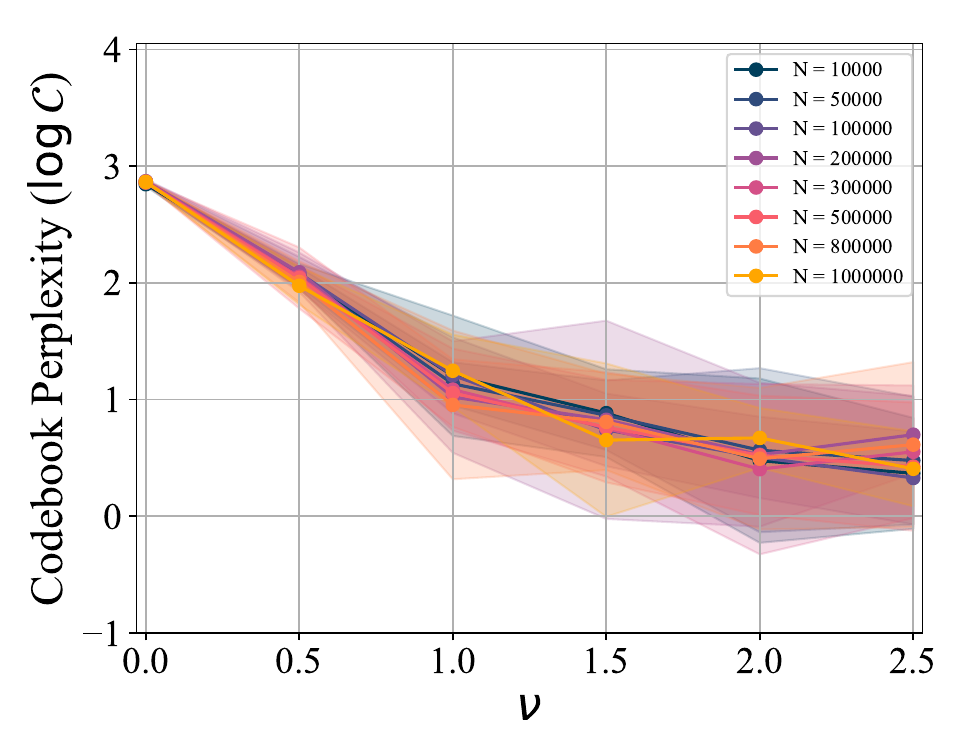}
	}
    \hspace{-0.39cm}
	\subfloat[$\mathcal{C}$ {w.r.t.} $K$]{ 
        \label{fig:uniform_codebooksize_sigma_perplexity}  
		\includegraphics[width=0.16\textwidth]{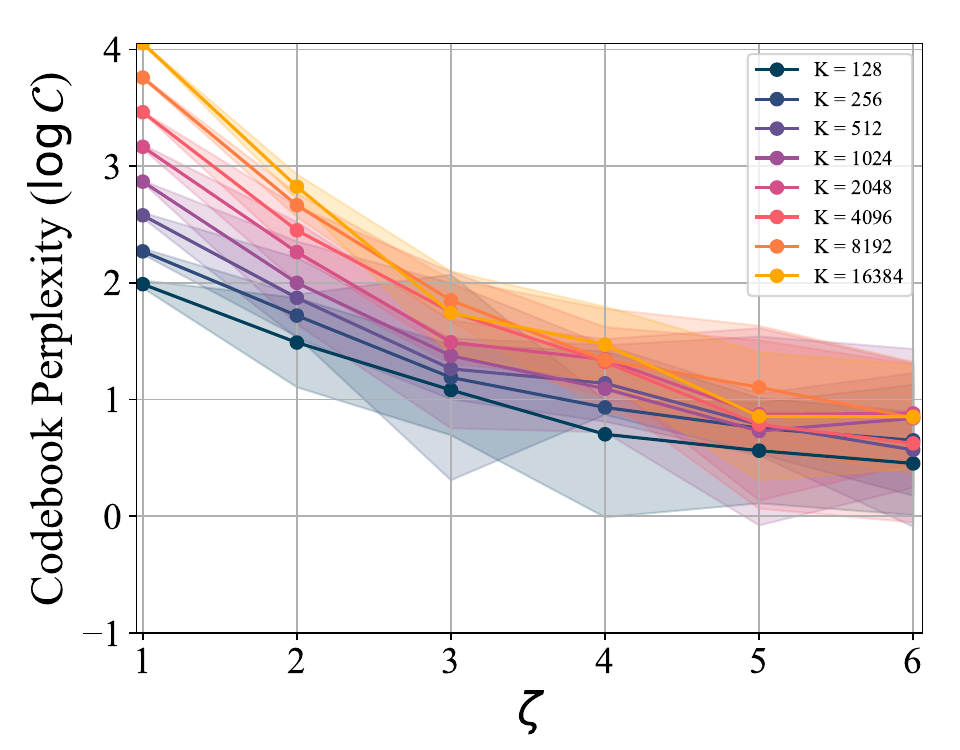}
	}
    \hspace{-0.39cm}
	\subfloat[$\mathcal{C}$ {w.r.t.} $d$]{
        \label{fig:uniform_featuredim_sigma_perplexity}
		\includegraphics[width=0.16\textwidth]{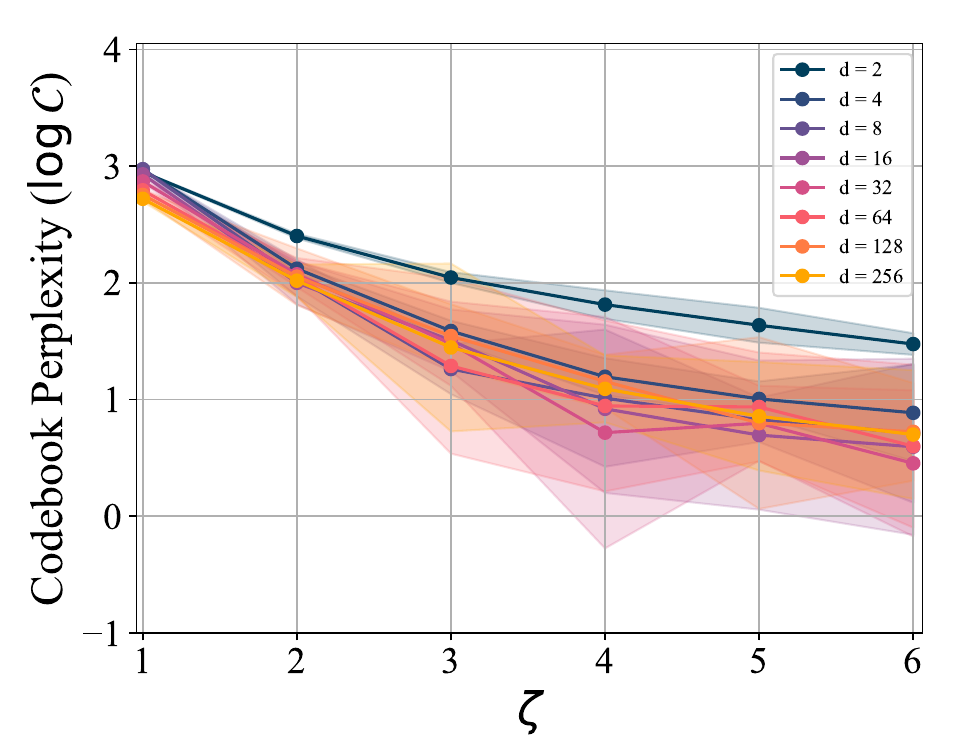}
	}
    \hspace{-0.39cm}
	\subfloat[$\mathcal{C}$ {w.r.t.} $N$]{
        \label{fig:uniform_featuresize_sigma_perplexity}
		\includegraphics[width=0.16\textwidth]{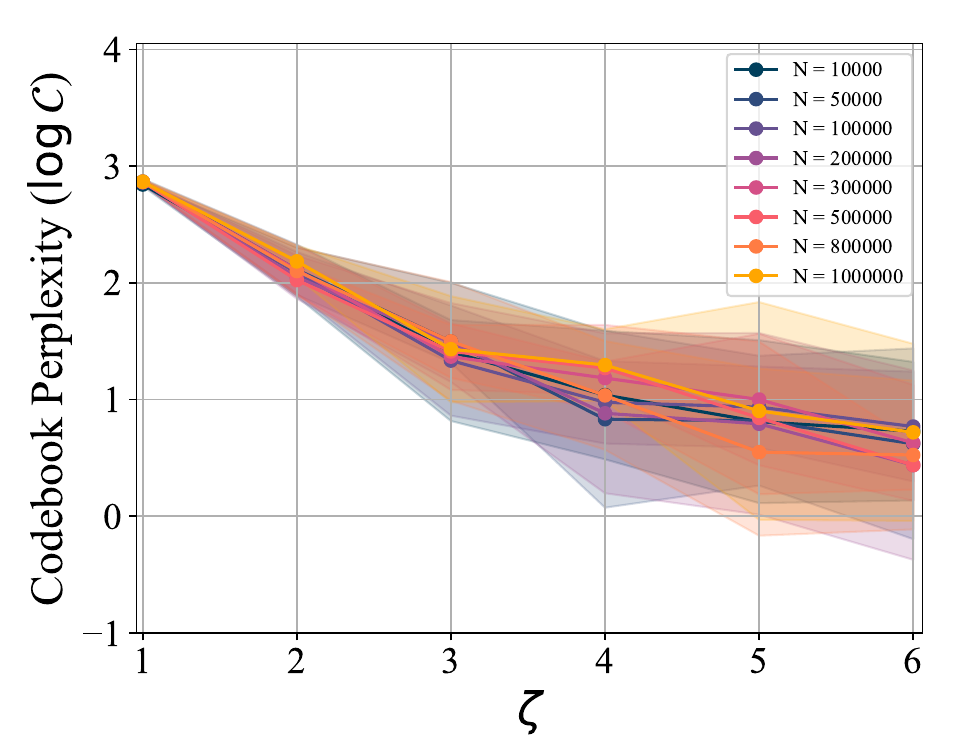}
	}

    \vspace{-1ex}
	\caption{Quantitative analyses of the criterion triple when $\mathcal{P}_A$ and $\mathcal{P}_B$ are uniform distributions.}
    \vspace{-2ex}
	\label{fig:quantitative analysis uniform distribution}
\end{figure*}

We begin by assuming that the distributions $\mathcal{P}_A$ and $\mathcal{P}_B$ are Gaussian. We generate a set of feature vectors $\{\bz_i\}_{i=1}^N$ from $\mathcal{N}_d(\bm 0,  \bm I)$ and a set of code vectors $\{\be_k\}_{k=1}^K$ from $\mathcal{N}_d(\mu \cdot \m 1, \bm I)$, with $\mu$ varying within $\{0.0, 0.5, 1.0, 1.5, 2.0, 2.5\}$. The criterion triple results are presented in Figures~\ref{fig:gaussian_codebooksize_mean_error} to~\ref{fig:gaussian_featuresize_mean_error}, Figures~\ref{fig:gaussian_codebooksize_mean_utilization} to~\ref{fig:gaussian_featuresize_mean_utilization}, and Figures~\ref{fig:gaussian_codebooksize_mean_perplexity} to~\ref{fig:gaussian_featuresize_mean_perplexity}. Across all tested configurations of $K, d, N$, we consistently observe that when $\mu = 0$ — indicating identical distributions between $\mathcal{P}_A$ and $\mathcal{P}_B$ — the criterion triple achieves the lowest $\cE$, highest $\cU$, and largest $\mathcal{C}$. This empirical evidence reinforces the effectiveness of aligning feature and codebook distributions in VQ.

Additionally, we further analyze the criterion triple by varying the covariance matrix. We sample a set of feature vectors $\{\bz_i\}_{i=1}^N$ from the distribution $\mathcal{N}_d(\bm 0, \bm I)$ and a set of code vectors $\{\be_k\}_{k=1}^K$ from $\mathcal{N}_d(\bm 0, \sigma^2 \bm I)$, where $\sigma$  is selected from $\{ 1, 2, 3, 4, 5, 6\}$. The results for the criterion triple are shown in Figures~\ref{fig:gaussian_codebooksize_sigma_error} to~\ref{fig:gaussian_featuresize_sigma_error}, Figures~\ref{fig:gaussian_codebooksize_sigma_utilization} to~\ref{fig:gaussian_featuresize_sigma_utilization}, and Figures~\ref{fig:gaussian_codebooksize_sigma_perplexity} to~\ref{fig:gaussian_featuresize_sigma_perplexity}. When $\sigma = 1$, indicating identical distributions between $\mathcal{P}_A$ and $\mathcal{P}_B$, all three evaluation criteria reach their optimal values: the lowest $\cE$, highest $\cU$, and largest $\mathcal{C}$ across all tested values of $K, d, N$. This result corroborates our earlier findings.

\subsection{Codebook Distribution and Feature Distribution are Unifrom Distributions}
\label{appendix:analyses under uniform distribution}

The above conclusion holds when $\mathcal{P}_A$ and $\mathcal{P}_B$ are other types of distributions, such as the uniform distribution. As shown in Figure~\ref{fig:quantitative analysis uniform distribution}, we sample a set of feature vectors $\{\bz_i\}_{i=1}^N$ from the distribution $\unif_d(-1, 1)$ and a set of code vectors $\{\be_k\}_{k=1}^K$ from $\unif_d(\nu-1, \nu+1)$, where $\nu$ is selected from the set  $\{ {0.0}, {0.5}, {1.0}, {1.5}, {2.0}, {2.5}\}$ or from $\unif_d(-\zeta, \zeta)$, with $\zeta$ drawn from the set $\{ 1, 2, 3, 4, 5, 6\}$. We observe that when $\mu = 0$ or $\zeta = 1$—indicating that $\mathcal{P}_A$  and $\mathcal{P}_B$ have identical distributions—the performance in terms of the criterion triple is optimal, achieving the lowerest $\cE$, the highest $\cU$, and the largest $\mathcal{C}$ across all tested values of $K, d, N$. Therefore, we conclude that our quantitative analyses are distribution-agnostic and can be generalized to other distributions.

\section{Statistical Distances over Gaussian Distributions}
\label{appendix: distribution distance}
We first introduce the definition of Wasserstein distance. 
\begin{definition}
\label{def:wasserstien distance}
The Wasserstein distance or earth-mover distance with $p$ norm is defined as below:
\begin{equation}
W_{p}(\mathbb{P}_r, \mathbb{P}_g) = (\inf_{\gamma \in \Pi(\mathbb{P}_r ,\mathbb{P}_g)} \mathbb{E}_{(x, y) \sim \gamma}\big[\|x - y\|^{p}\big])^{1/p}~.
\label{eq:def wasserstein distance}
\end{equation}
\end{definition} where $\Pi(\mathcal{P}_r,\mathcal{P}_g)$ denotes the set of all joint distributions $\gamma(x,y)$ whose marginals are $\mathcal{P}_r$ and $\mathcal{P}_g$ respectively. Intuitively, when viewing each distribution as a unit amount of earth/soil, the Wasserstein distance (also known as earth-mover distance) represents the minimum cost of transporting ``mass'' from $x$ to $y$ to transform distribution $\mathcal{P}_r$ into distribution $\mathcal{P}_g$. When $p=2$, this is called the quadratic Wasserstein distance.

In this paper, we achieve distributional matching using the quadratic Wasserstein distance under Gaussian distribution assumptions. We also examine other statistical distribution distances as potential loss functions for distributional matching and compare them with the Wasserstein distance. Specifically, we provide the Kullback-Leibler divergence and the Bhattacharyya distance over Gaussian distributions in Lemma~\ref{theorem:kl distance} and Lemma~\ref{theorem:bd distance}. It can be observed that the KL divergence for two Gaussian distributions involves calculating the determinant of covariance matrices, which is computationally expensive in moderate and high dimensions. Moreover, the calculation of the determinant is sensitive to perturbations and it requires full rank (In the case of not full rank, the determinant is zero, rendering the logarithm of zero undefined), which can be impractical in many cases. Other statistical distances like Bhattacharyya Distance suffer from the same issue. In contrast, quadratic Wasserstein distance does not require the calculation of the determinant and full-rank covariance matrices.

\begin{lemma}[Kullback-Leibler divergence~\cite{Lindley1959InformationTA}] \label{theorem:kl distance} 
Suppose two random variables $\m Z_1 \sim \mathcal{N}(\bm \mu_1, \m \Sigma_1)$ and $\m Z_2 \sim \mathcal{N}(\bm \mu_2, \m \Sigma_2)$ obey multivariate normal distributions, then Kullback-Leibler divergence between $\m Z1$ and $\m Z_2$ is:
\begin{align}
\KL(\m Z_1, \m Z_2) = \frac{1}{2}((\bm \mu_1-\bm \mu_2)^T \m \Sigma_2^{-1}(\bm \mu_1-\bm \mu_2) + \tr(\m \Sigma_2^{-1}\m \Sigma_1-\m I) + \ln{\frac{\det{\m \Sigma_2}}{\det \m \Sigma_1}}). \nonumber
\end{align}
\end{lemma}

\begin{lemma}
[Bhattacharyya Distance~\cite{Bhattacharyya1943OnAM}]
\label{theorem:bd distance}
Suppose two random variables $\m Z_1 \sim \mathcal{N}(\bm \mu_1, \m \Sigma_1)$ and $\m Z_2 \sim \mathcal{N}(\bm \mu_2, \m \Sigma_2)$ obey multivariate normal distributions, $\m \Sigma = \frac{1}{2}(\m \Sigma_1 + \m \Sigma_2)$, then bhattacharyya distance between $\m Z1$ and $\m Z_2$ is:
\begin{align}
\mathcal{D}_{\textrm{B}}(\m Z_1, \m Z_2) = \frac{1}{8}(\bm \mu_1-\bm \mu_2)^T \m \Sigma^{-1}(\bm \mu_1-\bm \mu_2) + \frac{1}{2}\ln \frac{\det \m \Sigma}{\sqrt{\det \m \Sigma_1 \det \m \Sigma_2}}. \nonumber
\end{align}
\end{lemma}

\begin{figure}[t]
    \centering
    \vspace{-2ex}
    \includegraphics[width=\textwidth]{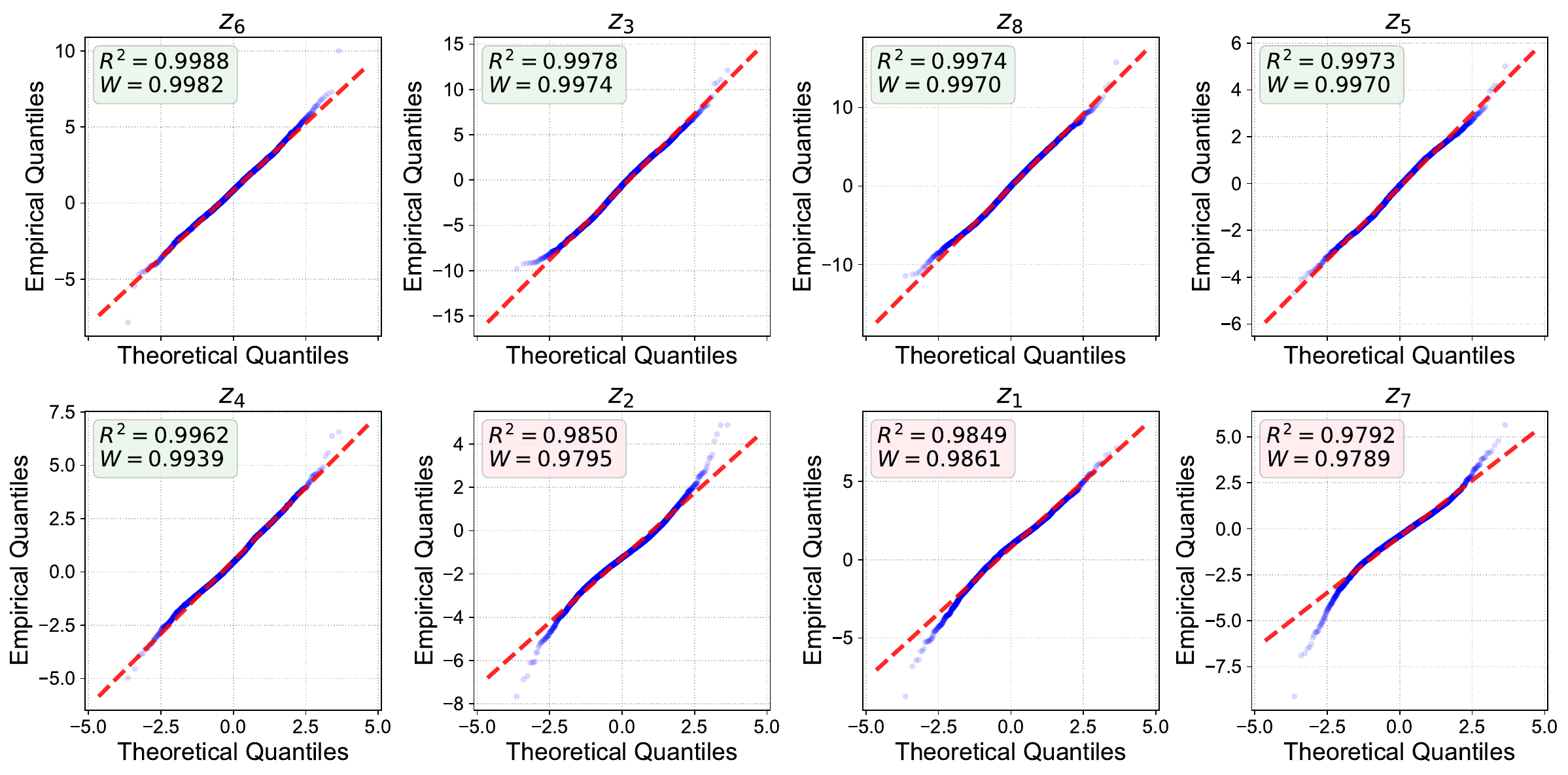}
    \vspace{-3ex}
    \caption{Comprehensive Q--Q plots for all latent dimensions. Subplots are ordered from the best-fitting dimension $z_6$ (top-left) to the least-fitting dimension $z_7$ (bottom-right) according to the $R^2$ score. \textbf{Visualization:} For visual clarity, blue points show a random subset of 5,000 samples. \textbf{Statistics:} The annotated $R^2$ and Shapiro--Wilk $W$ statistics (green/red boxes) are computed on the \textbf{full validation set ($N=409{,}600$)}. Across all dimensions, the Q--Q plots indicate a strong Gaussian fit for the central mass of the distributions, with only mild deviations in the extreme tails.}
    \label{fig:qq_full}
    \vspace{-2ex}
\end{figure}

\begin{figure}[!t]
    \centering
    \includegraphics[width=0.95\textwidth]{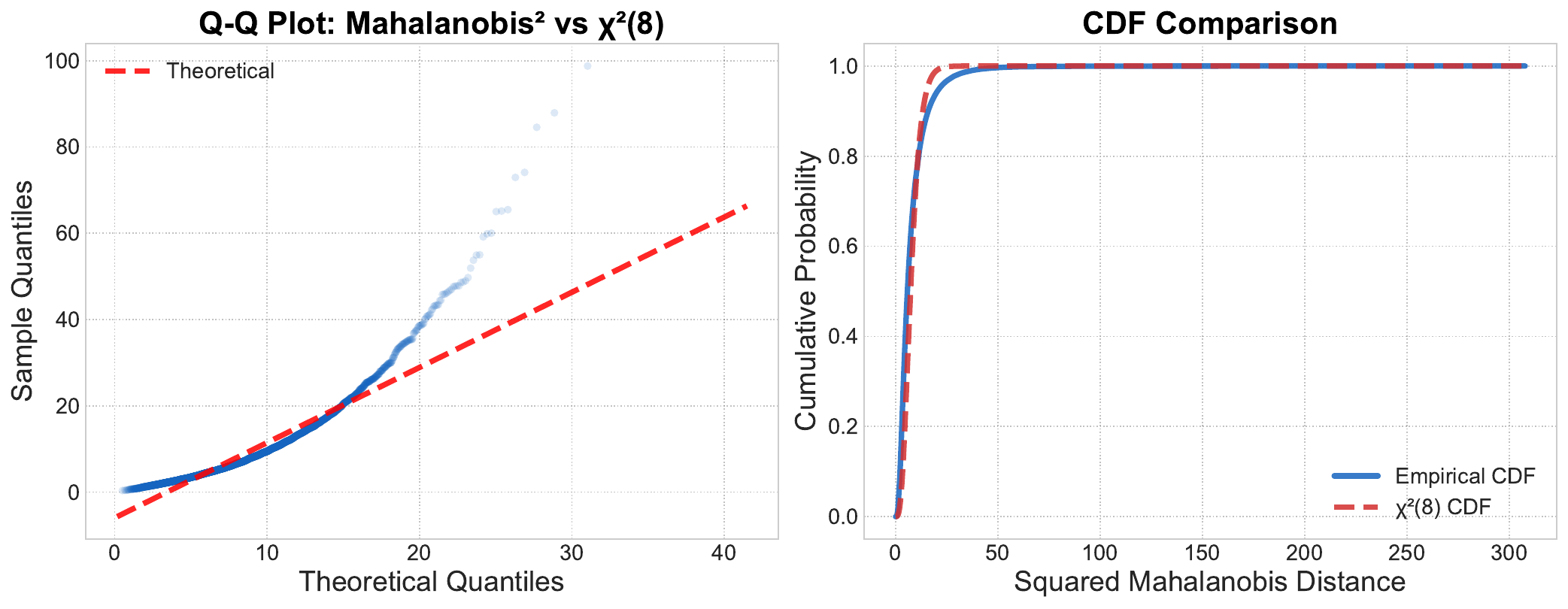}
    \vspace{-1ex}
    \caption{Multivariate normality assessment using Mahalanobis distance. Left (Q--Q plot): The squared Mahalanobis distances of the latent features are compared against the theoretical $\chi^2_8$ quantiles. Deviations are primarily observed in the extreme tails, while the bulk of the samples closely follows the reference line. Right (CDF): The empirical cumulative distribution function (blue solid) closely matches the theoretical $\chi^2_8$ CDF (red dashed) over the primary probability mass, indicating strong agreement in the high-density region.}
    \label{fig:mahalanobis}
    \vspace{-3ex}
\end{figure}

\section{Statistical Assessment of Gaussianity}
\label{sec:normality_tests}
To further validate the Gaussian assumption discussed in Section~\ref{sec:vqvae experiments}, we conduct a comprehensive set of univariate and multivariate normality tests on the validation set. Specifically, we analyze $N=409{,}600$ latent feature vectors ($d=8$), extracted from 1,600 images in the FFHQ dataset.

\textbf{Univariate Analysis.} 
Figure~\ref{fig:qq_full} shows the Q--Q plots for all eight dimensions of the latent feature vector, sorted by their $R^2$ goodness-of-fit scores. Consistent with the observations in the main paper, the majority of dimensions ($z_6, z_3, z_8, z_5, z_4$) demonstrate an excellent fit to the Gaussian distribution, with $R^2 > 0.99$. Although the remaining dimensions exhibit mild deviations in the extreme tails, their Shapiro--Wilk statistics remain high ($W > 0.97$), indicating that the central mass of the distributions is well approximated by a Gaussian.

\textbf{Multivariate Analysis.} 
To assess joint normality, we compute the squared Mahalanobis distance for each validation sample,
\[
d_i^2 = (z_i - \mu)^\top \Sigma^{-1} (z_i - \mu).
\]
Under the multivariate Gaussian hypothesis, $d_i^2$ should follow a Chi-squared distribution with $d=8$ degrees of freedom ($\chi^2_8$). The results are shown in Figure~\ref{fig:mahalanobis}. The Q--Q plot (Left) indicates strong agreement between the empirical and theoretical quantiles for the majority of samples, particularly in the high-density region near the origin, while heavier tails emerge only in the extreme quantiles. To further quantify this alignment, Figure~\ref{fig:mahalanobis} (Right) compares the cumulative distribution functions (CDFs). The empirical CDF (blue solid line) closely matches the theoretical $\chi^2_8$ CDF (red dashed line) across the primary probability mass.

Overall, while deep feature representations naturally exhibit heavy-tailed behavior due to outliers, both univariate and multivariate tests consistently indicate that the bulk of the distribution is well modeled by a Gaussian. This provides strong empirical justification for using the Wasserstein distance, which is derived under Gaussian assumptions, as an efficient and effective optimization surrogate. Moreover, the comparable performance of MMD VQ, which does not impose any parametric distributional assumptions, further suggests that the Gaussian approximation captures the essential statistical structure required for high-fidelity tokenization.

\begin{table*}[!h]
\centering 
\small
\vspace{-1ex}
\caption{Hyperparameters for the experiments in Table~\ref{tab:vqvae ffhq},~\ref{tab:vqvae imagenet},~\ref{tab:reconstruction_IN-main table}, ~\ref{tab:vqvae cifar10},~\ref{tab:vqvae svhn},~\ref{tab:vqgan tranplant ffhq}.}
\label{table:vqvae_hparams}
\vspace{-1ex}
\begin{tabular}{@{}lcccc@{}}
\toprule
Frameworks & VQ-VAE & VQ-VAE & VQGAN  & VQGAN\\
\midrule
\multicolumn{1}{l|}{Dataset} & CIFAR-10/SVHN & FFHQ/ImageNet & FFHQ & ImageNet \\
\multicolumn{1}{l|}{Input size} & $32\times 32\times 3$ & $256\times 256\times 3$ & $256\times 256\times 3$ & $256\times 256\times 3$\\
\multicolumn{1}{l|}{Latent size} & $8\times 8 \times 8$ & $16\times 16 \times 8$ & $16\times 16 \times 32$ & $16\times 16 \times 32$ \\
\multicolumn{1}{l|}{encoder/decoder channels} & $64$ & $64$ & $160$ & $160$\\
\multicolumn{1}{l|}{encoder/decoder channel mult.} & $[1, 1, 2]$ & $[1, 1, 2, 2, 4]$ & $[1, 1, 2, 2, 4]$ & $[1, 1, 2, 2, 4]$ \\
\multicolumn{1}{l|}{Batch size} & 128 & 32 & 32 & 32\\
\multicolumn{1}{l|}{Initial Learning rate $lr$} & $5\times 10^{-5}$ & $5\times 10^{-5}$ & $1\times 10^{-5}$ & $1\times 10^{-5}$\\
\multicolumn{1}{l|}{Perceptual loss Coefficient} & $0$ & $0$ & $1.0$ & $1.0$\\
\multicolumn{1}{l|}{Adversarial loss Coefficient} & $0$ & $0$ & $0.4$ & $0.4$\\
\multicolumn{1}{l|}{Codebook dimensions} & $8$ & $8$ & $32$  & $32$\\
\multicolumn{1}{l|}{Training Epochs} & $50$ & $30/4$ & $30$ & $5/10/15$ \\
\multicolumn{1}{l|}{GPU Resources} & 1 V100 16GB &  1 A100 40GB &  2 H100 80GB & 2 H100 80GB \\
\bottomrule
\end{tabular}
\vspace{-3ex}
\end{table*}

\section{The Experimental Details} 
\label{appendix:synthetic experiment details}

\subsection{Synthetic Experimental Details in Section~\ref{sec:effects of distribution matching} } \label{appendix:details part1}

As depicted in Figure~\ref{fig:criterion triple analysis} in Section~\ref{sec:effects of distribution matching}, we conduct a qualitative analyses of the criterion triple. Specifically, we sample a set of feature vectors $\{\bz_i\}_{i=1}^N$ from within the red circle, and a collection of code vectors $\{\be_k\}_{k=1}^K$ from within the green circle, with parameters set to $K=400$, $N=10000$ and $d=2$ for the calculation of the criterion triple $(\cE, \cU, \mathcal{C})$. For the visualization, we select 10\% of the feature vectors and 90\% of the code vectors for plotting.

\subsection{Synthetic Experimental Details in Appendix~\ref{appendix:supplementary comprehensive quantitative analyses}} \label{appendix:details part3}

As illustrate in Figure~\ref{fig:quantitative analysis gaussian distribution} in Appendix~\ref{appendix:analyses under gaussian distribution}, we undertake comprehensive quantitative analyses centered around the criterion triple $(\cE, \cU, \mathcal{C})$. In these analyses, we assume that $\mathcal{P}_A$ and $\mathcal{P}_B$ are Gaussian distributions, from which we sample a set of feature vectors $\{\bz_i\}_{i=1}^N$ and a collection of code vectors $\{\be_k\}_{k=1}^K$. The default parameters are set to $N=200,000$, $K=1024$, and $d=32$ for all figures unless otherwise specified. For instance, in Figure~\ref{fig:gaussian_codebooksize_mean_error}, $N$ and $d$ are taken at their default values, while the $K$ is varied within the set $\{128, 256, 512, 1024, 2048, 4096, 8192, 16284\}$. Additionally, each synthetic experiment is repeated five times, and the average results are reported, along with the calculation of 95\% confidence intervals. In all figures, mean results are represented by points, while the confidence intervals are shown as shaded areas. Identical parameter settings are employed when $\mathcal{P}_A$ and $\mathcal{P}_B$ are uniform distributions, as illustrated in Figure~\ref{fig:quantitative analysis uniform distribution} in Appendix~\ref{appendix:analyses under uniform distribution}.

\subsection{Synthetic Experimental Details in Appendix~\ref{appendix:Impact of Distribution Variance}} \label{appendix:details part4}
We set $K=8192, d=8, N=100000$ when calculating the criterion triple $(\cE, \cU, \mathcal{C})$ in Appendix~\ref{appendix:Impact of Distribution Variance}. Each synthetic experiment is repeated five times, and the average results are reported in Table~\ref{table:distribution variance}. 

\subsection{Synthetic Experimental Details in Section~\ref{sec:atomic setting}} \label{appendix:details part2}
We provide experimental details of Figure~\ref{fig:analysis on criterion triple (VQ)} in Section~\ref{sec:atomic setting}. In our experimental setup, we evaluate five distinct VQ algorithms using the criterion triple $(\cE, \cU, \mathcal{C})$. All experiments run on a single NVIDIA A100 GPU, with a codebook size $K$ of 16,384 and dimensionality $d$ of 8 across all algorithms. Each algorithm trains for 2,000 steps, with 50,000 feature vectors sampled from the specified Gaussian distribution at each step. For \emph{Wasserstein VQ}, Vanilla VQ, and VQ + MLP, we use the SGD optimizer for training. For VQ EMA and Online Clustering, we use classical clustering algorithms—$k$-means~\cite{Bradley1998RefiningIP} and $k$-means++\cite{Arthur2007kmeansTA}—to update code vectors.

\subsection{Experimental Details in Section~\ref{sec:experiments}}
\label{appendix:experimental details}

\paragraph{Data Augmentation} For FFHQ and ImageNet-1k datasets, we follow LLama Gen~\cite{Sun2024AutoregressiveMB} and apply iterative box downsampling to resize images to 256×256 resolution. For CIFAR-10 and SVHN, the images are kept at their original resolution.

\paragraph{Encoder-Decoder Architecture In VQ-VAE} For the ImageNet and FFHQ datasets, our proposed Wasserstein VQ and all baseline methods adopt identical encoder-decoder architectures and parameter configurations. Across all baselines in these frameworks, the encoder—a U-Net~\cite{Ronneberger2015UNetCN}—downscales the input image by a factor of 16. For CIFAR-10 and SVHN datasets, the encoder reduces the input resolution by a factor of 4. Detailed hyperparameter configurations are summarized in Table~\ref{table:vqvae_hparams}.

\paragraph{Encoder-Decoder Architecture In VQGAN} To accelerate adversarial training in VQGAN, our proposed Wasserstein VQ-a/b/c and Wasserstein VAR-a/b/c models on the ImageNet dataset, as well as the Wasserstein VQ model on the FFHQ dataset, are built upon the VQ-Transplant framework. Within this framework, we deploy the pretrained VAR tokenizer~\cite{Tian2024VisualAM} for initialization. Consequently, the encoder–decoder architecture is identical to that used in the VAR tokenizer.

\paragraph{Training Details in VQ-VAE} All experiments employ identical training settings: we use the AdamW optimizer~\cite{Loshchilov2017DecoupledWD} with $\beta_1=0.9$ and $\beta_1=0.95$, an initial learning rate $lr$, and apply a half-cycle cosine decay schedule following a linear warm-up phase. For specific details on training epochs and batch sizes, refer to Table~\ref{table:vqvae_hparams}.

\paragraph{Training Details in VQ-GAN} Our proposed Wasserstein VQ and Wasserstein VAR were trained on two NVIDIA H100 GPUs using the AdamW optimizer~\cite{Loshchilov2017DecoupledWD} with $\beta_1=0.9$ and $\beta_1=0.95$.
During VQ module substitution, we used an initial learning rate of $10^{-4}$ with linear decay to 
$10^{-5}$.
For decoder adaptation, the learning rate remained constant at $lr = 10^{-5}$. For specific details on training epochs and batch sizes, refer to Table~\ref{table:vqvae_hparams}. For adversarial training, we follow the VQ-Transplant framework~\cite{Fang2025VQTransplant}, which employs a frozen DINO-S~\cite{Caron2021EmergingPI,Oquab2023DINOv2LR} discriminator with an architecture reminiscent of StyleGAN~\cite{Karras2019AnalyzingAI,karras2019style}. Consistent with VQ-Transplant~\cite{Fang2025VQTransplant}, we further incorporate DiffAug~\cite{Zhao2020DifferentiableAF}, consistency regularization~\cite{Zhang2019ConsistencyRF}, and LeCAM regularization~\cite{Tseng2021RegularizingGA} to stabilize discriminator training.

\paragraph{Loss Weight in VQ-VAE} For all three baselines, $\beta$ is typically set to a value within the range $[0.25, 2]$. In our experiments, $\beta$ is set to a fixed value of $1.0$. For our proposed \emph{Wasserstein VQ} model, we set $\beta$ to a much smaller value,  e.g., $\beta=0.1$. The smaller $\beta$ values enable the Wasserstein distance to dominate the loss function, thereby more effectively narrowing the gap between the distributions.

\paragraph{Loss Weight in VQGAN} For our proposed Wasserstein VQ and Wasserstein VAR models, the perceptual loss weight $\lambda_P$ is fixed at 1. In multi-scale quantization experiments, we set $\lambda_G = 0.5$, whereas in fixed-scale quantization experiments, $\lambda_G = 0.4$. The coefficient $\gamma$ is set to 0.2 for all configurations incorporating the Wasserstein distance (i.e., Wasserstein VQ and Wasserstein VAR).

\section{Supplementary Results and Analyses in VQ-VAE Framework}
\subsection{VQ-VAE Performance on  CIFAR-10 and SVHN datasets}
\label{appendix:vqvae cifar-10 and svhn}
Due to space limitations in the main text, we have relocated the VQ-VAE evaluation on CIFAR-10 and SVHN datasets to the appendix. As demonstrated in Table~\ref{tab:vqvae cifar10}, \ref{tab:vqvae svhn}, our Wasserstein VQ consistently outperforms all baselines across both datasets, achieving superior results on nearly all evaluation metrics regardless of codebook size. Notably, we observe that Wasserstein VQ fails to reach 100\% codebook utilization on SVHN, which may be attributed to the dataset’s limited diversity.

\begin{table*}[!h]
  \centering
  \vspace{-1ex}
  \caption{Comparison of VQ-VAEs trained on CIFAR-10 dataset following~\cite{Oord2017NeuralDR}.}
  \label{tab:vqvae cifar10}
  \vspace{-1ex}
  \resizebox{\textwidth}{!}{
  \begin{tabular}{@{}lccccccc@{}} \hline
    Approaches  & Tokens & Codebook Size & $\mathcal{U}$ ({\color{green!60!black}\bfseries$\pmb{\uparrow}$}) & $\mathcal{C}$ ({\color{green!60!black}\bfseries$\pmb{\uparrow}$}) & PSNR({\color{green!60!black}\bfseries$\pmb{\uparrow}$}) & SSIM({\color{green!60!black}\bfseries$\pmb{\uparrow}$}) & Rec. Loss ({\color{green!60!black}\bfseries$\pmb{\downarrow}$}) \\\hline
    Vanilla VQ & 64 & $8192$ & 2.7\% & 186.9 & 27.15 & 0.83 & 0.0147   \\
    EMA VQ & 64 & $8192$ & 99.7\% & 6416.1 & 29.43 & 0.88 & 0.0095  \\
    Online VQ & 64 & $8192$ & 22.1\% & 995.4 & 28.20 & 0.85 & 0.0123  \\
    \textbf{Wasserstein VQ}  & 64 & $8192$ & \textbf{100.0\%} & \textbf{7781.8}  & \textbf{29.88} & \textbf{0.90} & \textbf{0.0085} \\\hline
    Vanilla VQ & 64 & $16384$ & 1.6\% & 220.3 & 27.36 & 0.84 & 0.0141  \\
    EMA VQ & 64 & $16384$ & 80.8\% & 10557.3 & 29.43 & 0.88 & 0.0093\\
    Online VQ & 64 & $16384$ & 13.4\% & 798.5 & 27.54 & 0.82 & 0.0141 \\
    \textbf{Wasserstein VQ}  & 64 & $16384$ & \textbf{100.0\%} & \textbf{15583.7}  & \textbf{30.19} & \textbf{0.90} & \textbf{0.0080}  \\\hline
    Vanilla VQ & 64 & $32768$ & 0.5\% & 154.8 & 27.10 & 0.83 & 0.0150  \\
    EMA VQ & 64 & $32768$ & 54.4\% & 14427.0 & 29.57 & 0.88 & 0.0091 \\
    Online VQ & 64 & $32768$   & 7.2\% & 1556.0 & 28.84 & 0.87 & 0.0106 \\
    \textbf{Wasserstein VQ} & 64 & $32768$ & \textbf{99.0\%} & \textbf{29845.1}  & \textbf{30.63} & \textbf{0.91} & \textbf{0.0071} \\\hline
  \end{tabular}
  \vspace{-8ex}
}
\end{table*}

\begin{table*}[!t]
  \centering
  \vspace{-1ex}
  \caption{Comparison of VQ-VAEs trained on SVHN dataset following~\cite{Oord2017NeuralDR}.}
  \label{tab:vqvae svhn}
  \vspace{-1ex}
  \resizebox{\textwidth}{!}{%
  \begin{tabular}{@{}lccccccc@{}} \hline
    Approaches & Tokens & Codebook Size & $\mathcal{U}$ ({\color{green!60!black}\bfseries$\pmb{\uparrow}$}) & $\mathcal{C}$ ({\color{green!60!black}\bfseries$\pmb{\uparrow}$}) & PSNR({\color{green!60!black}\bfseries$\pmb{\uparrow}$}) & SSIM({\color{green!60!black}\bfseries$\pmb{\uparrow}$}) & Rec. Loss ({\color{green!60!black}\bfseries$\pmb{\downarrow}$}) \\\hline
    Vanilla VQ  & 64 & $8192$ & 8.1\% & 533.1 & 37.81 & 0.97 & 0.0018 \\
    EMA VQ & 64 & $8192$ & 56.8\% & 3363.0 & 40.38 & \textbf{0.98} & 0.0010 \\
    Online VQ  & 64 & $8192$ & 27.8\% & 1325.1 & 39.04 & 0.97 & 0.0016 \\
    \textbf{Wasserstein VQ}  & 64 & $8192$ & \textbf{88.2\%} & \textbf{6154.5}  & \textbf{41.04} & \textbf{0.98} & \textbf{0.0009} \\\hline
    Vanilla VQ & 64 & $16384$ & 3.4\% & 446.0 & 37.87 & 0.97 & 0.0017 \\
    EMA VQ & 64 & $16384$ & 22.2\% & 2593.8 & 40.19 & \textbf{0.98} & 0.0011 \\
    Online VQ & 64 & $16384$ & 13.5\% & 1090.5 & 39.12 & 0.97 & 0.0014  \\
    \textbf{Wasserstein VQ} & 64 & $16384$ & \textbf{87.5\%} & \textbf{11967.2}  & \textbf{41.49} & \textbf{0.98} & \textbf{0.0008}  \\\hline
    Vanilla VQ & 64 & $32768$ & 1.8\% & 467.5 & 37.87 & 0.97 & 0.0017  \\
    EMA VQ & 64 & $32768$ & 35.8\% & 7662.9 & 40.25 & \textbf{0.98} & 0.0010  \\
    Online VQ & 64 & $32768$ & 7.0\% & 1334.8 & 39.26 & 0.97 & 0.0014 \\
    \textbf{Wasserstein VQ} & 64 & $32768$ & \textbf{88.7\%} & \textbf{24376.3}  & \textbf{41.84} & \textbf{0.98} & \textbf{0.0008}  \\\hline
  \end{tabular}
  \vspace{-4ex}
}
\end{table*}

\begin{table*}[!h]
  \centering
  \vspace{-1ex}
  \caption{Supplementary comparison of VQ-VAEs trained on FFHQ dataset following~\cite{Oord2017NeuralDR} w.r.t codebook size $K$.}
  \label{tab:vqvae ffhq codebook size}
  \vspace{-1ex}
  \resizebox{\textwidth}{!}{%
  \begin{tabular}{@{}lccccccc@{}} \hline
    Approaches & Tokens & Codebook Size & $\mathcal{U}$ ({\color{green!60!black}\bfseries$\pmb{\uparrow}$}) & $\mathcal{C}$ ({\color{green!60!black}\bfseries$\pmb{\uparrow}$}) & PSNR({\color{green!60!black}\bfseries$\pmb{\uparrow}$}) & SSIM({\color{green!60!black}\bfseries$\pmb{\uparrow}$}) & Rec. Loss ({\color{green!60!black}\bfseries$\pmb{\downarrow}$})\\\hline
    Vanilla VQ & 256 & 1024 & 51.7\% & 446.2 & 27.64 & 73.0 & 0.0125 \\
    EMA VQ & 256 & 1024 & 74.1\% & 618.9 & 27.66 & 72.7 & 0.0125  \\
    Online VQ & 256 & 1024 & \textbf{100.0\%} & 759.3 & 28.08 & 74.0 & 0.0114  \\
    \textbf{Wasserstein VQ} & 256 & 1024 & \textbf{100.0\%} & \textbf{977.4} & \textbf{28.11} & \textbf{74.4} & \textbf{0.0112} \\\hline
    Vanilla VQ & 256 & 2048 & 27.6\% & 453.0 & 27.78 & 73.8 & 0.0121 \\
    EMA VQ & 256 & 2048 & \textbf{100\%} & 1608.0 & \textbf{28.39} & 74.9 & \textbf{0.0107}  \\
    Online VQ & 256 & 2048 & \textbf{100\%} & 1462.6 & 28.34 & 74.6 & 0.0108  \\
    \textbf{Wasserstein VQ} & 256 & 2048 & \textbf{100\%} & \textbf{1840.5} & 28.32 & \textbf{75.3} & \textbf{0.0107} \\\hline
    Vanilla VQ & 256 & 4096 & 12.5\% & 435.0 & 27.84 & 73.7 & 0.0119 \\
    EMA VQ & 256 & 4096 & 76.7\% & 2443.1 & 28.49 & 75.0 & 0.0104 \\
    Online VQ & 256 & 4096 & 70.7\% & 1600.0 & 28.25 & 74.1 & 0.0110 \\
    \textbf{Wasserstein VQ} & 256 & 4096 & \textbf{100\%} & \textbf{3895.4} & \textbf{28.54} & \textbf{75.1} & \textbf{0.0102} \\\hline
    Vanilla VQ  & 256 & 8192 & 5.6\% & 398.1 & 27.69 & 73.5 & 0.0122  \\
    EMA VQ & 256 & 8192 & 28.9\% & 1839.2 & 28.39 & 74.8 & 0.0106  \\
    Online VQ & 256 & 8192 & 34.9\% & 1474.4 & 28.15 & 73.9 & 0.0113 \\
    \textbf{Wasserstein VQ} & 256 & 8192 & \textbf{100\%} & \textbf{7731.5} & \textbf{28.81} & \textbf{76.2} & \textbf{0.0099} \\\hline
  \end{tabular}
  \vspace{-8ex}
}
\end{table*}

\begin{table*}[!t]
  \centering
  \vspace{-1ex}
  \caption{Analysis On codebook dimension by the comparison of VQ-VAEs trained on CIFAR-10 dataset following~\cite{Oord2017NeuralDR}. (The codebook size $K$ is fixed to 16384)}
  \label{tab:vqvae cifar-10 codebook dimension}
  \vspace{-1ex}
  \resizebox{\textwidth}{!}{%
  \begin{tabular}{@{}lccccccc@{}} \hline
    Approaches  & Tokens & Codebook Dim & $\mathcal{U}$ ({\color{green!60!black}\bfseries$\pmb{\uparrow}$}) & $\mathcal{C}$ ({\color{green!60!black}\bfseries$\pmb{\uparrow}$}) & PSNR({\color{green!60!black}\bfseries$\pmb{\uparrow}$}) & SSIM({\color{green!60!black}\bfseries$\pmb{\uparrow}$}) & Rec. Loss ({\color{green!60!black}\bfseries$\pmb{\downarrow}$})\\\hline
    Vanilla VQ & 256 & 2 & 3.8\% & 532.2 & 27.00 & 0.80 & 0.0162 \\
    EMA VQ & 256 & 2 & 97.6\% & \textbf{14460.3} & 27.25 & 0.80 & \textbf{0.0155} \\
    Online VQ & 256 & 2 & 9.0\% & 611.8 & 26.62 & 0.79 & 0.0178 \\
    \textbf{Wasserstein VQ} & 256 & 2 & \textbf{99.3\%} & 12278.9  & \textbf{27.30} & \textbf{0.81} & \textbf{0.0155} \\\hline
    Vanilla VQ  & 256 & 4 & 1.3\% & 176.7 & 27.15 & 0.83 & 0.0149 \\
    EMA VQ  & 256 & 4 & 99.8\% & 13153.9 & 29.57 & 0.89 & 0.0092 \\
    Online VQ  & 256 & 4 & 11.1\% & 877.7 & 26.69 & 0.79 & 0.0173 \\
    \textbf{Wasserstein VQ}  & 256 & 4 & \textbf{100.0\%} & \textbf{15724.7}  & \textbf{29.93} & \textbf{0.89} & \textbf{0.0087} \\\hline
    Vanilla VQ   & 256 & 8 & 1.6\% & 220.3 & 27.36 & 0.84 & 0.0141 \\
    EMA VQ  & 256 & 8 & 80.8\% & 10557.3 & 29.43 & 0.88 & \textbf{0.0009} \\
    Online VQ & 256 & 8  & 13.4\% & 798.5 & 27.54 & 0.82 & 0.0141 \\
    \textbf{Wasserstein VQ} & 256 & 8  & \textbf{100.0\%} & \textbf{15583.7}  & \textbf{30.19} & \textbf{0.90} & 0.0080 \\\hline
    Vanilla VQ & 256 & 16  & 1.1\% & 150.8 & 27.05 & 0.83 & 0.0152 \\
    EMA VQ & 256 & 16 & 32.5\% & 4169.2 & 29.31 & 0.88 & 0.0099 \\
    Online VQ & 256 & 16 & 18.2\% & 2051.0 & 28.29 & 0.85 & 0.0122 \\
    \textbf{Wasserstein VQ} & 256 & 16 & \textbf{99.2\%} & \textbf{14832.2}&\textbf{30.27}  & \textbf{0.91} & \textbf{0.0078} \\\hline
    Vanilla VQ & 256 & 32 & 0.7\% & 94.37 & 26.67 & 0.81 & 0.0165 \\
    EMA VQ & 256 & 32 & 7.0\% & 942.7 & 28.24 & 0.85 & 0.0122 \\
    Online VQ & 256 & 32 & 18.8\% & 2278.0 & 28.92 & 0.87 & 0.0104 \\
    \textbf{Wasserstein VQ} & 256 & 32 & \textbf{96.4\%} & \textbf{14056.9}  & \textbf{30.39} & \textbf{0.91} & \textbf{0.0076} \\\hline
    \vspace{-6ex}
  \end{tabular}
}
\end{table*}

\begin{table*}[!t]
  \centering
  \vspace{-1ex}
  \caption{Sensitivity analysis of $\gamma$ trained on FFHQ dataset following~\cite{Oord2017NeuralDR}.}
  \label{tab:sensitivity analysis}
  \vspace{-1ex}
  \resizebox{\textwidth}{!}{%
  \begin{tabular}{@{}lcccccccc@{}} \hline
    Approaches &$\gamma$& Tokens & Codebook Size & $\mathcal{U}$ ({\color{green!60!black}\bfseries$\pmb{\uparrow}$}) & $\mathcal{C}$ ({\color{green!60!black}\bfseries$\pmb{\uparrow}$}) & PSNR({\color{green!60!black}\bfseries$\pmb{\uparrow}$}) & SSIM({\color{green!60!black}\bfseries$\pmb{\uparrow}$}) & Rec. Loss ({\color{green!60!black}\bfseries$\pmb{\downarrow}$}) \\\hline
    \textbf{Wasserstein VQ} & 0.0 & 256 & $16384$ & 3.8\% & 527.2 & 27.83 & 73.8 & 0.0119 \\
    \textbf{Wasserstein VQ} & 0.00001 & 256 & $16384$ & 24.8\% & 903.0 & 27.97 & 74.2 &  0.0114\\
    \textbf{Wasserstein VQ} & 0.0001 & 256 & $16384$ & 52.3\% & 6764.3 & 28.70 & 75.5 &  0.0100\\ 
    \textbf{Wasserstein VQ} & 0.001 & 256 & $16384$ & 90.6\% & 9988.0 & 28.93 & 76.7 &  0.0094\\ 
    \textbf{Wasserstein VQ} & 0.01 & 256 & $16384$ & \textbf{100\%} & 14952.5 & 29.06 & 76.7 &  \textbf{0.0092}\\ 
    \textbf{Wasserstein VQ} & 0.1 & 256 & $16384$ & \textbf{100\%} & \textbf{15943.5} & \textbf{29.07} & 76.7 & \textbf{0.0092} \\ 
    \textbf{Wasserstein VQ} & 0.5 & 256 & $16384$ &  \textbf{100\%} & 15713.3 & 29.03 & 76.6 & 0.0093 \\
    \textbf{Wasserstein VQ} & 1.0 & 256 & $16384$ &  \textbf{100\%}& 14712.4 & 29.02 & \textbf{76.9} & 0.0093 \\\hline
  \end{tabular}
  \vspace{-10ex}
}
\end{table*}

\subsection{Analyses on Codebook Size and Dimensionality}
\label{appendix:ablation study}

\paragraph{Analyses of Codebook Size} We investigate the impact of the codebook size $K$  on the performance of VQ by varying across a wide range: $K \in [1024, 2048, 4096, 8192, 16384, 50000, 100000]$. As shown in Table~\ref{tab:vqvae ffhq},~\ref{tab:vqvae ffhq codebook size}, the vanilla VQ model suffers from severe codebook collapse even with a relatively small $K$, such as $K=1024$. In contrast, improved algorithms like EMA VQ and Online VQ can handle smaller codebook sizes effectively, but they still experience codebook collapse when $K$ is very large, e.g., $K\geq 50000$.
Notably, the \emph{Wasserstein VQ} model consistently maintains 100\% codebook utilization, irrespective of the codebook size. This underscores the effectiveness of distributional matching via the quadratic Wasserstein distance in mitigating the issue of codebook collapse.

\paragraph{Analyses of Codebook Dimensionality} We further investigate the impact of codebook dimensionality $d$ on VQ performance. Conducting experiments on CIFAR-10 with dimensionality $d$ ranging from 2 to 32, our proposed Wasserstein VQ consistently outperforms all baselines regardless of dimensionality, as shown in Table~\ref{tab:vqvae cifar-10 codebook dimension}. Notably, we observe the curse of dimensionality phenomenon—performance degrades as dimensionality increases. Vanilla VQ exhibits the most severe degradation, followed by EMA VQ and Online VQ, while our Wasserstein VQ shows only minimal codebook utilization reduction.

\subsection{Sensitivity Analysis on $\gamma$.}
\label{appendix:sensitivity analysis}
We conduct a sensitivity analysis with respect to $\gamma$ on the FFHQ dataset by varying $\gamma \in \{0, 10^{-5}, 10^{-4}, 10^{-3}, 10^{-2}, 10^{-1}, 1\}$. As reported in Table~\ref{tab:sensitivity analysis}, when $\gamma = 0$, Wasserstein VQ yields the worst performance, as it degenerates into the vanilla VQ formulation without distributional matching. As $\gamma$ increases from $10^{-5}$ to $10^{-3}$, the performance of Wasserstein VQ consistently improves. When $\gamma$ reaches $10^{-2}$, Wasserstein VQ achieves full ($100\%$) codebook utilization along with competitive quantitative results. Moreover, within the range $\gamma \in [10^{-2}, 1]$, the performance of Wasserstein VQ remains stable, indicating that the method is not sensitive to the precise choice of $\gamma$ once it exceeds a moderate threshold.

\subsection{Computational Overhead Comparison among Various VQ Approaches}
\label{appendix:computational overhead}
To evaluate the runtime efficiency of different VQ approaches, we measure the forward and backward pass times of the VQ module over 100 iterations across three different codebook sizes, with the feature dimension set to $d=8$ and the number of data samples $N=8192$. As shown in Table~\ref{tab:computational overhead}, even at large codebook sizes (specifically, $K \geq 50,000$), the runtime of Wasserstein VQ is only slightly longer than that of Vanilla VQ. This demonstrates that Wasserstein VQ maintains high computational efficiency, and incorporating the quadratic Wasserstein distance does not introduce significant time overhead. Notably, while Online VQ exhibits substantial runtime increases at large codebook sizes, Wasserstein VQ remains considerably more efficient, further underscoring its scalability.

\subsection{Discussion with VQ-WAE~\cite{Vuong2023VectorQW}}
\label{appendix:VQ-WAE}
VQ-WAE~\cite{Vuong2023VectorQW} introduces an alternative approach to distributional matching by employing Optimal Transport to optimize codebook vectors. Compared with our proposed distributional matching method, there are three key differences.

\textbf{First, regarding theoretical contributions}: VQ-WAE~\cite{Vuong2023VectorQW} claims that achieving optimal transport (OT) between code vectors and feature vectors yields the best reconstruction performance. Their notion of optimality encompasses both the VQ process and the encoder-decoder reconstruction pipeline. While we contend that incorporating complex encoder-decoder functions renders rigorous theoretical analysis intractable, VQ-WAE nevertheless asserts this conclusion. In contrast, our work deliberately excludes encoder-decoder components, focusing solely on the VQ process, which admits rigorous mathematical modeling. Through our proposed criterion triple, we theoretically prove that distributional matching guarantees optimal performance.

\textbf{Second, regarding distribution modeling}: VQ-WAE~\cite{Vuong2023VectorQW} assumes both code vectors and feature vectors follow uniform discrete distributions, whereas our method models them as continuous distributions. Specifically, VQ-WAE~\cite{Vuong2023VectorQW} represents the distributions of feature vectors $\{\bz_i\}_{i=1}^N$ and code vectors $\{\be_k\}_{k=1}^K$ as empirical measures:
\begin{align}
\label{eq:vqwae distribution}
\mathcal{P}_A = \frac{1}{N}\sum_{i=1}^{N} \delta_{\bz_i}, \quad \mathcal{P}_B = \frac{1}{N}\sum_{k=1}^{K} \delta_{\be_k}
\end{align} 
where $\delta_{\bz_i}$ and $\delta_{\be_k}$ denote Dirac delta functions centered at $\bz_i$ and $\be_k$, respectively. To align  $\mathcal{P}_A$ and $\mathcal{P}_B$, VQ-WAE formulates the OT problem as:
\begin{equation}
\label{eq:ot_problem}
\begin{aligned}
&\min_{\mathbf{P} \in \Pi(\mathcal{P}_A, \mathcal{P}_B)} \sum_{i=1}^N \sum_{k=1}^K P_{ik} \|\bz_i - \be_k\|^2, \\
&\text{s.t.} \quad \mathbf{P} \mathbf{1}_K = \frac{1}{N}\mathbf{1}_N, \quad \mathbf{P}^\top \mathbf{1}_N = \frac{1}{K}\mathbf{1}_K, \quad P_{ik} \geq 0 \quad \forall i,k,
\end{aligned}
\end{equation}
where $\mathbf{P}$ is the transport plan, and the feasible set is:
\begin{equation}
\label{eq:coupling_set}
\Pi(\mathcal{P}_A, \mathcal{P}_B) = 
\left\{ 
\mathbf{P} \in \mathbb{R}_+^{N \times K} 
\; \middle| \;
\mathbf{P} \mathbf{1}_K = \frac{1}{N} \mathbf{1}_N, \ 
\mathbf{P}^\top \mathbf{1}_N = \frac{1}{K} \mathbf{1}_K 
\right\}
\end{equation}
In contrast, we simplify the distributional assumption by modeling $\mathcal{P}_A$ and  $\mathcal{P}_B$ as Gaussian distributions.

\begin{table*}[!t]
\centering
\caption{Reconstruction performance ($\downarrow:$ the lower the better and $\uparrow:$ the higher the better). $^{\dagger}$:Results cited from VQ-WAE~\cite{Vuong2023VectorQW}. Codebook size $K$ is fixed to $512$.}
\vspace{-2ex}
\resizebox{0.98\textwidth}{!}{
\begin{tabular}{@{}lrcrrrrrr@{}}
\toprule 
Dataset & Model & Tokens & SSIM ({\color{green!60!black}\bfseries$\pmb{\uparrow}$}) & PSNR ({\color{green!60!black}\bfseries$\pmb{\uparrow}$}) & LPIPS ({\color{green!60!black}\bfseries$\pmb{\downarrow}$}) & Rec. Loss ({\color{green!60!black}\bfseries$\pmb{\downarrow}$})  & Perplexity ({\color{green!60!black}\bfseries$\pmb{\uparrow}$})
\tabularnewline 
\midrule 
CIFAR10 & VQ-VAE$^{\dagger}$  & 64 & 70 & 23.14 & 0.35 & & 69.8  \tabularnewline
 & SQ-VAE$^{\dagger}$  & 64 & 80 & 26.11 & 0.23 & & 434.8 \tabularnewline
 & VQ-WAE$^{\dagger}$  & 64 & 80 & 25.93 & 0.23 & & 497.3 \tabularnewline
 & VQ-WAE (Our run) & 64 & 13 & 14.60 & 0.41 & 0.247 & 1.0 \tabularnewline
 & Vanilla VQ & 64 & 83 & 27.19 & 0.03 & 0.015 &192.5\tabularnewline
 & EMA VQ  & 64 & 84 & 27.97 & 0.04 & 0.013 & 436.1 \tabularnewline
 & Online VQ & 64 & 84 & 27.87 & 0.04  & 0.013 &451.4\tabularnewline
 & \textbf{Wasserstein VQ} & 64 & 86 & 28.26 & 0.03  & 0.012 & 481.7 \tabularnewline
\midrule         
SVHN & VQ-VAE$^{\dagger}$ & 64 & 88 & 26.94 & 0.17  & & 114.6 \tabularnewline
 & SQ-VAE$^{\dagger}$  & 64 & 96 & 35.37 & 0.06 & & 389.8 \tabularnewline
 & VQ-WAE$^{\dagger}$  & 64 & 96 & 34.62 & 0.07 &  & 485.1 \tabularnewline
 & VQ-WAE (Our run) & 64 & 25 & 15.87 & 0.26 & 0.2026 &1.0\tabularnewline
 & Vanilla VQ & 64 & 97 & 38.18 & 0.01 & 0.0016 & 407.1\tabularnewline
 & EMA VQ & 64 & 97 & 38.35 & 0.01 & 0.0017 & 408.9 \tabularnewline
 & Online VQ & 64 & 97 & 38.54 & 0.01 & 0.0017 & 421.5\tabularnewline
 & \textbf{Wasserstein VQ} & 64 & 97& 38.25 & 0.01 & 0.0016 & 423.5 \tabularnewline
\midrule 
\end{tabular}}
\label{tab:vq-wae}
\vspace{-4ex}
\end{table*}

\textbf{Third, regarding computational efficiency}, The OT problem in VQ-WAE is prohibitively complex, whereas our quadratic Wasserstein distance incurs minimal overhead. To mitigate complexity, VQ-WAE employs a Kantorovich potential network. However, upon reproducing their code (no official implementation was released; we derived it from their ICLR 2023 supplementary material, \footnote{See \url{https://openreview.net/forum?id=Z8qk2iM5uLI}. We includes the reproduced code and training logs of VQ-WAE in our supplementary materials.}), 
we observed severe non-convergence—the method degenerated to using a single code vector, failing to achieve distributional matching. Notably, VQ-WAE underperformed all other VQ baselines (Table~\ref{tab:vq-wae}).

In comparison, our quadratic Wasserstein distance (Equation~\ref{eq:wasserstein distance}) requires only low-dimensional matrix operations (e.g., $d=8$), achieving superior performance and effective matching (Figure~\ref{fig:feature codebook visualization}).